\documentclass[11pt]{article}

\usepackage[preprint]{acl}

\usepackage{times}
\usepackage{latexsym}

\usepackage[T1]{fontenc}

\usepackage[utf8]{inputenc}

\usepackage{microtype}

\usepackage{inconsolata}

\usepackage{graphicx}

\usepackage{mathtools}

\usepackage{amsthm}

\newtheoremstyle{breakable}
  {\topsep}
  {\topsep}
  {\itshape}
  {}
  {\bfseries}
  {.}
  {0.5em}
  {\thmname{#1}\thmnumber{ #2}\thmnote{ (#3)}}

\theoremstyle{breakable}

\newtheorem{assumption}{Assumption}
\newtheorem{theorem}{Theorem}
\newtheorem{remark}{Remark}

\newtheorem{proposition}{Proposition}
\newtheorem{definition}{Definition}

\newtheorem{lemma}{Lemma}

\newtheorem{question}{Question}

\usepackage{amssymb}
\usepackage{amsmath}
\usepackage{booktabs}
\usepackage{subcaption}
\usepackage{placeins}

\newcommand{\bad}[1]{{\color{failred} #1}}
\newcommand{\weak}[1]{{\color{weakorange} #1}}
\newcommand{\KL}{D_{\mathrm{KL}}}
\newcommand{\signaturepanel}[2]{%
  \begin{subfigure}[t]{0.235\linewidth}
    \centering
    \includegraphics[width=\linewidth]{#1}
    \caption{#2}
  \end{subfigure}%
}

\definecolor{failred}{RGB}{180,30,30}
\definecolor{weakorange}{RGB}{200,120,20}

\title{The Stepwise Informativeness Assumption: \\ When Do Entropy Dynamics Track Reasoning \\ Correctness in LLMs?}

\author{
  Mar Gonzàlez I Català \\
  University of Cambridge \\
  \texttt{mg2211@cam.ac.uk}
  \And
  Haitz Sáez de Ocáriz Borde \\
  University of Cambridge \\
  \texttt{hs788@cam.ac.uk}
  \AND
  George D. Montañez \\
  Harvey Mudd College \\
  \texttt{gmontanez@g.hmc.edu}
  \And
  Pietro Liò \\
  University of Cambridge \\
  \texttt{pl219@cam.ac.uk}
}

\begin{document}
\maketitle
\begin{abstract}
Recent work uses entropy-based signals at multiple representation levels to study reasoning in large language models, but the field remains largely empirical. From a theoretical perspective, reductions in model uncertainty do not necessarily reflect progress toward correctness. In this work, we characterize when such alignment can be expected through the Stepwise Informativeness Assumption (SIA), which requires entropy-reducing reasoning-prefix extensions increase the probability assigned to valid answers in expectation. We show that next-token prediction can transfer answer-relevant information from training data to the model, but does not, by itself, guarantee alignment with correctness. Instead, we propose that this alignment depends on how the training regime couples reasoning traces to correct answers. Furthermore, we hypothesize observable signatures of this phenomenon and characterize its expected failure regimes. We empirically investigate how entropy dynamics relate to correctness across four reasoning benchmarks (GSM8K, ARC, SVAMP, MATH) and diverse open-weight language models spanning base, instruction-tuned, and RL-trained regimes. Our results show that entropy-correctness alignment strengthens with post-training but can weaken when task difficulty exceeds model capabilities.
\end{abstract}

\section{Introduction}

As language models become more complex, researchers  increasingly turn to model-derivable signals as a way to study their behavior. In particular, a growing body of empirical studies analyzes model-internal entropy dynamics and consistently reports strong alignment between characteristic patterns and reasoning quality in large language models~(LLMs). These signals have been successfully used to improve reasoning performance \cite{agarwal2025unreasonableeffectivenessentropyminimization, li2025compressingchainofthoughtllmsstep}, guide exploration and early stopping \cite{sharma2025thinkjustenoughsequencelevel, zhang2025entropybasedexplorationconductionmultistep}, identify critical decision points \cite{ali2026entropylensuncoveringdecisionstrategies, qian2025demystifyingreasoningdynamicsmutual, wang20258020rulehighentropyminority}, and detect failures such as hallucinations or overthinking \cite{farquhar_detecting_2024, ton2025understandingchainofthoughtllmsinformation}. However, despite the empirical success of entropy-based approaches to reasoning, a central unresolved question remains:

\begin{question}\label{q:entropy}
When do internal entropy dynamics—defined purely with respect to a model’s predictive distribution—reliably track external correctness, which is defined only relative to the ground-truth answer?
\end{question}

\begin{figure*}
    \centering
    \includegraphics[width=\linewidth]{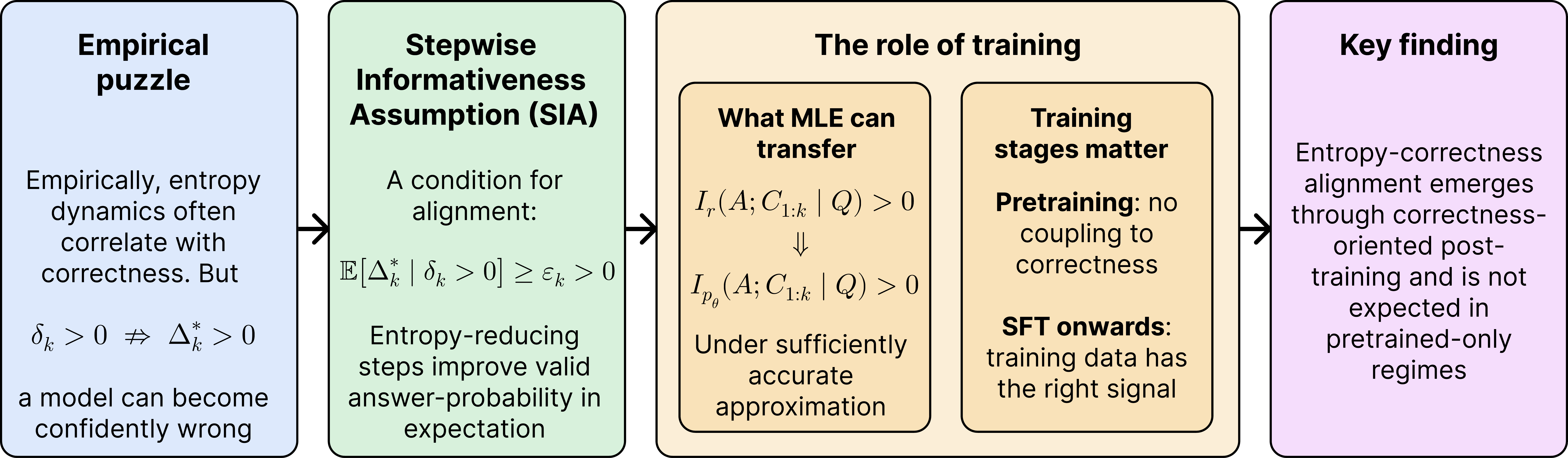}
    \caption{\textbf{Overview of the proposed framework.} We formalize the conditions for entropy-correctness alignment and investigate how different training stages may promote it.}
    \label{fig:ds}
\end{figure*}

In this paper, we argue that entropy is theoretically justified as a correctness-tracking signal when reductions in a model's predictive answer uncertainty are systematically directed toward the ground-truth answer. We formalize this via the \emph{Stepwise Informativeness Assumption} (SIA), a minimal information-theoretic condition requiring that, among reasoning-prefix extensions that reduce predictive answer entropy, the probability assigned to the ground-truth answer increases in expectation. Having isolated this condition, we then ask under what circumstances it should be expected to hold. 

One natural place to look is the training process, since it determines the statistical relationships the model learns between reasoning prefixes and subsequent answers. We show that, while maximum-likelihood training can transfer answer-relevant structure from the training data to the model, whether that structure is tied to correctness depends on the training regime: generic pretraining provides no explicit coupling to the ground-truth answer, whereas supervised fine-tuning on correct reasoning traces does. We further characterize the reasoning dynamics associated with entropy-correctness alignment. In aligned models, correct trajectories accumulate answer-relevant information earlier and become distinguishable from incorrect trajectories well before the final answer. 

We empirically test these predictions across sixteen open-weight models spanning base, instruction-tuned, and reinforcement-learning-trained regimes on four reasoning benchmarks. We find that entropy reduction is not reliably aligned with increasing probability of the ground-truth answer in base models, while this alignment generally strengthens with supervised and reasoning-oriented post-training. A controlled fine-tuning intervention further shows that the effect emerges during supervised training itself. These results support the view that the usefulness of entropy as a correctness-tracking signal is not an intrinsic property of autoregressive models, but depends on how training couples intermediate reasoning states to the correct answer. More broadly, our results highlight the importance of validating model-internal reasoning signals before treating them as proxies for external properties.

\section{Preliminaries and Notation}
\label{sec:preliminaries}

We now provide the notation and standard definitions of language factorization, ground truth and model distributions, and information-theoretic quantities used throughout the paper.

\subsection{Next-token prediction and likelihood factorization}

Modern language models are trained under the next-token prediction paradigm.

\begin{definition}[Next-token prediction and autoregressive factorization]
Given an input prefix $X_{1:k}$, a language model with parameters $\theta$ defines a conditional distribution over the next token $p_\theta(X_{k+1}\mid X_{1:k}),$ and the likelihood of a full sequence factorizes autoregressively as $p_\theta(X_{1:K})=\prod_{k=1}^{K} p_\theta(X_k \mid X_{<k})$, where $X_{<k} \coloneqq X_{1:k-1}$.
\end{definition}

\begin{definition}[Autoregressive language model training objective]
Let the training corpus be a collection of $N$ token sequences of variable length $K_i$, $\mathcal{D}=\{X_{1:K_i}^{(i)}\}_{i=1}^N$. The maximum-likelihood training objective for a language model with parameters $\theta$ is defined as $\theta^\ast= \arg \max_\theta \sum^N_{i=1}\log p_\theta(X_{1:K_i}^{(i)})$, which expands autoregressively as a sum over token log-likelihoods.
\end{definition}

In practice, this objective is implemented  by minimizing the cross-entropy loss $\mathcal{L}_{\text{CE}}=-\sum^N_{i=1}\sum^{K_i}_{k=1}\log p_\theta(X_k^{(i)} \mid X_{<k}^{(i)}).$ This encourages the model to make each future token as predictable as possible given the past. We later examine how this pressure towards subsequent predictability affects reasoning processes, correctness, and entropy minimization.

\subsection{Ground truth, data, and model distribution}
\label{subsec:preliminaries-distributions}

The following definitions distinguish the model's internal predictive distribution from the external notion of correctness referenced in Question~\ref{q:entropy}, and introduce the training-data distribution through which these may become aligned.

\begin{definition}[Valid-answer set]
Let $\mathcal{Q}$ denote the space of valid queries and $\mathcal{A}$ the space of possible answers. For a query $q \in \mathcal{Q}$, the subset
$\mathcal{A}^{\star}(q) \subseteq \mathcal{A}$
denotes the nonempty set of valid answers.
\label{def:groundtruth}
\end{definition}

\begin{definition}[Training-data distribution]
Let $Q$ denote a query, $C_{1:K}=(C_1,\ldots,C_K)$ a reasoning trace, and $A_{1:T}=(A_1,\ldots,A_T)$ its subsequent answer, where $K$ and $T$ are random sequence lengths. We denote by
\(
r(Q,C_{1:K},A_{1:T})
\)
the joint distribution over queries, reasoning traces, and answers in the training data.
\label{def:datadistribution}
\end{definition}

\begin{definition}[Model predictive distribution]
Given a query \(Q\), a reasoning model with parameters \(\theta\) generates a sequence of intermediate reasoning tokens
\(
C_{1:K} = (C_1,\dots,C_K)
\)
followed by an answer sequence
\(
A_{1:T} = (A_1,\dots,A_T),
\)
where $K$ and $T$ are stochastic sequence lengths. The corresponding autoregressive distribution over reasoning traces is given by 
$p_\theta(C_{1:K}\mid Q)
=
\prod_{k=1}^{K}
p_\theta(C_k\mid Q,C_{1:k-1}).$
Conditioned on a complete reasoning trace \(C_{1:K}\), the answer distribution factorizes as
$
p_\theta(A_{1:T}\mid Q,C_{1:K})
=
\prod_{t=1}^{T}
p_\theta(A_t\mid Q,C_{1:K},A_{1:t-1}).
$
\label{def:modelpredictive}
\end{definition}

When defining stepwise entropy and information-gain quantities, we will also condition on partial prefixes $C_{1:k}$ (for $k<K$), which yields $p_\theta(A\mid Q,C_{1:k})$ by the same factorization. Importantly, note that token-level entropies and conditional answer entropies are properties of the model's internal predictive distribution $p_\theta$ and are not defined with respect to the valid-answer set $\mathcal{A}^\star(q)$. The central question studied in this paper is under what conditions changes in these internal quantities are aligned with external correctness.

\subsection{Information-Theoretic Preliminaries}

Next, we summarize the information-theoretic measures used throughout the paper. All random variables are assumed to be discrete. Our notation is: uppercase letters (e.g., $A$, $Q$, $C_k$) denote random variables; lowercase letters (e.g., $a$, $q$, $c_k$) denote particular realizations or sampled values of those variables; calligraphic letters (e.g., $\mathcal{A}$, $\mathcal{Q}$, $\mathcal{C}$) denote the set of all possible values each random variable can take. Unless otherwise stated, all logarithms are natural logarithms.

\begin{definition}[Surprisal]
    Let $X$ be a discrete random variable taking values in $\mathcal{X}$, with probability mass function $p(x)=\Pr[X=x]$. The surprisal of a realization $x\in\mathcal{X}$ is defined as
    $
    h(x):=-\log p(x).
    $
    Conditioned on a realization $Y=y$, we similarly write
    $
    h(x\mid y):=-\log p(x\mid y).
    $
\end{definition}

\begin{definition}[Entropy]
    Let $X$ be a discrete random variable taking values in $\mathcal{X}$, with probability mass function $p(x)=\Pr[X=x]$. The entropy (average or expected surprisal) of $X$ is defined as $H(X):=-\sum_{x\in\mathcal{X}} p(x)\log p(x).$
\end{definition}

\begin{definition}[Conditional entropy]
    Let $X$ and $Y$ be discrete random variables taking values in $\mathcal{X}$ and $\mathcal{Y}$, with joint pmf $p(x,y)=\Pr[X=x,Y=y]$, and marginal pmfs $p(x)=\Pr[X=x]$, $p(y)=\Pr[Y=y]$. The \emph{conditional entropy} of $Y$ given $X$ is $H(Y\mid X):=-\sum_{x\in\mathcal{X},y\in\mathcal{Y}}p(x,y)\log \left(\frac{p(x,y)}{p(x)}\right),$
    with the convention that $0\log0=0$.
\end{definition}

\begin{definition}[Mutual Information]
    Let $X$ and $Y$ be discrete random variables taking values in $\mathcal{X}$ and $\mathcal{Y}$, with joint pmf $p(x,y)=\Pr[X=x,Y=y]$, and marginal pmfs $p(x)=\Pr[X=x]$, $p(y)=\Pr[Y=y]$. The \emph{mutual information between $X$ and $Y$} is defined as $I(X;Y):=\sum_{x\in\mathcal{X}, y\in \mathcal{Y}}p(x,y)\log\left(\frac{p(x,y)}{p(x)p(y)}\right).$
\end{definition}

Note that all these definitions rely on logarithms. While the use of logarithms is not mandated, they uniquely satisfy a few intuitive properties: information from independent events adds, rarer events carry more information, and small changes in probability produce small changes in information.

\section{Why should entropy track correctness?}
\label{sec:entropy-correctness_puzzle}

Internal entropy is defined entirely under a model’s predictive distribution (Definition~\ref{def:modelpredictive}), whereas correctness is defined with respect to an external valid-answer set (Definition~\ref{def:groundtruth}). There is no \textit{a priori} reason for these notions to be aligned: internal uncertainty could track stylistic variability or model-internal ambiguity unrelated to task success. Indeed, recent work cautions against treating intermediate tokens as faithful indicators of reasoning progress or task difficulty \cite{kambhampati2025stopanthropomorphizingintermediatetokens, palod2025performativethinkingbrittlecorrelation}.

\subsection{Empirical evidence for entropy-correctness alignment}

Despite this conceptual gap, numerous studies report an alignment between internal entropy dynamics and reasoning accuracy and exploit it for analysis, control, and prediction of reasoning behavior.

\paragraph{Analysis.}
High entropy is associated with overextrapolation (``hallucination'') and unreliable outputs, while entropy plateaus correspond to ``overthinking,'' where additional reasoning does not improve accuracy \cite{farquhar_detecting_2024}. Successful trajectories exhibit distinctive entropy patterns: uncertainty concentrates at critical “forking” steps and is systematically reduced thereafter \cite{qian2025demystifyingreasoningdynamicsmutual, wang20258020rulehighentropyminority}.

\paragraph{Control.}
Early-stopping methods terminate chain-of-thought generation once entropy plateaus or falls below a threshold \cite{ sharma2025thinkjustenoughsequencelevel}, while compression and exploration-based approaches treat entropy as a signal of decision points, pruning or expanding reasoning accordingly \cite{li2025compressingchainofthoughtllmsstep, zhang2025entropybasedexplorationconductionmultistep}.

\paragraph{Prediction.}
Entropy-based metrics can reliably predict whether an ongoing reasoning trajectory will ultimately be correct: traces with a decreasing entropy trajectory are much more likely to end in correct answers \cite{guo2025measuringreasoningutilityllms, liu2025tokensignaturepredictingchainofthought}.

\subsection{Common justifications for entropy-based reasoning methods}

But why should internal uncertainty dynamics track external correctness? The literature offers several recurring explanations, none of which fully resolve the puzzle. A common implicit assumption is that reductions in entropy reflect a narrowing of the space of plausible solutions \cite{qian2025demystifyingreasoningdynamicsmutual, ton2025understandingchainofthoughtllmsinformation}. This presupposes that the uncertainty being reduced concerns the correct answer.

Other works appeal to training-induced alignment: since models are trained to be correct, their internal uncertainty should track correctness \cite{sharma2025thinkjustenoughsequencelevel}. This would be compelling if it specified the structural properties of the learned distribution that ensure predictive entropy becomes aligned with the ground truth throughout a reasoning chain, but such conditions are not articulated. 

Some analyses assume that reasoning steps reduce uncertainty about the true hypothesis \cite{ton2025understandingchainofthoughtllmsinformation}. However, this presupposes a coupling between intermediate model states and the ground-truth answer that is not derived from the training objective or the structure of the learned distribution.

Finally, many works offer no justification at all, treating entropy-correctness alignment as an empirical fact to be exploited rather than a phenomenon to be explained \cite{liu2025tokensignaturepredictingchainofthought, zhang2025entropybasedexplorationconductionmultistep}. Exploiting an alignment, however, does not explain it. To our knowledge, no prior work asks why this alignment should arise, or under what conditions it should be expected to hold or fail.

\section{Stepwise Informativeness Assumption}
\label{sec:assumption}

We now formalize the condition under which reductions in a model's predictive answer entropy are associated with progress toward valid answers. We then investigate how training may encourage this alignment, what observable patterns may accompany it, and when we do not expect it to hold. Proofs for all formal results and supporting auxiliary lemmas can be found in Appendix~\ref{app:proofs}.

\subsection{Entropy reduction and true-answer gain}
\label{subsec:stepwise_information_gain}

To distinguish becoming more certain from becoming more correct, we introduce two quantities: the change in the model's predictive answer entropy and the change in probability assigned to valid answers as reasoning progresses.

\begin{definition}[Valid-answer probability]
For a query $Q=q$, let $\mathcal A^\star(q)\subseteq\mathcal A$ denote
the set of valid answers. For any realized reasoning prefix $c_{1:k}$, define
\(
p_\theta\!\left(\mathcal A^\star(q)\mid q,c_{1:k}\right)
:=
\sum_{a\in\mathcal A^\star(q)}
p_\theta(a\mid q,c_{1:k}).
\)
\end{definition}

\begin{definition}[Set-valued surprisal]
For a query $Q=q$, a set of answers $\mathcal S\subseteq\mathcal A$, and a realized reasoning prefix
$C_{1:k}=c_{1:k}$, we define the \emph{conditional set surprisal} as
\(
h_\theta(\mathcal S\mid q,c_{\leq k})
:=
-\log p_\theta(\mathcal S\mid q,c_{\leq k}).
\)
In particular, for the valid-answer set $\mathcal A^\star(q)$, we write
\(
h_\theta^\star(q,c_{\leq k})
:=
h_\theta\!\left(\mathcal A^\star(q)\mid q,c_{\leq k}\right)
=
-\log
p_\theta\!\left(\mathcal A^\star(q)\mid q,c_{\leq k}\right).
\)
\end{definition}

\begin{definition}[Set-valued information gain]
For a query $Q=q$, a set of answers $\mathcal S\subseteq\mathcal A$, and a realized reasoning prefix
$C_{1:k}=c_{1:k}$, we define the \emph{set-valued information gain at step $k$} as
\(
\Delta_k(q,\mathcal S,c_{1:k})
:=
h_\theta(\mathcal S\mid q,c_{<k})
-
h_\theta(\mathcal S\mid q,c_{\leq k}).
\)
Equivalently,
\(
\Delta_k(q,\mathcal S,c_{1:k})
=
\log
\frac{
p_\theta(\mathcal S\mid q,c_{\leq k})
}{
p_\theta(\mathcal S\mid q,c_{<k})
}.
\)
In particular, for the valid-answer set $\mathcal A^\star(q)$, we write
\(
\Delta_k^\star
:=
\Delta_k\!\left(q,\mathcal A^\star(q),c_{1:k}\right)
=
h_\theta^\star(q,c_{<k})
-
h_\theta^\star(q,c_{\leq k})
=
\log
\frac{
p_\theta\!\left(\mathcal A^\star(q)\mid q,c_{\leq k}\right)
}{
p_\theta\!\left(\mathcal A^\star(q)\mid q,c_{<k}\right)
}.
\)
\end{definition}

\begin{definition}[Predictive answer entropy]
For a query $q$ and a realized reasoning prefix $c_{1:k}$, we define
\(
H_\theta(A\mid q,c_{\leq k})
:=
-\sum_{a\in\mathcal A}
p_\theta(a\mid q,c_{\leq k})
\log p_\theta(a\mid q,c_{\leq k}).
\)
\end{definition}

\begin{definition}[Prefix-conditioned entropy reduction]
For a query $q$ and a realized prefix $c_{1:k}$, define the predictive
answer-entropy reduction produced by step $k$ as
\(
\delta_k(q,c_{1:k})
:=
H_\theta(A\mid q,c_{<k})
-
H_\theta(A\mid q,c_{\leq k}).
\)
\end{definition}

\begin{remark}
The quantities $\delta_k$ and $\Delta_k(q,\mathcal{S},c_{1:k})$ measure different
effects of the same realized prefix extension. The quantity $\delta_k$
measures whether the model's predictive distribution becomes more
concentrated over the full answer space, whereas
$\Delta_k(q,\mathcal{S},c_{1:k})$ measures the change in probability assigned to
a particular set of answers $\mathcal{S}$. In particular, $\Delta_k^\star$ measures
whether the valid-answer set becomes more probable. In general,
neither $\delta_k$ nor $\Delta_k^\star$ determines the sign of the
other.
\end{remark}

\begin{proposition}[Entropy reduction does not imply true-answer gain]
\label{prop:entropy-not-sia}
Reductions in predictive answer entropy do not by themselves imply true-answer gain.
In particular, there exist settings in which
\(
\Pr\!\left[\delta_k(Q,C_{1:k})>0\right]>0
\)
but
\(
\mathbb{E}\!\left[
\Delta_k^\star
\;\middle|\;
\delta_k(Q,C_{1:k})>0
\right]
\leq 0.
\)
\end{proposition}

\subsection{Stepwise Informativeness Assumption}

Since entropy reduction alone does not guarantee progress toward valid answers, we now introduce an additional condition under
which the two are aligned. 

\begin{assumption}[Stepwise Informativeness Assumption (SIA)]
\label{ass:SIA}
Model prefixes are informative about the true answer in expectation:
whenever a realized prefix extension reduces the model's predictive
answer entropy, the probability assigned to the ground-truth answer
increases on average. For every $k\geq 1$ such that
\(
\Pr\!\left[\delta_k(Q,C_{1:k})>0\right]>0,
\)
\(
\mathbb{E}\!\left[
\Delta_k^\star
\;\middle|\;
\delta_k(Q,C_{1:k})>0
\right]
\geq \varepsilon_k>0,
\)
where the expectation is over task instances and model-generated
reasoning prefixes.
\end{assumption}

We call this the Stepwise Informativeness Assumption (SIA) because it
requires model-generated prefixes to be informative about the
ground-truth answer in expectation. SIA does not require every individual
prefix extension to increase the probability of the ground-truth answer;
rather, it requires positive average true-answer gain among extensions
that reduce predictive answer entropy. This formulation accounts for redundant tokens or stalling phrases that do not provide immediate marginal information. Notably, no
restriction is imposed on realizations for which $\delta_k\leq 0$.

Having defined the required alignment condition, we next ask how it might arise in language models.

\subsection{How might SIA emerge through training?}
\label{sec:sia_training}

In this section, we consider how human-generated reasoning traces may contain answer-relevant information, then examine how such information can transfer to a model through MLE training, and finally discuss how different training regimes may influence its alignment with correctness.

\subsubsection{Answer-relevant information in human-generated reasoning traces}

Recent research  \cite{futrell_linguistic_2025} shows that under realistic cognitive constraints (limited memory, attention, and processing capacity) sequential signals that minimize \textit{predictive information}, the mutual information between past and future, tend to decompose information into approximately independent components expressed locally and incrementally. This yields sequences that are systematically and progressively informative, a structure that closely matches natural language.

Human-written reasoning traces are a special case of these sequential signals, with the additional property that their future includes the correct answer. Thus, as a trace unfolds, prefixes $C_{1:k}$ should increasingly constrain the plausible continuations and answers. Formally, for a data-generating distribution $r(Q,C_{1:K},A)$, this implies that \(H_r(A \mid Q, C_{1:k})\) is non-increasing with $k$ in expectation.

These observations suggest that human-generated reasoning traces may contain answer-relevant dependencies. This motivates studying whether such dependencies can be learned by language models through MLE training.

\subsubsection{MLE and Prefix-Answer Informativeness}

We now study how prefix-answer informativeness present in human-generated reasoning traces can transfer to a model via MLE training. 

\begin{lemma}
\label{lemma:log-likelihood_identity}
Let $r$ denote the empirical data-generating distribution over full sequences 
$X = (Q,C_{1:K},A)$, and let $p_\theta$ be the model distribution.  
The negative log-likelihood objective $\mathcal{L}(\theta) \;=\; \mathbb{E}_{X\sim r}\!\left[-\log p_\theta(X)\right]$
satisfies the identity
\(
\mathcal{L}(\theta)
= -\sum_{x} r(x)\log p_\theta(x)
= H_r(\cdot) + \KL\!\left(r \,\Vert\, p_\theta\right),
\)
where $H_r(\cdot)$ is independent of $\theta$.  
Hence minimizing $\mathcal{L}(\theta)$ is equivalent to minimizing 
$\KL(r \,\Vert\, p_\theta)$, and thus drives $p_\theta$ toward $r$ in forward KL-divergence, up to model capacity.
\end{lemma}

Lemma~\ref{lemma:log-likelihood_identity} establishes that MLE encourages distributional matching. Importantly, matching the joint distribution also constrains the model's conditional distributions over reasoning traces and subsequent answers, as made explicit by the following decomposition.

\begin{lemma}[KL Decomposition of the Joint Conditional]
\label{lemma:KLDecompositionoftheJointConditional}
For any joint distributions $r$ and $p_\theta$ over $(C_{1:K},A)$ conditioned on $Q$,
the following identity holds:
\(
\KL\!\left(
r(C_{1:K},A\mid Q)
\;\Vert\;
p_\theta(C_{1:K},A\mid Q)
\right)
= \KL\!\left(r(C_{1:K}\mid Q)\,\Vert\,p_\theta(C_{1:K}\mid Q)\right)
+
\mathbb{E}_{r(C_{1:K}\mid Q)}
\!\left[
\KL\!\left(
r(A\mid Q,C_{1:K})
\,\Vert\,
p_\theta(A\mid Q,C_{1:K})
\right)
\right].
\)
\label{lemma:joint-kl-decomposition}
\end{lemma}

By nonnegativity of KL divergence, sufficiently close joint approximation also bounds the marginal and expected conditional discrepancies (Lemma~\ref{lemma:MLE_implies_alignment}, Appendix~\ref{app:proofs}). Thus, distributional matching concerns, in particular, the statistical relationships between reasoning traces and answers.

We now ask: if reasoning prefixes carry information about subsequent answers under the training distribution, does this property persist under the learned model? The following theorem answers this question in the affirmative under sufficiently accurate distributional approximation.

\begin{theorem}[Transfer of internal stepwise informativeness to the model]
\label{thm:transfer-sia}
Given a step $k \geq 1$, let $r$ denote the empirical data joint over $(Q,C_{1:K},A)$, and suppose that $I_r(A;C_{\leq k} | Q)>0$. Let $p_\theta$ denote the model distribution over $(Q,C_{1:K},A)$, then there exists $\delta_k > 0$ such that, whenever
\(
\KL(r \,\Vert\, p_\theta) \le \delta_k,
\)
the model joint $p_\theta$ satisfies
\(
I_{p_\theta}(A; C_{\leq k} \mid Q) \ge \frac{\varepsilon_k}{2} > 0.
\)
\end{theorem}

Theorem~\ref{thm:transfer-sia} follows from continuity of conditional mutual information under small distributional perturbations (Lemmas~\ref{lemma:entropy-continuity},~\ref{lemma:cond-entropy-continuity}, and~\ref{lemma:mi-continuity}, Appendix~\ref{app:proofs}) and establishes that prefix-answer informativeness is preserved when the model distribution sufficiently approximates the data-generating distribution. However, this concerns the informativeness of generated reasoning prefixes about the model's own generated answer, and does not necessarily imply alignment with external correctness.

\subsection{The role of training in correctness alignment}

The previous section established that prefix-answer dependencies present in training data can be transfered into the model. The key question now is whether the training data contains dependencies through which reasoning prefixes make the correct answer more predictable. We examine how this may differ across the training stages formalized by InstructGPT \cite{ouyang2022traininglanguagemodelsfollow}.

\paragraph{Pretraining on raw data.} In pretraining, the model is trained by maximum-likelihood next-token prediction on large-scale text corpora. Correctness is not explicitly optimized at this stage, the model is simply rewarded for continuations that make future tokens predictable. Consequently, models may learn predictable reasoning patterns without those patterns being systematically aligned with correctness.

\paragraph{Supervised fine-tuning on labeled chain-of-thought triples.} In supervised fine-tuning (SFT), the model is trained on explicit triples $(Q, C_{1:K}, A)$, where $C_{1:K}$ is a human-written chain-of-thought leading to the correct answer $A$. The same maximum-likelihood objective is used, but correctness is now part of the data distribution: traces that make $A$ probable receive higher likelihood. Thus, the model is encouraged to generate steps that reduce uncertainty about the true answer.

\paragraph{Post-training with reinforcement learning.} Finally, reinforcement-learning post-training methods such as PPO~\cite{schulman2017proximal}, GRPO~\cite{shao2024deepseekmath}, or RL with verifiable rewards~(RLVR)~\cite{wen2025reinforcementlearningverifiablerewards} refine the generation policy using outcome- or process-level rewards. They further reinforce trajectories leading to correct answers and penalize failures, strengthening the alignment between uncertainty reduction and correctness by reshaping probability mass over reasoning patterns learned during pretraining and SFT.

This suggests that entropy-correctness alignment is not an inherent consequence of next-token prediction training, but may instead emerge through correctness-oriented post-training. We investigate this empirically in Section~\ref{sec:empirical_validation}.

\subsection{Observable signatures of entropy-correctness alignment}

SIA characterizes the relationship between reductions in predictive answer entropy and gains in the probability assigned to valid answers. We now consider two trajectory-level patterns that may accompany this alignment.

\paragraph{Early information accumulation.}
If entropy reduction is aligned with correctness, successful trajectories may accumulate answer-relevant information earlier, as earlier gains allow the model to assign greater probability to valid answers throughout a larger portion of the reasoning process.

\paragraph{Early separability.}
If entropy reduction is systematically associated with valid-answer gain, successful and unsuccessful trajectories may exhibit different predictive answer-entropy patterns throughout chain-of-thought generation.

\subsection{Regimes in which SIA does not hold}

We now consider regimes in which the alignment required by SIA may fail. As stated in Section~\ref{sec:sia_training}, we do not expect SIA to hold in pretrained-only models, where next-token prediction does not explicitly reinforce reasoning prefixes that increase the probability of the ground-truth answer. Similarly, alignment may also fail in weakly supervised, out-of-distribution, or adversarial settings, where the learned relationships between reasoning prefixes and correct answers may be unreliable. Moreover, even in models trained for reasoning, SIA may weaken when task difficulty exceeds model capabilities, preventing the model from reliably identifying and accumulating information relevant to the correct answer.

Lastly, it is worth noting our framework is most applicable to problems with a well-defined terminal variable as opposed to free-form outputs like creative writing.

\section{Empirical validation}
\label{sec:empirical_validation}

In this section, we empirically investigate when predictive answer entropy serves as a reliable diagnostic of correctness, and how this relationship varies across training regimes. 
In particular, we evaluate whether
reductions in predictive answer entropy are accompanied by increases in
the probability assigned to the ground-truth answer.
We organize our empirical evaluation around three questions: (i) Does conditional entropy descent align with increasing probability of the true answer? (ii) Is this alignment induced and strengthened by training for reasoning?  And, (iii) What are the observable signatures and failure modes of entropy-correctness alignment?

Directly testing SIA would require estimating expected true-answer gains conditioned on entropy reduction at every reasoning step. Given the computational cost of estimating these quantities, we instead investigate trajectory-level entropy-correctness alignment using checkpoint-based measurements. We evaluate sixteen models across four datasets (GSM8K, ARC, SVAMP, MATH), spanning base, instruction-tuned, CoT-tuned and RL-trained regimes. All entropy quantities are estimated via Monte Carlo rollouts under stochastic decoding. Full evaluation details are provided in Appendix~\ref{app:exp}. Further results and ablations are provided in Appendices~\ref{app:further_results} and~\ref{app:ablations}.

\subsection{Entropy-answer alignment}
 
We investigate whether reductions in predictive answer entropy coincide with gains in true-answer-probability along generated reasoning trajectories.
To test this, we define the following diagnostic.

\paragraph{SIA alignment coefficient.}
For each generated trace, we compute the correlation (across prefix steps~$k$) between conditional answer entropy and gold surprisal:
\(
\rho_{\mathrm{SIA}} \;:=\; \mathrm{corr}_{k}\!(
H_{\theta}(A \mid q, c_{\le k}),
\;-\log p_{\theta}(a^{\star} \mid q, c_{\le k})
).
\)
Positive $\rho_{\mathrm{SIA}}$ indicates that uncertainty reduction is aligned with increasing probability of the correct answer.

Table~\ref{tab:sia_alignment_by_dataset} summarizes $\rho_{\mathrm{SIA}}$ by model and dataset. Across GSM8K, SVAMP, and ARC base models generally exhibit weak or negative alignment, while their instruction-tuned counterparts consistently improve, and reasoning-optimized models achieve the strongest alignment. These results suggest that truth-directed entropy descent is not a generic property of autoregressive pretraining, but is induced by supervision and strengthened by reasoning-oriented post-training. A controlled LoRA SFT intervention confirms that supervision itself induces this alignment (see Appendix~\ref{app:lora-sft-sia}).

MATH introduces a capability-mismatch regime. Because it is substantially harder than the other datasets, many smaller models fail to accumulate reliable information about the true answer. Nevertheless, the relative pattern remains: instruction-tuned models usually improve over their base counterparts, and larger instruction-tuned models recover strong positive alignment. Thus, the observed entropy-correctness alignment appears when training and model capacity are sufficient relative to task difficulty. Model-size trends support this capacity-mismatch interpretation
(see Appendix~\ref{app:scaling-sia}).

\begin{table*}[htbp!]
\caption{\textbf{Training aligns entropy descent with the true answer.}
We report the correlation between conditional answer entropy and gold surprisal along each trace
(SIA alignment coefficient $\rho_{\mathrm{SIA}}$), averaged by model and dataset.
\bad{Negative} or \weak{near-zero} values indicate failure of alignment.}
\centering
\small
\setlength{\tabcolsep}{4pt}
\resizebox{\textwidth}{!}{%
\begin{tabular}{lcccccc}
\toprule
\textbf{Model} & \textbf{Training} &
\textbf{GSM8K} & \textbf{SVAMP} & \textbf{ARC} & \textbf{MATH} \\
\midrule

Qwen2.5-3B               & Base   & 0.682 {\tiny [0.649, 0.714]} & 0.603 {\tiny [0.512, 0.681]} & 0.344 {\tiny [0.300, 0.386]} & \bad{-0.0975 {\tiny [-0.1939, -0.0042]}} \\
Qwen2.5-3B-it            & SFT    & 0.744 {\tiny [0.717, 0.770]} & 0.835 {\tiny [0.771, 0.893]} & 0.666 {\tiny [0.633, 0.701]} & \weak{0.0082 {\tiny [-0.0979, 0.1146]}} \\
Qwen2.5-Math-1.5B        & SFT    & 0.499 {\tiny [0.429, 0.567]} & 0.802 {\tiny [0.643, 0.922]} & 0.701 {\tiny [0.654, 0.747]} & \bad{-0.5411 {\tiny [-0.6248, -0.4579]}} \\
DeepSeek-Chat-7B         & SFT    & 0.346 {\tiny [0.305, 0.388]} & 0.295 {\tiny [0.190, 0.414]} & 0.143 {\tiny [0.090, 0.199]} & \bad{-0.1277 {\tiny [-0.1879, -0.0703]}} \\
DeepSeek-R1-Distilled    & SFT+RL & 0.795 {\tiny [0.763, 0.829]} & 0.593 {\tiny [0.436, 0.730]} & 0.783 {\tiny [0.744, 0.817]} & \bad{-0.7746 {\tiny [-0.8462, -0.6865]}} \\
Gemma-2-2B               & Base   & \bad{-0.530 {\tiny [-0.689, -0.346]}} & 0.169 {\tiny [-0.489, 0.844]} & \bad{-0.208 {\tiny [-0.266, -0.151]}} & \bad{-0.5595 {\tiny [-0.7502, -0.3559]}} \\
Gemma-2-2B-it            & SFT    & 0.522 {\tiny [0.475, 0.565]} & 0.462 {\tiny [0.282, 0.629]} & 0.578 {\tiny [0.533, 0.623]} & 0.2143 {\tiny [0.1197, 0.3059]} \\
LLaMA-3.2-3B             & Base   & \bad{-0.361 {\tiny [-0.438, -0.287]}} & 0.424 {\tiny [0.191, 0.621]} & \bad{-0.366 {\tiny [-0.422, -0.314]}} & \bad{-0.2038 {\tiny [-0.3265, -0.0789]}} \\
LLaMA-3.2-3B-it          & SFT    & 0.576 {\tiny [0.536, 0.616]} & 0.399 {\tiny [0.114, 0.655]} & 0.545 {\tiny [0.500, 0.591]} & \weak{0.0394 {\tiny [-0.0341, 0.1169]}} \\
Qwen2.5-14B              & Base   & 0.783 {\tiny [0.754, 0.813]} & 0.663 {\tiny [0.568, 0.755]} & 0.698 {\tiny [0.666, 0.728]} & 0.2880 {\tiny [0.2206, 0.3496]} \\
Qwen2.5-14B-it           & SFT    & 0.867 {\tiny [0.848, 0.886]} & 0.933 {\tiny [0.883, 0.967]} & 0.849 {\tiny [0.824, 0.870]} & 0.7256 {\tiny [0.6503, 0.7913]} \\
Gemma-2-9B               & Base   & \bad{-0.297 {\tiny [-0.382, -0.204]}} & 0.302 {\tiny [-0.042, 0.612]} & \bad{-0.097 {\tiny [-0.147, -0.041]}} & \bad{-0.6174 {\tiny [-0.7382, -0.4759]}} \\
Gemma-2-9B-it            & SFT    & 0.641 {\tiny [0.582, 0.698]} & 0.466 {\tiny [0.235, 0.688]} & 0.793 {\tiny [0.756, 0.829]} & 0.5094 {\tiny [0.3815, 0.6357]} \\
LLaMA-3.1-8B             & Base   & \bad{-0.315 {\tiny [-0.396, -0.233]}} & 0.335 {\tiny [0.121, 0.530]} & \bad{-0.028 {\tiny [-0.082, 0.030]}} & \bad{-0.3929 {\tiny [-0.5122, -0.2835]}} \\
LLaMA-3.1-8B-it          & SFT    & 0.755 {\tiny [0.720, 0.786]} & 0.766 {\tiny [0.476, 0.969]} & 0.712 {\tiny [0.679, 0.743]} & 0.4869 {\tiny [0.4140, 0.5669]} \\
Yi-1.5-34B               & Base   & \bad{-0.005 {\tiny [-0.162, 0.163]}} & 0.267 {\tiny [-0.343, 0.827]} & 0.510 {\tiny [0.462, 0.564]} & \bad{-0.2282 {\tiny [-0.3701, -0.0900]}} \\
\bottomrule
\end{tabular}
}
\label{tab:sia_alignment_by_dataset}
\end{table*}

\subsection{Observable signatures: early lock-in and separability}
\label{sec:signatures}

\begin{figure*}[t]
\centering

\begin{subfigure}[t]{0.49\textwidth}
    \centering
    \includegraphics[width=\linewidth]{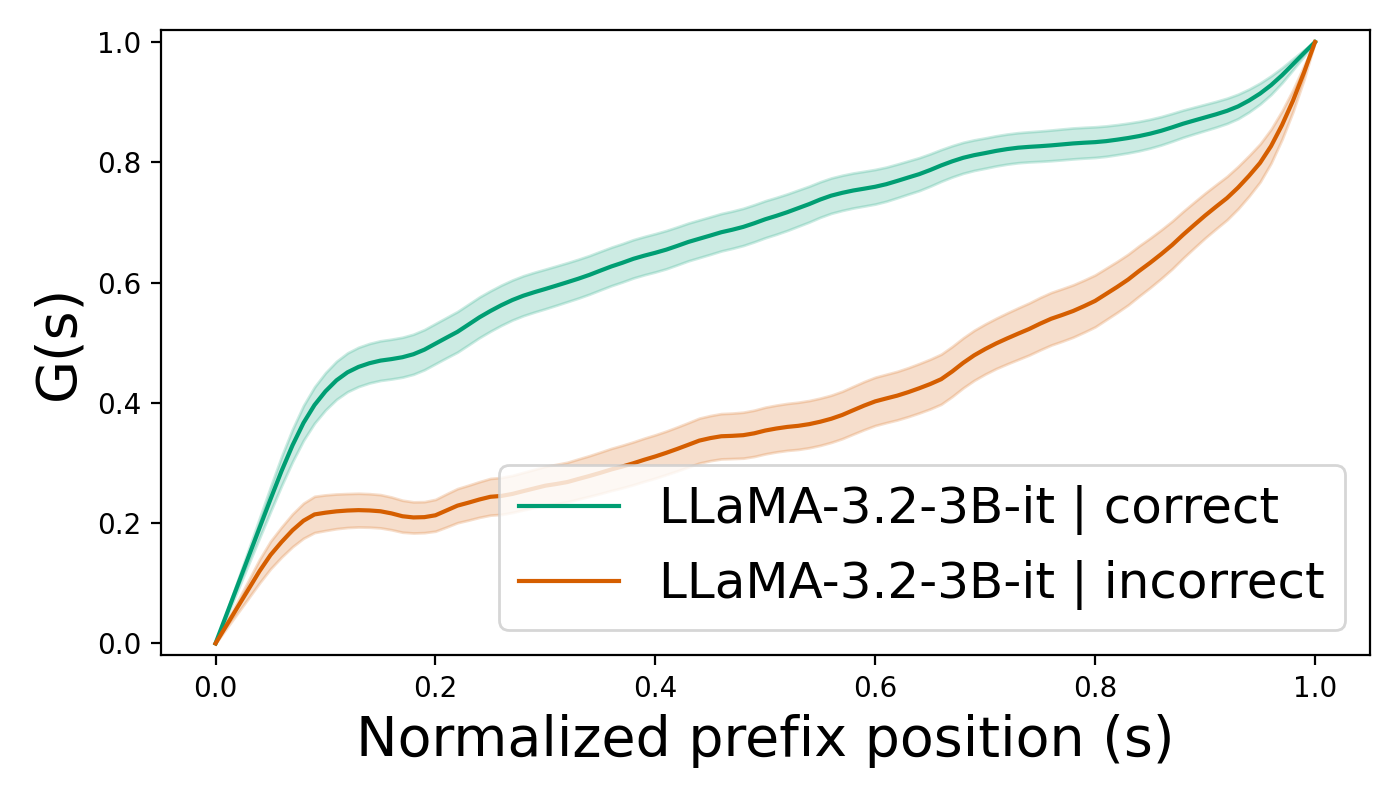}
    \caption{Early information accumulation.}
    \label{fig:gain_correct_incorrect}
\end{subfigure}
\hfill
\begin{subfigure}[t]{0.49\textwidth}
    \centering
    \includegraphics[width=\linewidth]{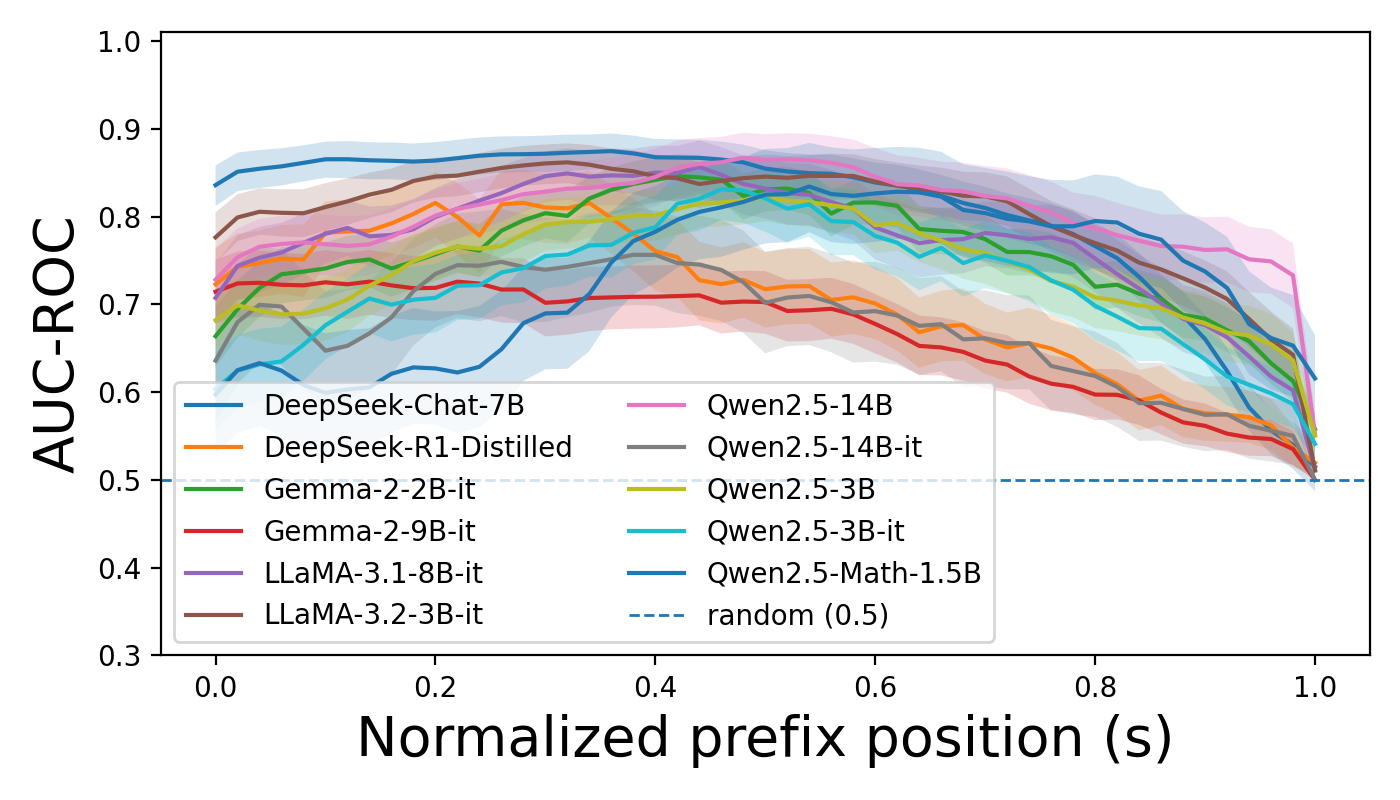}
    \caption{Early separability.}
    \label{fig:auc_by_r}
\end{subfigure}

\caption{\textbf{Observable signatures of entropy-correctness alignment on GSM8K.}
Correct traces accumulate answer-relevant information earlier and conditional answer entropy separates correct from incorrect traces before answer generation. Further results can be found in Appendix~\ref{app:further_results}.}
\label{fig:sia_signatures}
\end{figure*}

We next investigate two observable signatures that may accompany entropy-correctness alignment.

\paragraph{Early information accumulation.}
Figure~\ref{fig:gain_correct_incorrect} plots a normalized version of the cumulative information gain, 
\(
G(s):=\frac{I(A;C_{\le s}\mid Q)}{I(A;C_{\le 1}\mid Q)}
\)
split by correctness. In aligned models, correct traces accumulate a larger fraction of their total answer-relevant information earlier in the generation. 

\paragraph{Early separability of correct vs.\ incorrect traces.}
Figure~\ref{fig:auc_by_r} reports the AUC for using conditional entropy at prefix length $s$ to distinguish correct from incorrect traces. For models exhibiting strong entropy-correctness alignment, separability is already strong well before the answer is produced. 

These results characterize reasoning dynamics in entropy-correctness aligned settings. Importantly, all signatures vanish or weaken when alignment is absent (see Appendix~\ref{sec:signature_appendix}). 

\section{Conclusion and Open Questions}
\label{sec:open_quest}

This work provides insight into when internal entropy dynamics can reliably track reasoning correctness in autoregressive language models. We formalized the Stepwise Informativeness Assumption (SIA), which identifies the condition under which reductions in predictive answer entropy are aligned, in expectation, with increasing probability assigned to valid answers. We then studied how this alignment might emerge through training, showing that next-token prediction can transfer prefix-answer dependencies from training data to the model, while their alignment with correctness depends on the training regime. Finally, we empirically investigated observable consequences of SIA, finding that entropy-correctness alignment strengthens with correctness-oriented post-training.

Lastly, some open questions remain. Entropy-based diagnostics may fail in regimes where reasoning-trace prefixes are only weakly informative about the true answer: characterizing the distributions that produce such behavior would clarify the limits of entropy as a proxy for reasoning. Also, it remains open whether targeted interventions that modify entropy dynamics can reliably change reasoning outcomes. Finally, an important direction is to generalize entropy-based diagnostics to other modalities and generative modeling paradigms.

\section*{Limitations}

Our work has several limitations that also suggest directions for future research.

First, our empirical evaluation does not test SIA directly. Doing so would require estimating conditional expectations over entropy-reducing extensions at each reasoning step, which is computationally prohibitive at the scale of our experiments. Instead, we use checkpoint-based measurements to assess trajectory-level entropy-correctness alignment, which provides a practical way to investigate the relationship SIA formalizes across models and training regimes.

Second, our theoretical analysis does not formally derive SIA from a particular training objective. Establishing such a result would require characterizing how complex training dynamics shape the internal behavior of large language models, which remains a challenging problem even for considerably simpler properties. Instead, our analysis builds on the structure of reasoning data and the properties of maximum-likelihood training to explain how prefix-answer dependencies can be learned and why correctness-oriented post-training may promote entropy-correctness alignment. This provides a theoretically grounded account of the role of training that is complemented by our empirical findings.

Third, our evaluation focuses on open-weight language models and reasoning benchmarks with identifiable correct answers. This setting allows us to measure changes in probability assigned to correct answers but does not establish whether similar relationships hold for more open-ended tasks.

Finally, our results characterize associations between entropy dynamics and correctness. Whether interventions on entropy dynamics can reliably improve reasoning outcomes remains an important direction for future work.

\section*{Acknowledgments}

Mar Gonzàlez I Català acknowledges that this project was supported by G-Research. George D. Montañez acknowledges support from the William Whewell Centre for Science and Natural Theology and the Global Scholars Foundation.

\bibliography{custom}

\appendix

\section{Experimental setup and evaluation protocol}
\label{app:exp}

\subsection{Tasks and datasets}
\label{subsec:tasks_datasets}

We focus on reasoning tasks with a discrete answer space \(\mathcal{A}\), which enables empirical estimation of conditional answer entropy. Each example consists of a question \(Q \in \mathcal{Q}\) and a ground-truth answer \(A \in \mathcal{A}\). We evaluate on the following datasets:

\begin{itemize}
\item \textbf{GSM8K} \cite{cobbe2021trainingverifierssolvemath}: grade-school mathematical word problems with numeric answers.
\item \textbf{ARC} \cite{clark2018thinksolvedquestionanswering}: multiple-choice science questions.
\item \textbf{SVAMP} \cite{patel2021nlpmodelsreallyable}: arithmetic word problems designed to test robustness to linguistic variation.
\item \textbf{MATH} \cite{hendrycks2021measuringmathematicalproblemsolving}: competition-level mathematics (we use algebra track).
\end{itemize}

For all datasets, we use the official test splits and apply deterministic answer normalization and parsing to map model outputs to discrete answer labels (e.g., numeric normalization for GSM8K, SVAMP and MATH, letter-to-option mapping for ARC). Invalid or unparsable outputs are mapped to a special null answer category.

\subsection{Models}
\label{subsec:models}

We evaluate a diverse set of open-weight LLMs corresponding to different training regimes:

\begin{itemize}
\item \textbf{Gemma-2-2B, Gemma-2-9B} \cite{gemmateam2024gemma2improvingopen}: base and instruction-tuned variants.
\item \textbf{LLaMA-3.2-3B, LLaMA-3.1-8B} \cite{grattafiori2024llama3herdmodels}: base and instruction-tuned variants.
\item \textbf{Qwen-2.5-3B, Qwen-2.5-14B} \cite{qwen2025qwen25technicalreport}: base and instruction-tuned variants.
\item \textbf{Qwen-2.5-Math-1.5B} \cite{qwen2025qwen25technicalreport}: SFT-trained specialized on math problems.
\item \textbf{DeepSeek-Chat-7B} \cite{deepseekai2024deepseekllmscalingopensource}: SFT-trained chat model.
\item \textbf{DeepSeek-R1-distilled-7B} \cite{Guo_2025}: reasoning-specialized RL model.
\item \textbf{Yi-1.5-34B} \cite{ai2025yiopenfoundationmodels}: base variant.
\end{itemize}

Base models correspond to pretrained LLMs without supervised or reinforcement fine-tuning. Instruction-tuned (IT) models are supervised fine-tuned on instruction-following data. RL-trained models are optimized using reinforcement learning from human or synthetic feedback.

\subsection{Generation procedure}
\label{subsec:generation_procedure}

For each question \(Q=q\), we sample one independent reasoning trajectory from the model under a fixed stochastic decoding configuration (temperature, nucleus sampling, and maximum generation length). Concretely, we draw
\[
C_{1:K} \sim p_\theta(\cdot \mid q),
\]
where \(K\) denotes the generated reasoning length (up to a fixed truncation limit). We treat this sampled trajectory \(C_{1:K}\) as one realization of the model's reasoning process for the given query.

Unless otherwise specified, decoding uses:
\begin{itemize}
    \item temperature $T=0.7$
    \item nucleus sampling with $p=0.9$
    \item a maximum generation length of 600 tokens
\end{itemize}

All rollouts used for entropy estimation use the same decoding configuration to ensure comparability.

\subsection{Monte-Carlo estimation of conditional answer entropy}
\label{subsubsec:mc_conditional_entropy}

Given a fixed query \(Q = q\) and a realized reasoning prefix \(C_{1:k} = c_{1:k}\), the model induces an implicit distribution over final answers
\[
p_\theta(A \mid q, c_{1:k}),
\]
We approximate \(H_{p_\theta}(A \mid q, c_{1:k})\) using Monte-Carlo sampling. For a fixed prefix \((q, c_{1:k})\), we draw \(N\) independent stochastic rollouts from the model:
\[
A^{(i)} \sim p_\theta(\cdot \mid q, c_{1:k}), \qquad i = 1, \dots, N.
\]
using the same decoding parameters as the base generation, followed by deterministic answer extraction.

These samples induce an empirical distribution
\[
\hat{p}_k(a \mid q, c_{1:k})
= \frac{1}{N} \sum_{i=1}^N \mathbf{1}\{A^{(i)} = a\}.
\]
and the plug-in estimator
\[
\widehat{H}_{p_\theta}(A \mid q, c_{1:k})
= -\mkern-46mu \sum_{a \in \mathcal{A} : \hat{p}_k(a \mid q, c_{1:k}) > 0}
\mkern-46mu 
\hat{p}_k(a \mid q, c_{1:k})
\log \hat{p}_k(a \mid q, c_{1:k}).
\]

This estimator is biased for finite \(N\) but consistent as \(N \to \infty\), and sufficient for comparing entropy trends across token positions and training regimes.

All rollouts are performed in evaluation mode, without gradient computation. Sampling parameters are held fixed across models and prefixes. In practice, we use $N=16$ continuations per prefix unless otherwise stated. For an ablation using $N=32$ continuations, see Appendix~\ref{app:ablations}.

\subsection{Checkpointed prefix evaluation}
\label{subsec:checkpointed_prefix_evaluation}

Estimating conditional answer entropy at every token is computationally expensive. We therefore evaluate at checkpoint positions
\[
\mathcal{K} = \{k_1, k_2, \dots, k_m\} \subseteq \{0,1,\dots,K\},
\]
spaced uniformly at stride $s$, and always including the final prefix length of the trajectory (\(k_m = K\)). Here \(k=0\) corresponds to the empty prefix.

For each \(k \in \mathcal{K}\), we compute \(\widehat{H}_{p_\theta}(A \mid Q, C_{\le k})\) independently. When needed for visualization, we linearly interpolate entropy values between checkpoints, but all reported quantitative results are computed on \(\mathcal{K}\).

\subsection{Statistical reporting}
\label{subsec:statistics}

All reported curves show the mean across questions, with shaded regions denoting nonparametric 95\% bootstrap confidence intervals computed by resampling questions. For scalar metrics reported in tables, we report the mean across questions together with bracketed intervals denoting nonparametric 95\% bootstrap confidence intervals, also computed by resampling questions. These uncertainty estimates capture variability across evaluation questions under the fixed experimental configuration described in Appendix~\ref{app:exp}.

\subsection{Compute resources}
\label{subsec:compute}

All experiments were run on a GPU cluster using NVIDIA 40GB A100 GPUs. Each run used only a single GPU and fit within GPU memory, no model-parallel or multi-GPU execution was required. The reported experiments used compute on the order of $10^3$ A100 GPU-hours.

\section{Complete results for observable signatures of entropy-correctness alignment}
\label{sec:signature_appendix}

This appendix provides the complete results for the trajectory-level patterns investigated in Section~\ref{sec:signatures}.

\subsection{Early information accumulation}
\label{subsec:appendix-earlyinfo}

Normalized cumulative information gain $G(s)$ is reported as a function of relative prefix length $s$, split by correctness. Each panel corresponds to one model-dataset pair. Curves show the mean across evaluation questions, and shaded regions denote 95\% bootstrap confidence intervals.

\paragraph{Aligned model-dataset pairs.}
Figures~\ref{fig:aligned-earlyinfo-gsm8k}-\ref{fig:aligned-earlyinfo-math} report normalized cumulative information gain for the model-dataset pairs classified as aligned.

\begin{figure*}[htbp]
  \centering

  \signaturepanel
    {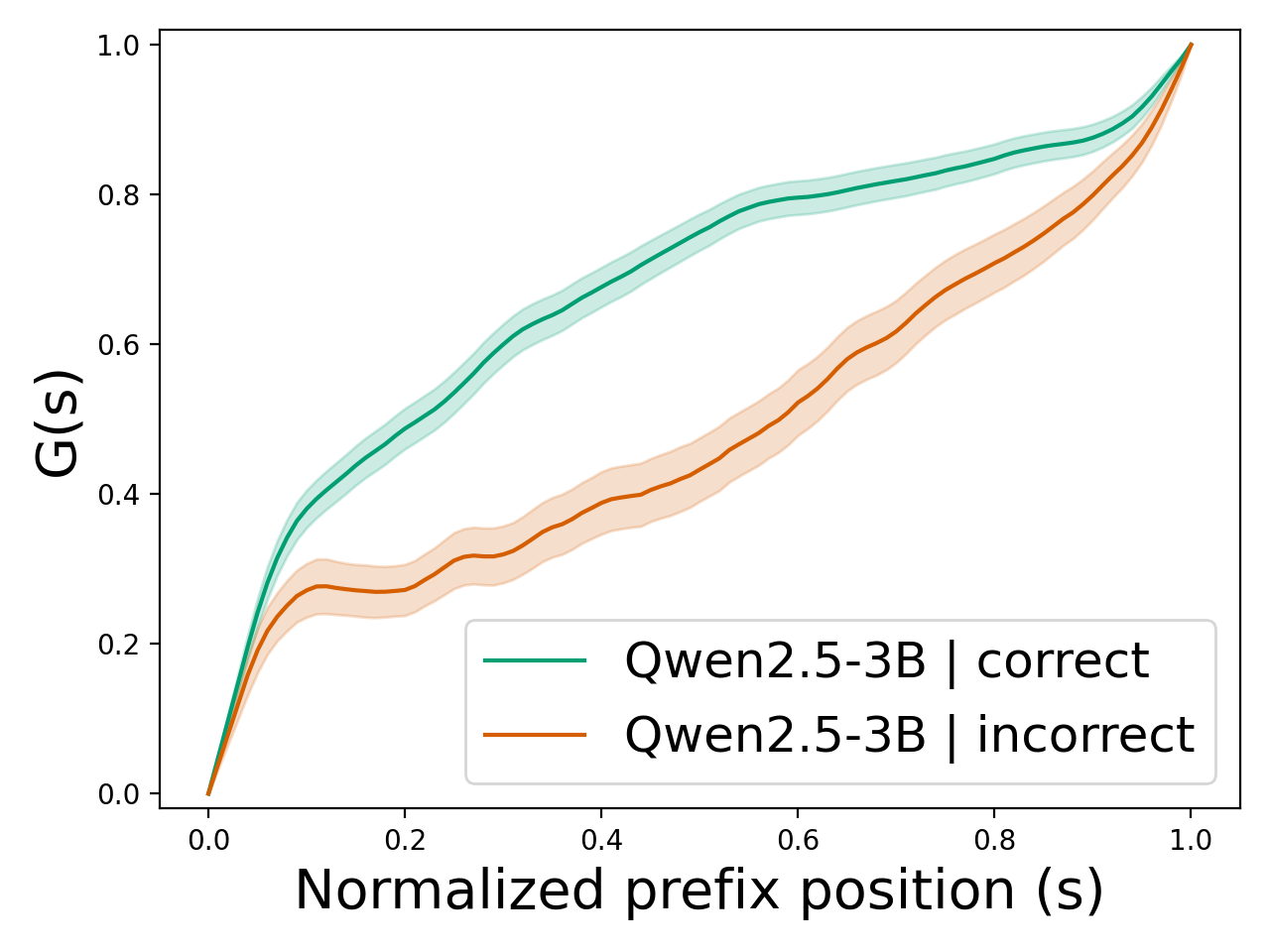}
    {Qwen2.5-3B}%
  \hfill
  \signaturepanel
    {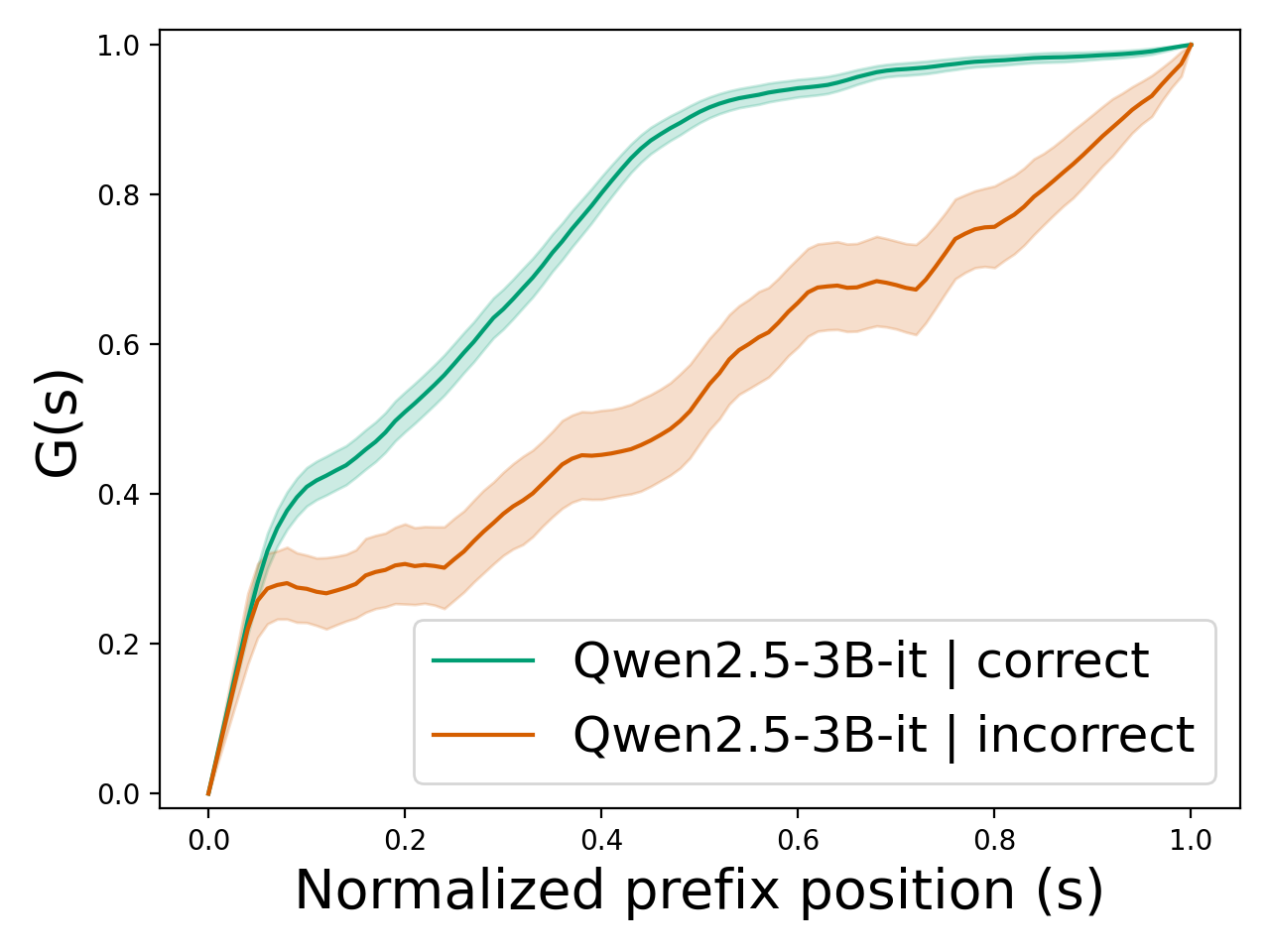}
    {Qwen2.5-3B-IT}%
  \hfill
  \signaturepanel
    {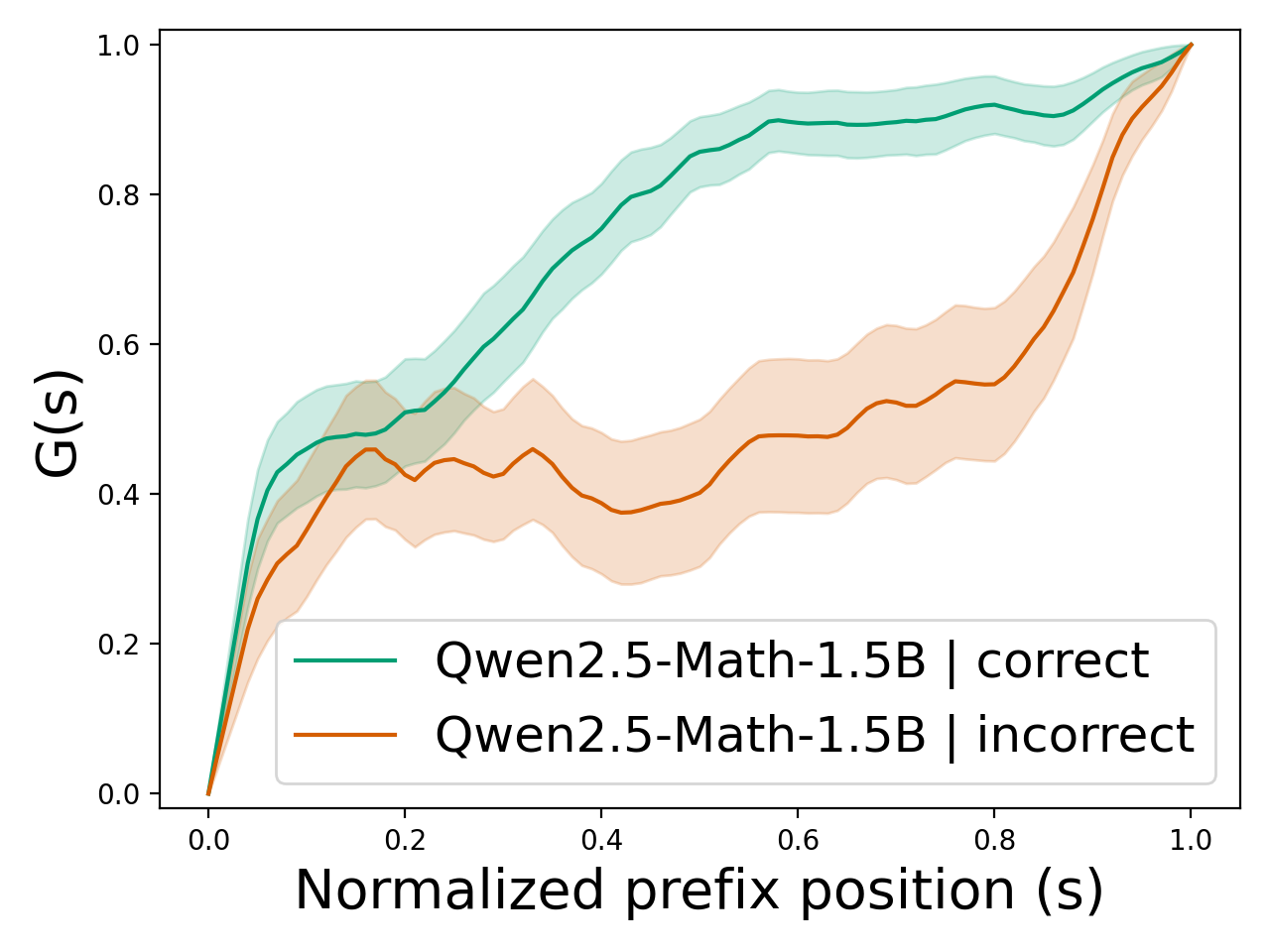}
    {Qwen2.5-Math-1.5B}%
  \hfill
  \signaturepanel
    {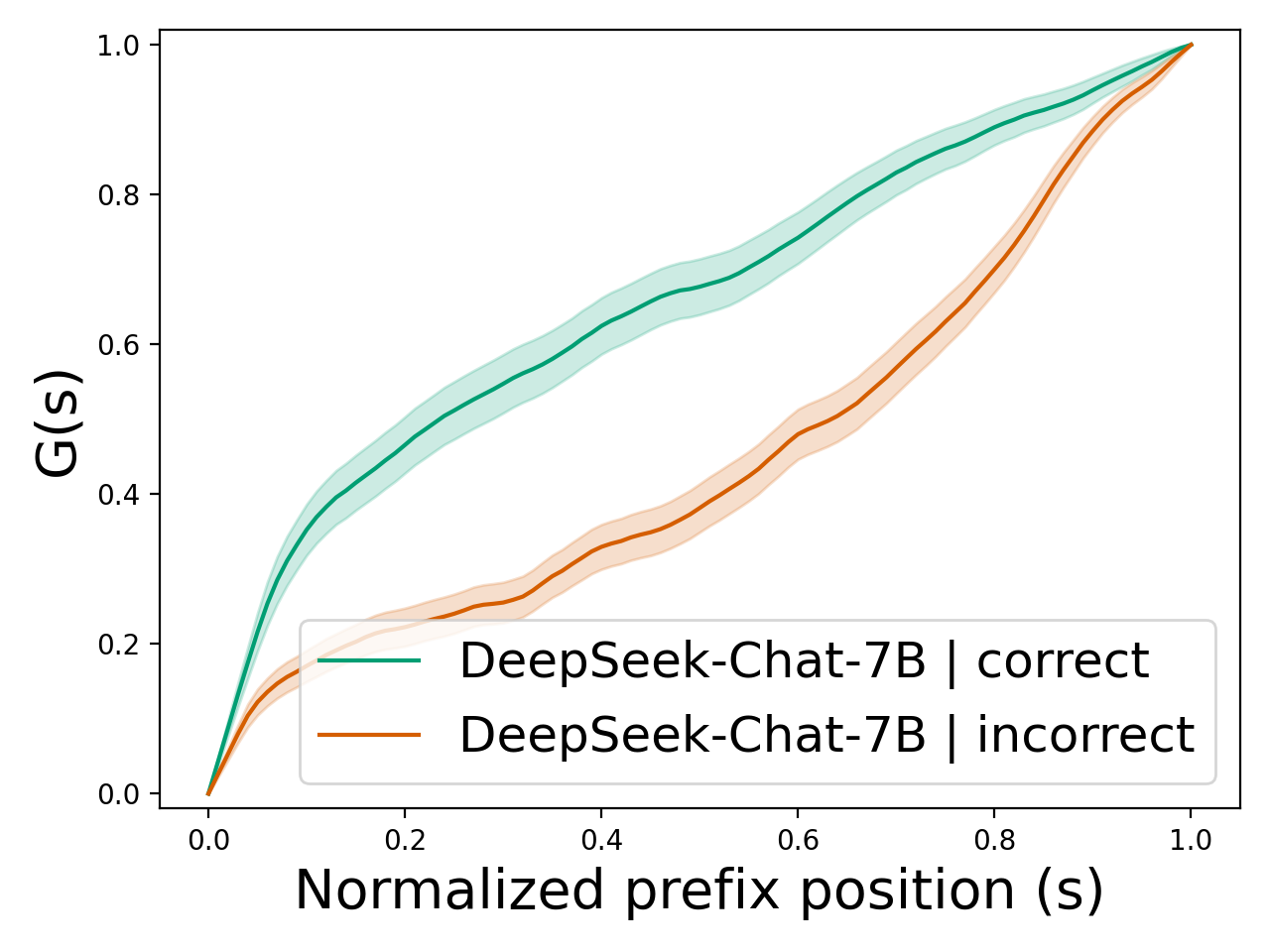}
    {DeepSeek-Chat-7B}%

  \par\medskip

  \signaturepanel
    {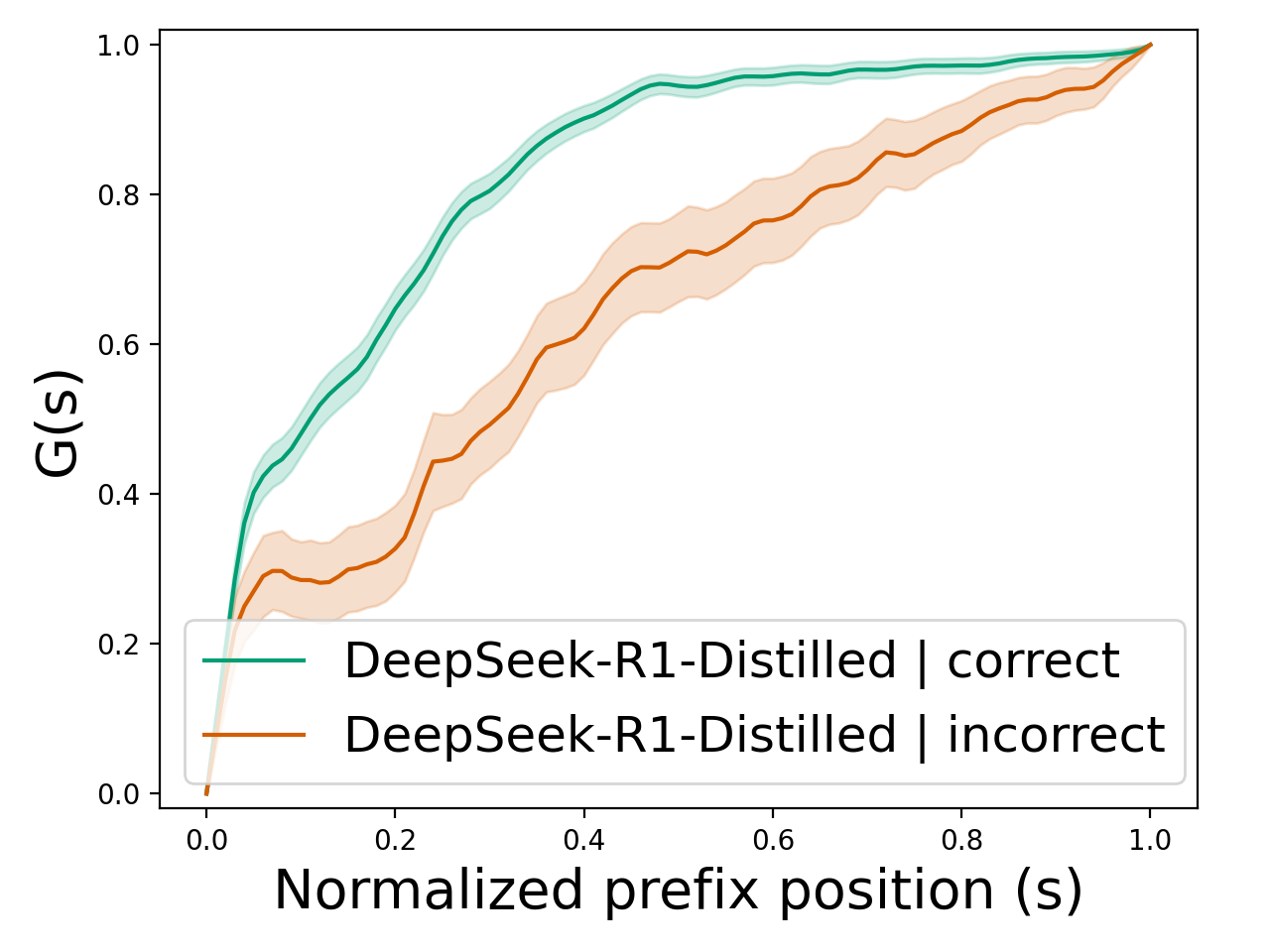}
    {DeepSeek-R1-Distilled}%
  \hfill
  \signaturepanel
    {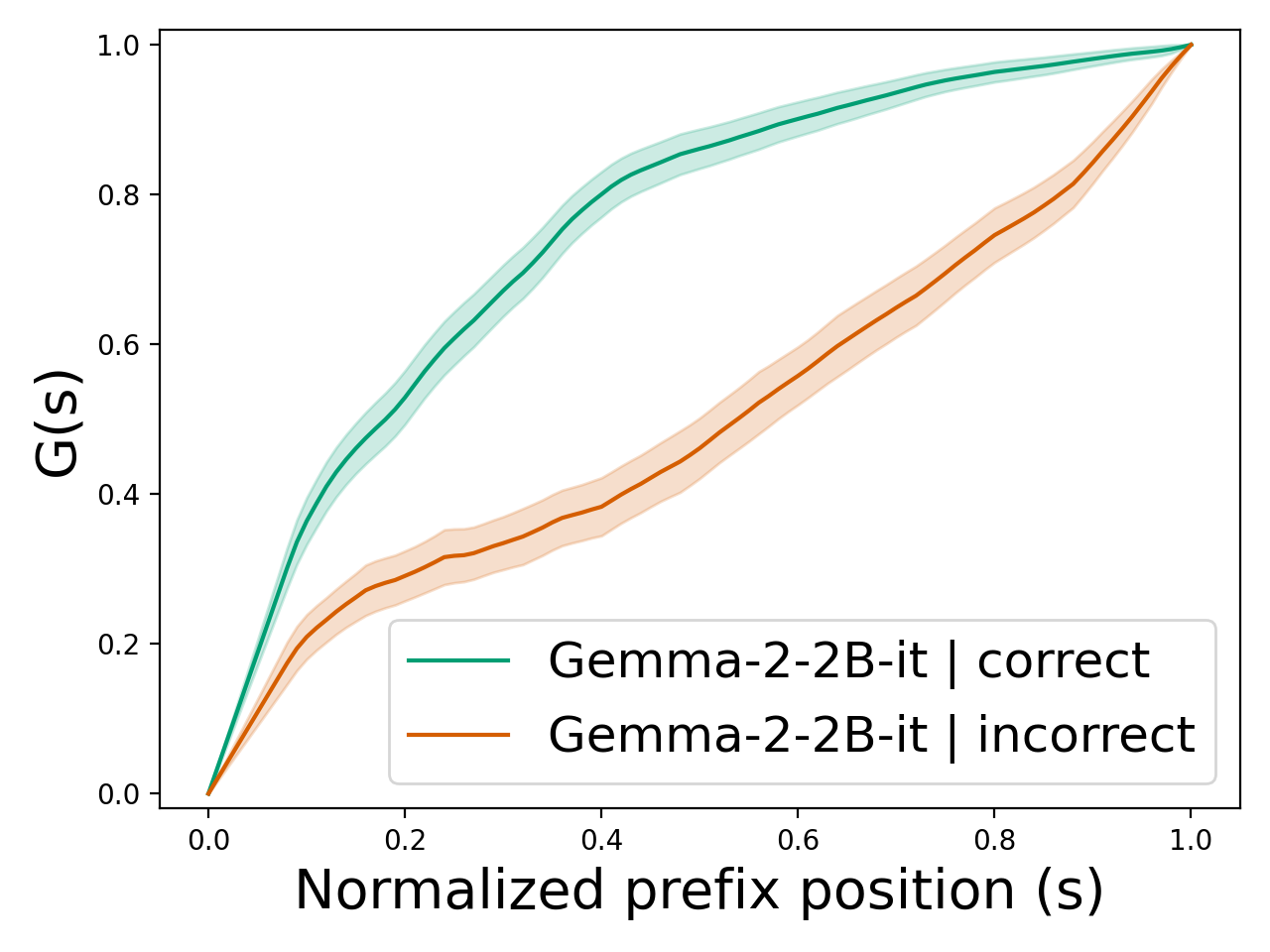}
    {Gemma-2-2B-IT}%
  \hfill
  \signaturepanel
    {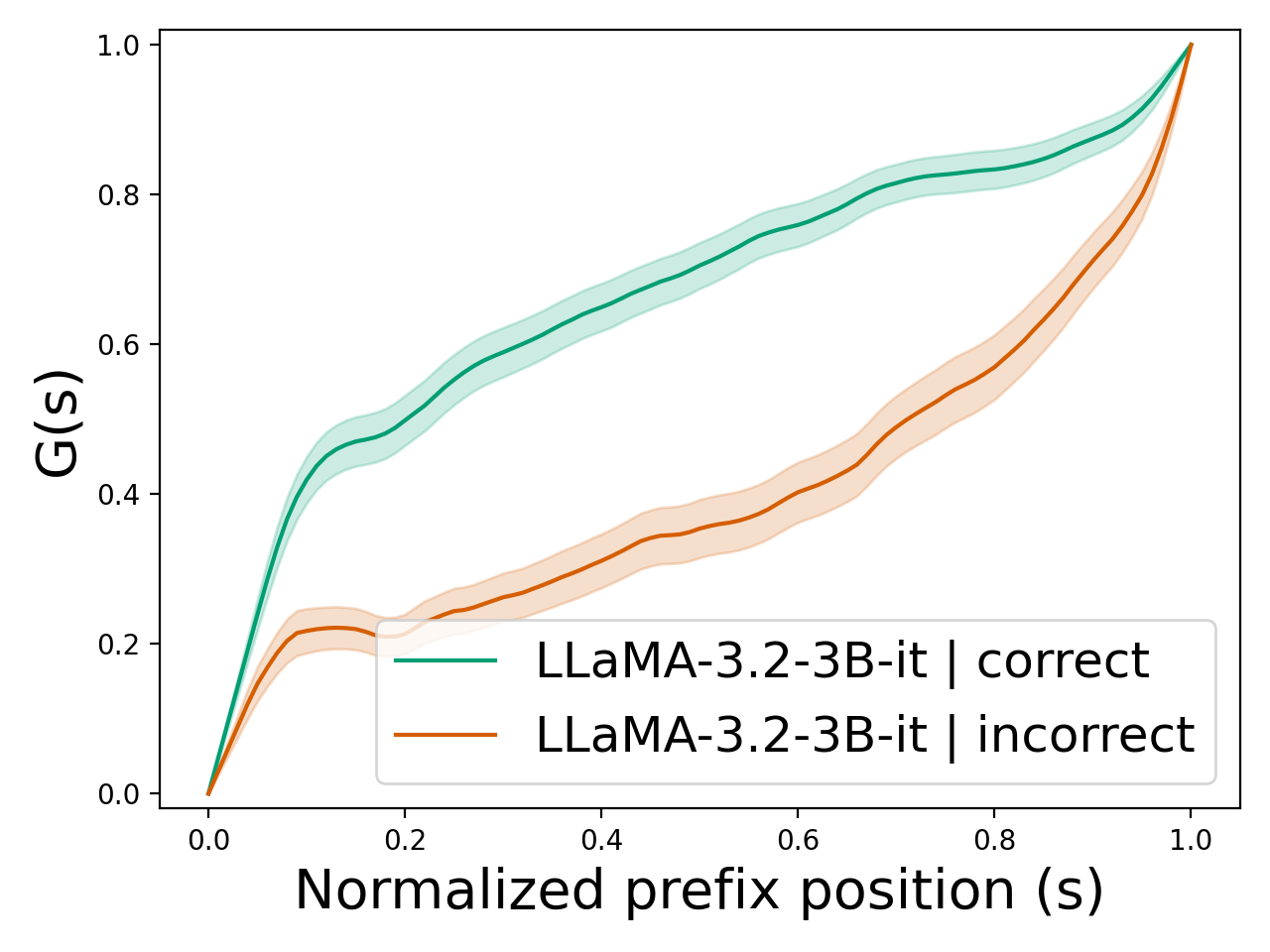}
    {LLaMA-3.2-3B-IT}%
  \hfill
  \signaturepanel
    {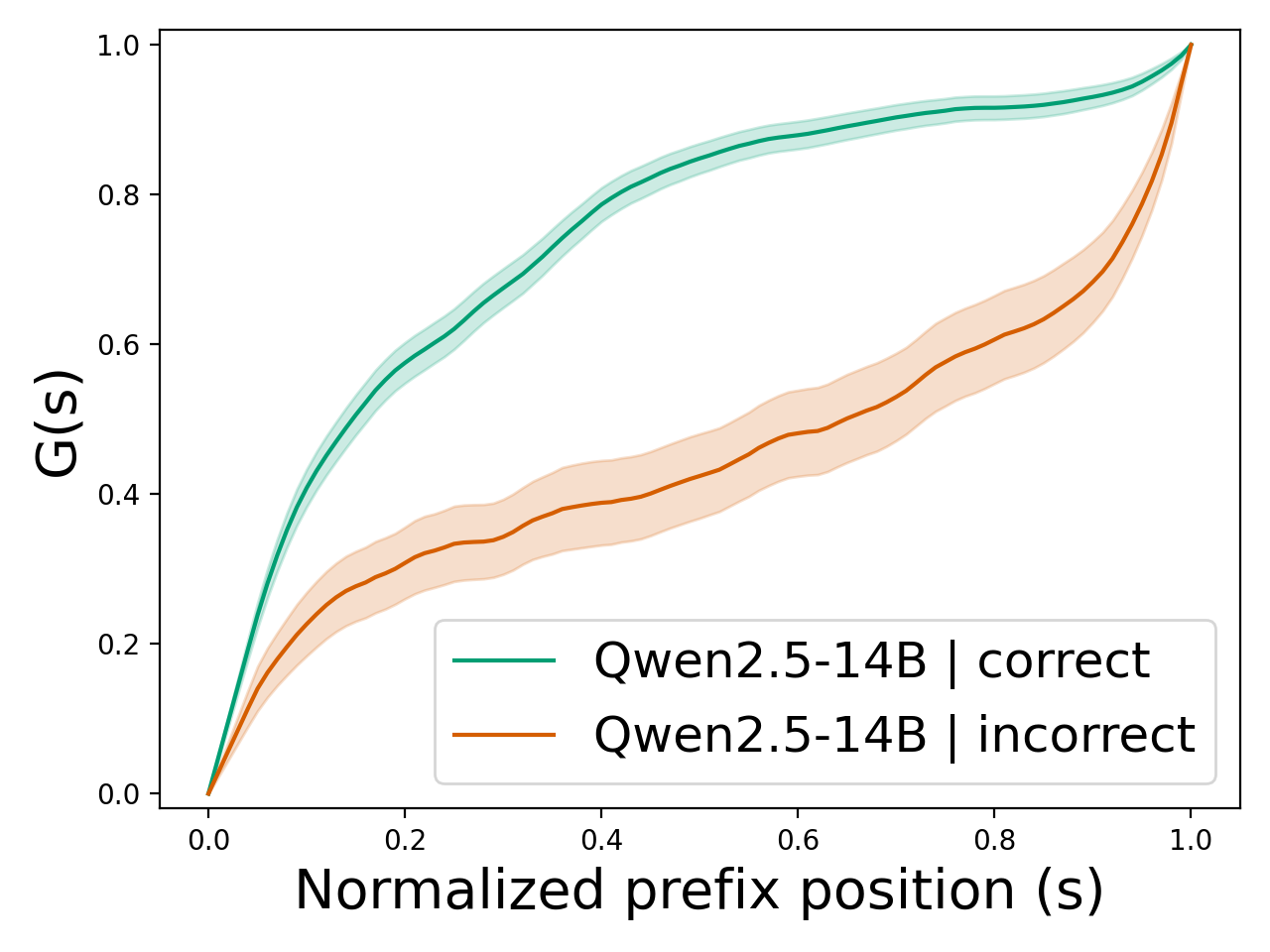}
    {Qwen2.5-14B}%

  \par\medskip

  \makebox[\linewidth][c]{%
    \signaturepanel
      {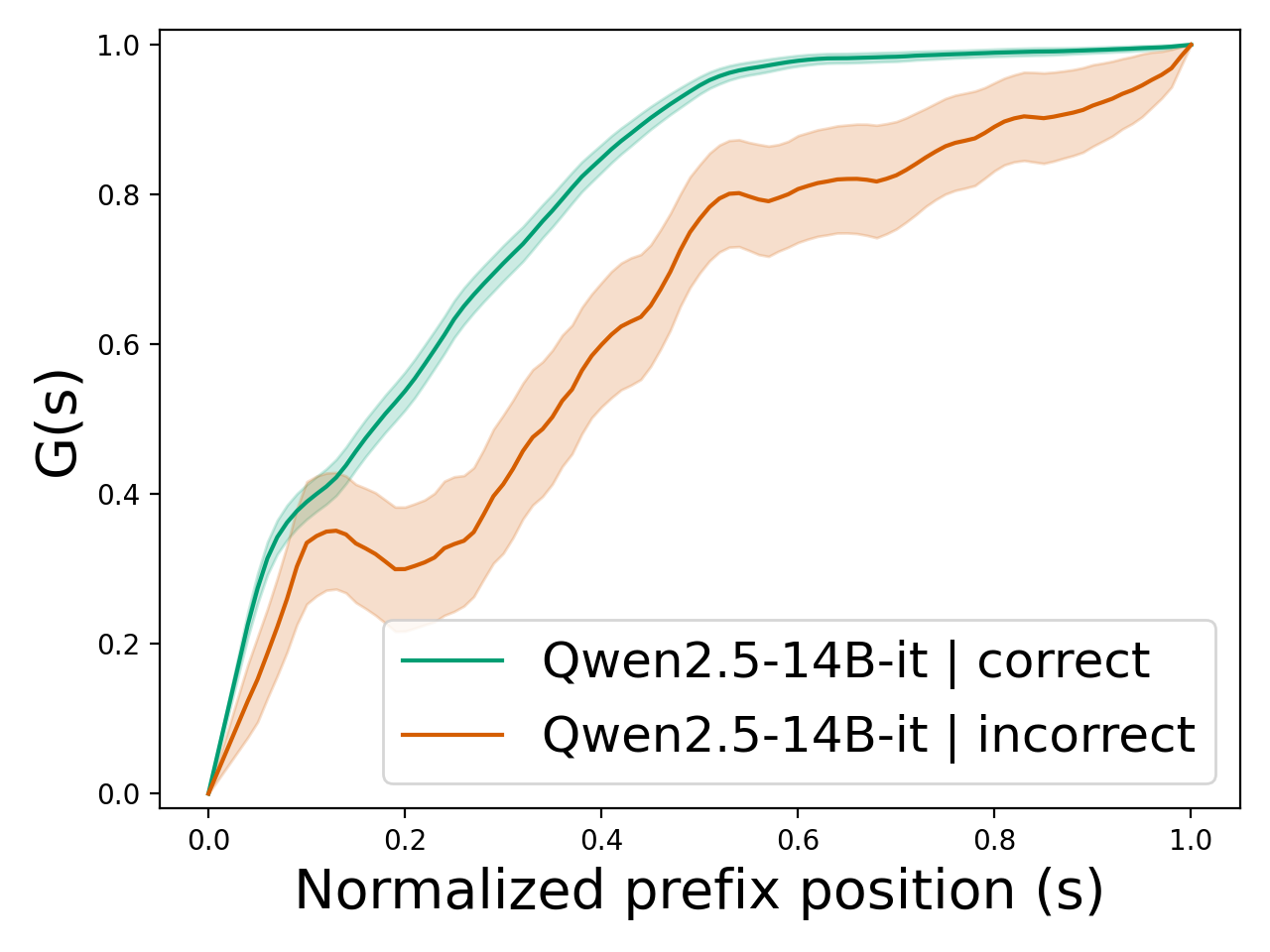}
      {Qwen2.5-14B-IT}%
    \hspace{0.015\linewidth}
    \signaturepanel
      {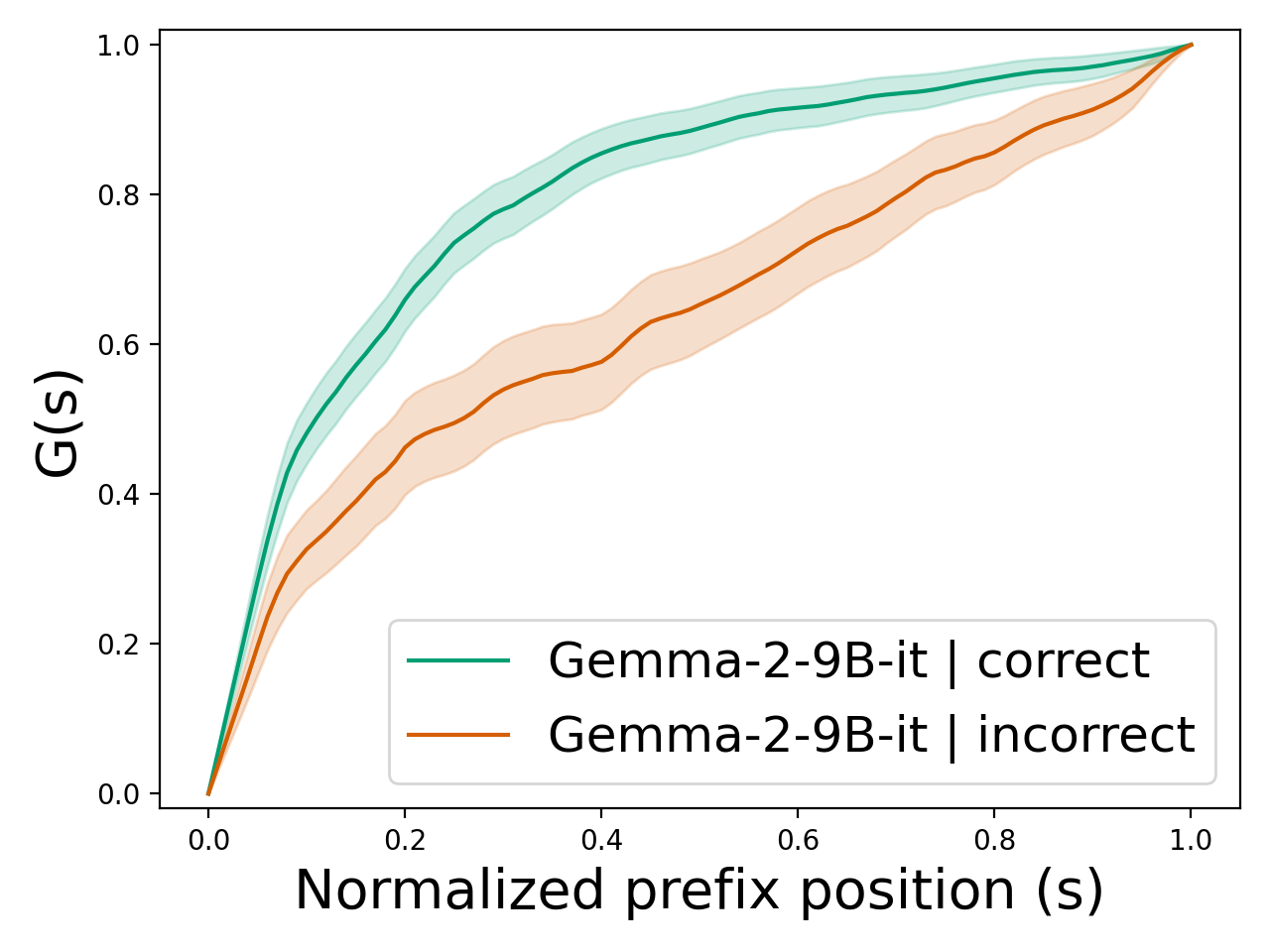}
      {Gemma-2-9B-IT}%
    \hspace{0.015\linewidth}
    \signaturepanel
      {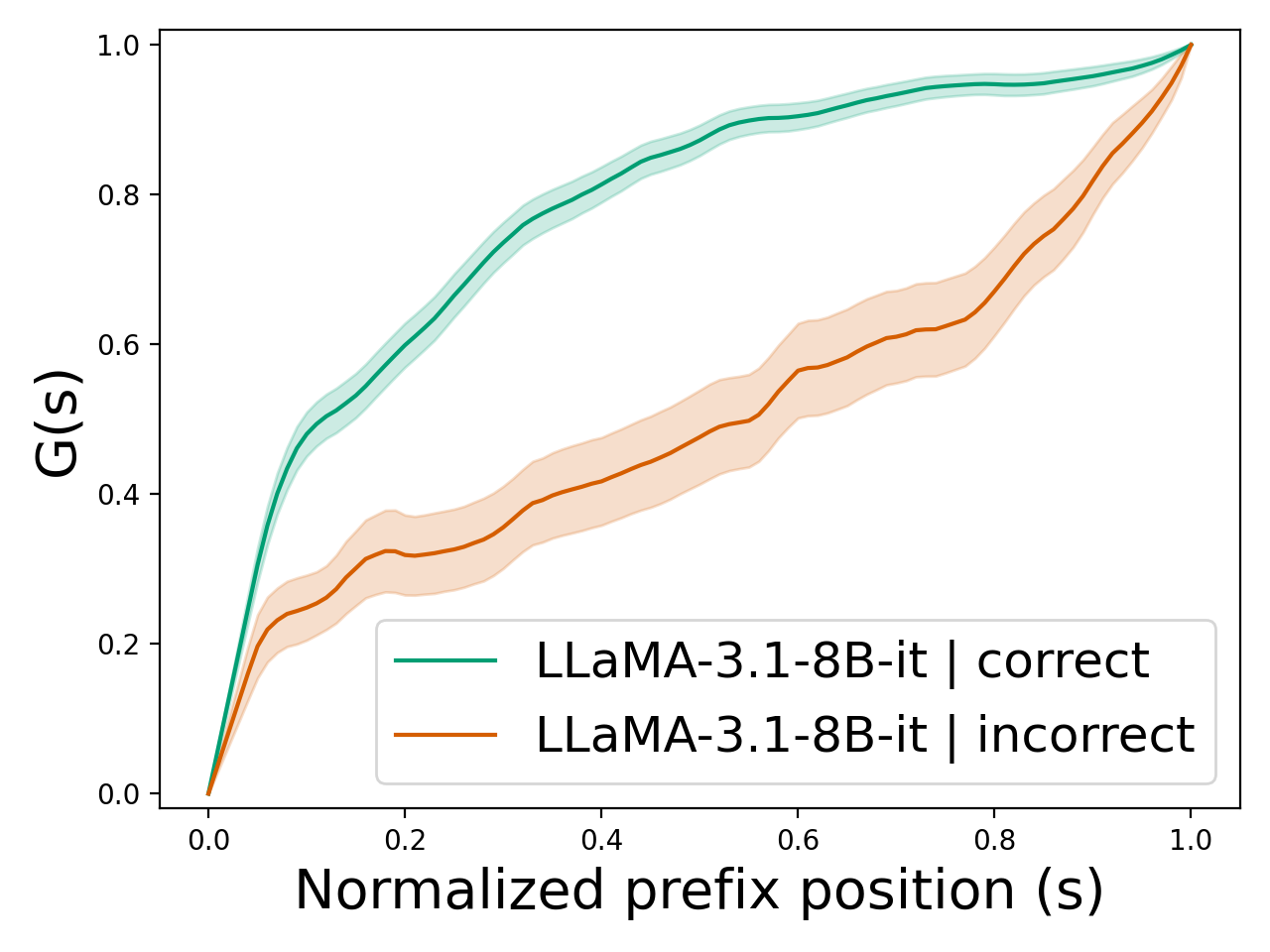}
      {LLaMA-3.1-8B-IT}%
  }

  \caption{\textbf{Early information accumulation for aligned
  model-dataset pairs on GSM8K.}
  Normalized cumulative information gain $G(s)$ versus relative prefix length $s$, split by correctness. Each panel corresponds to one model. Shaded regions denote 95\% bootstrap confidence intervals.}
  \label{fig:aligned-earlyinfo-gsm8k}
\end{figure*}

\begin{figure*}[htbp]
  \centering

  \signaturepanel
    {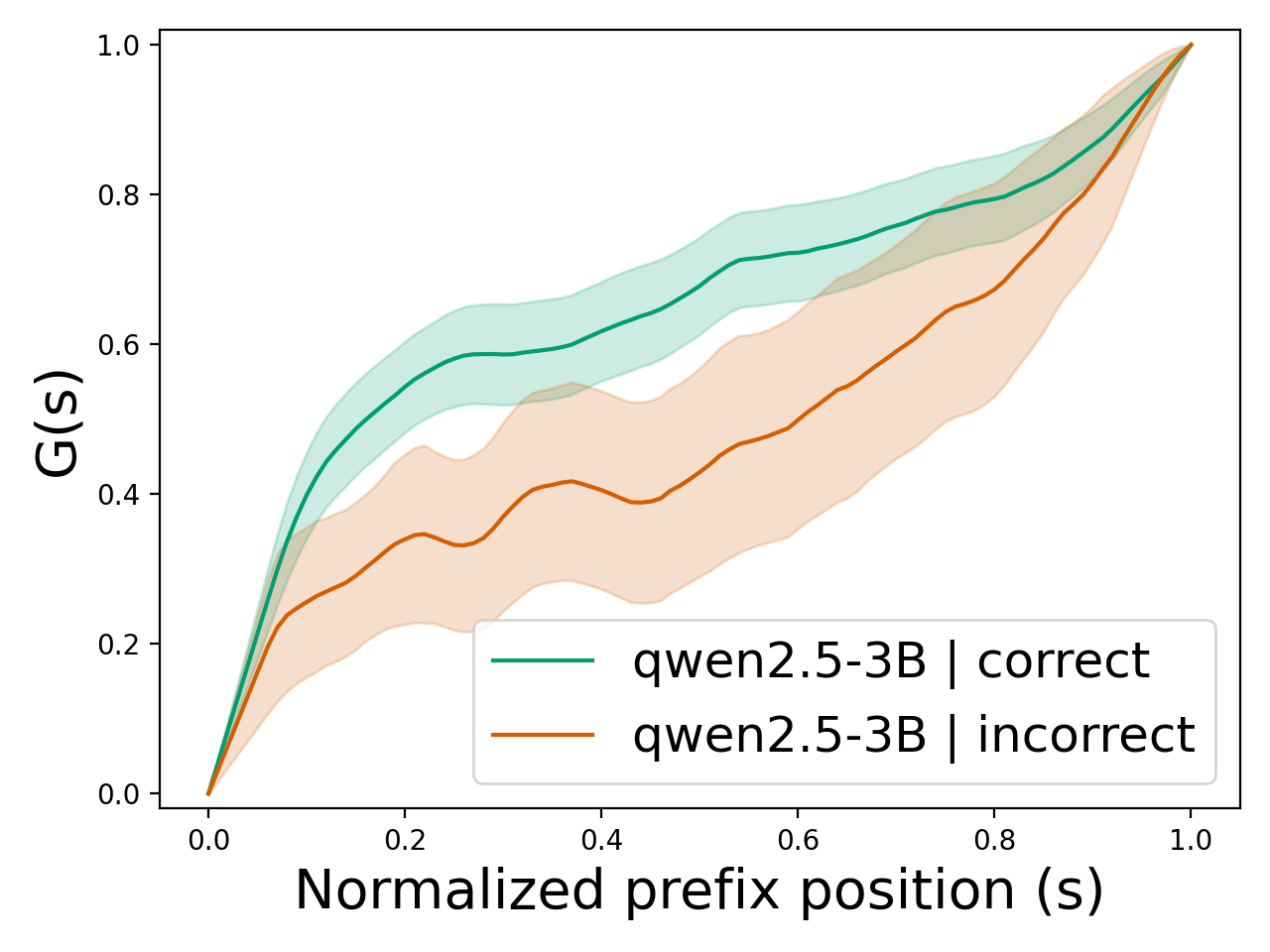}
    {Qwen2.5-3B}%
  \hfill
  \signaturepanel
    {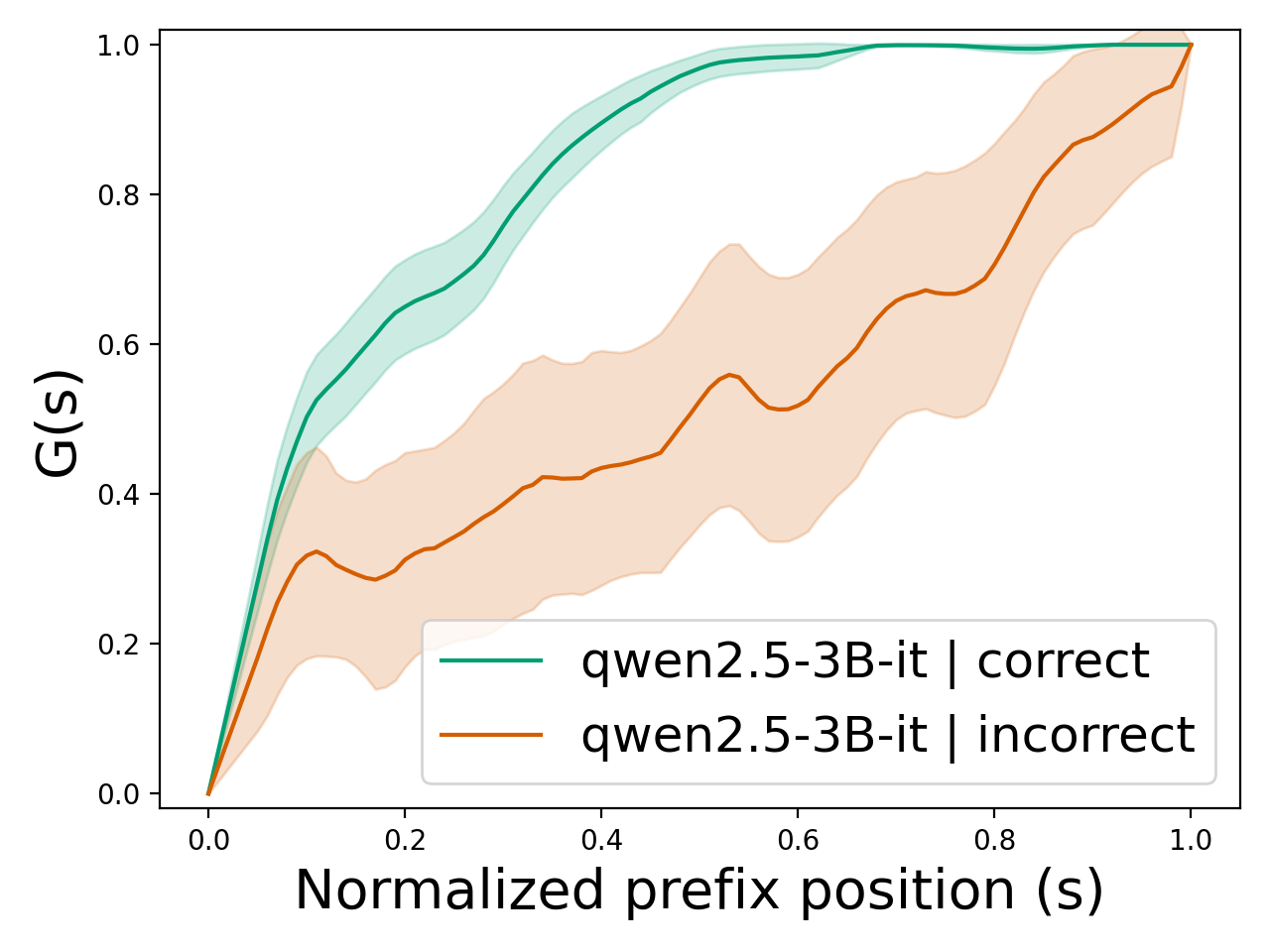}
    {Qwen2.5-3B-IT}%
  \hfill
  \signaturepanel
    {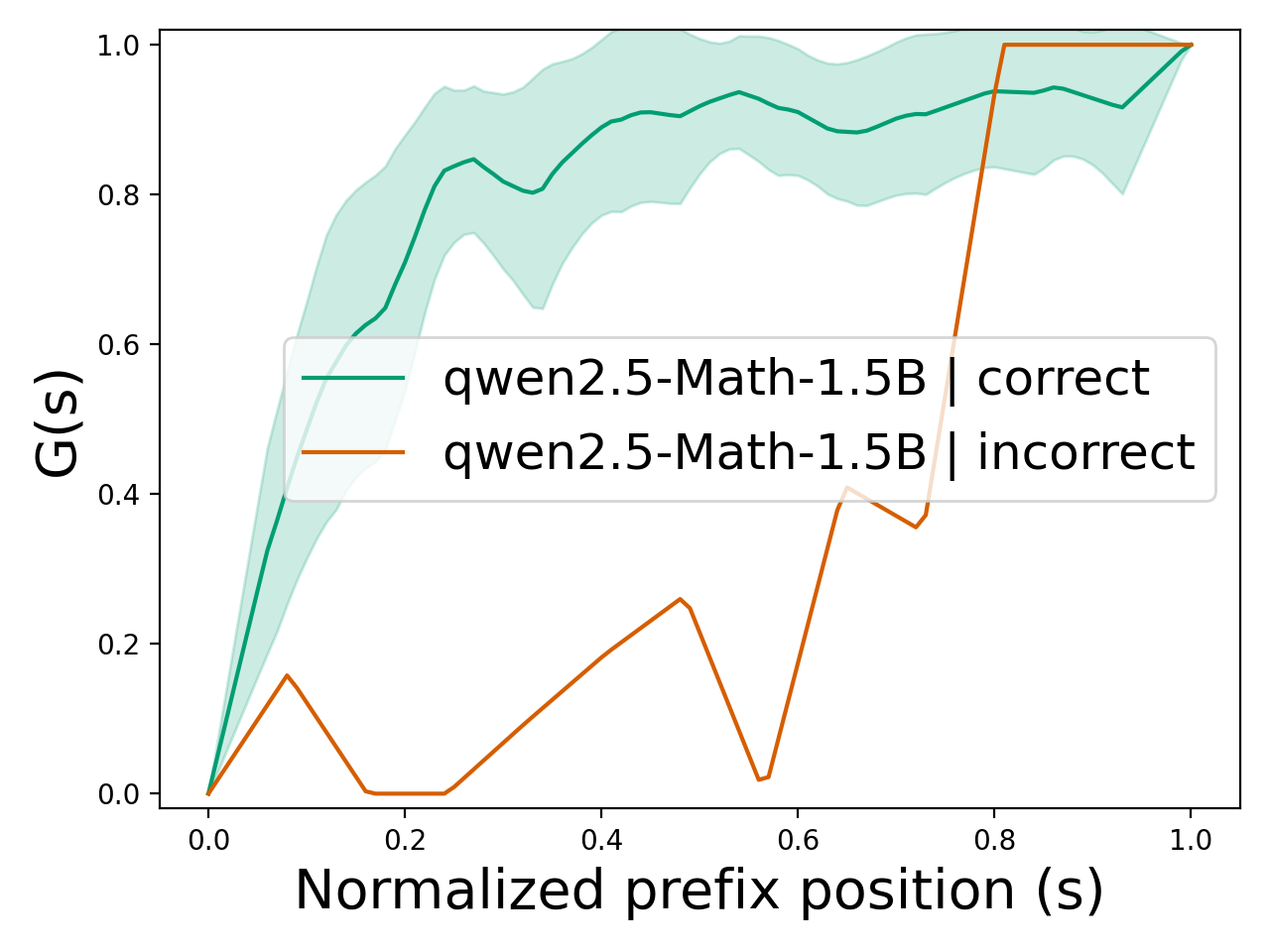}
    {Qwen2.5-Math-1.5B}%
  \hfill
  \signaturepanel
    {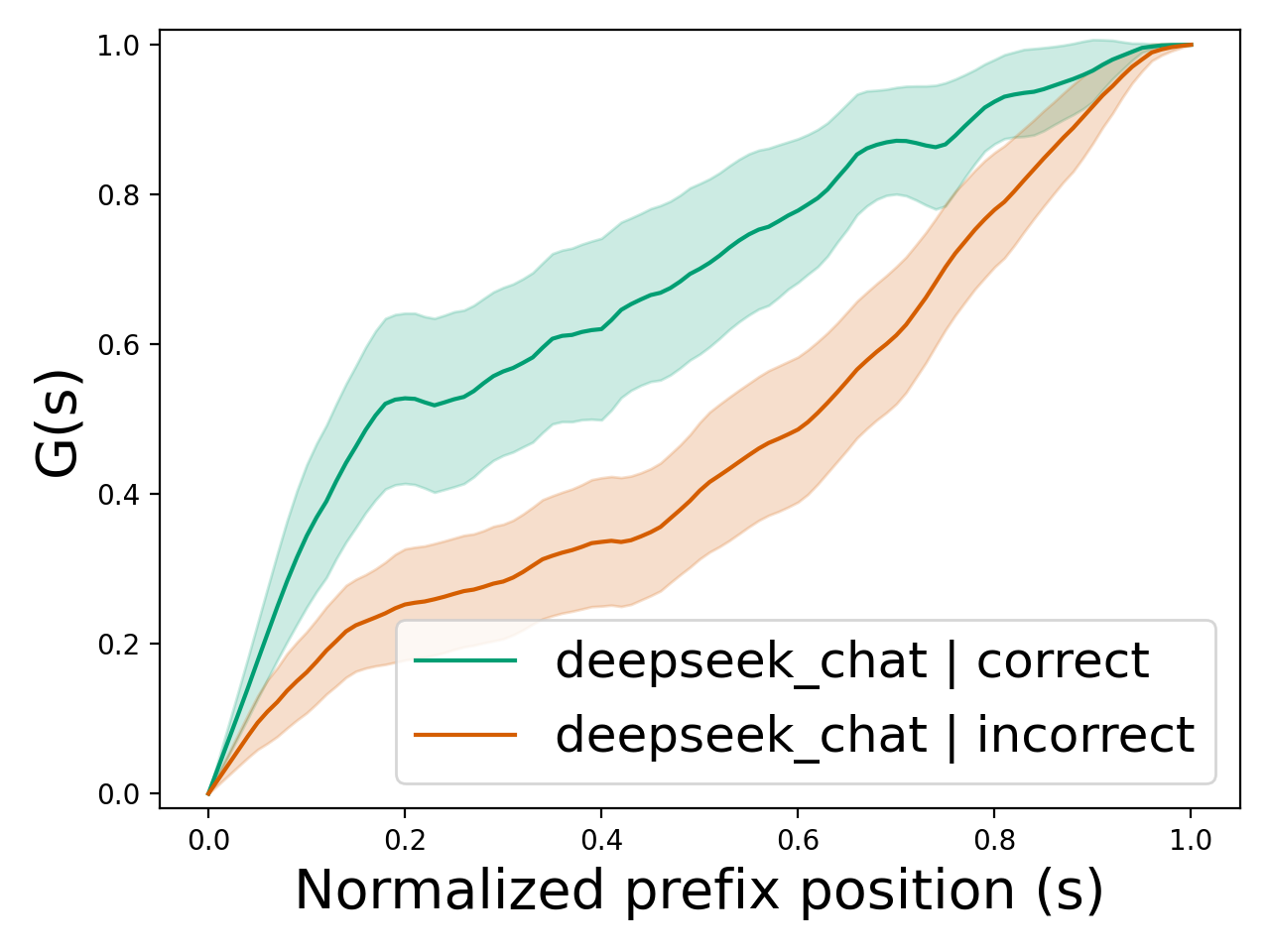}
    {DeepSeek-Chat-7B}%

  \par\medskip

  \signaturepanel
    {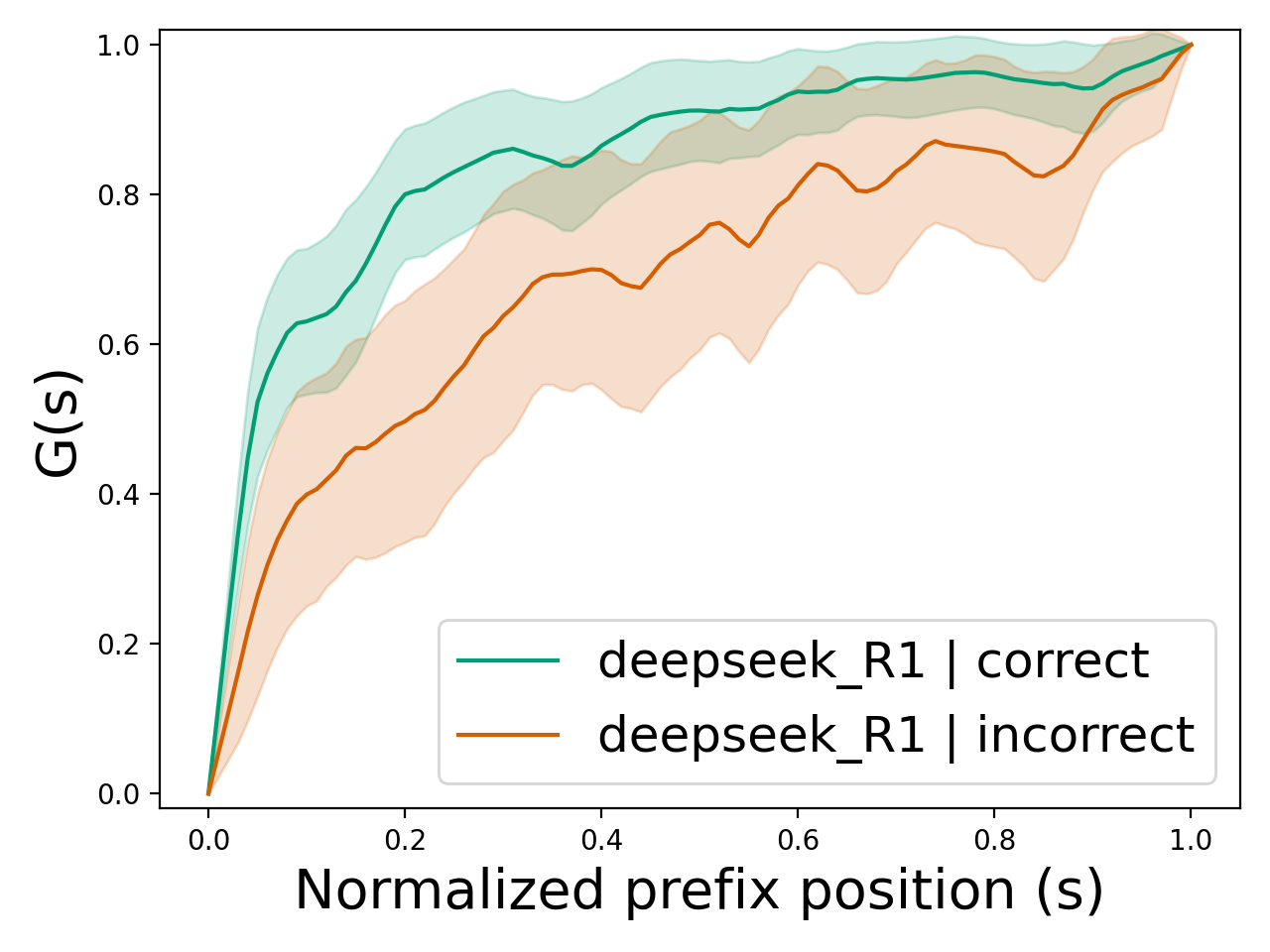}
    {DeepSeek-R1-Distilled}%
  \hfill
  \signaturepanel
    {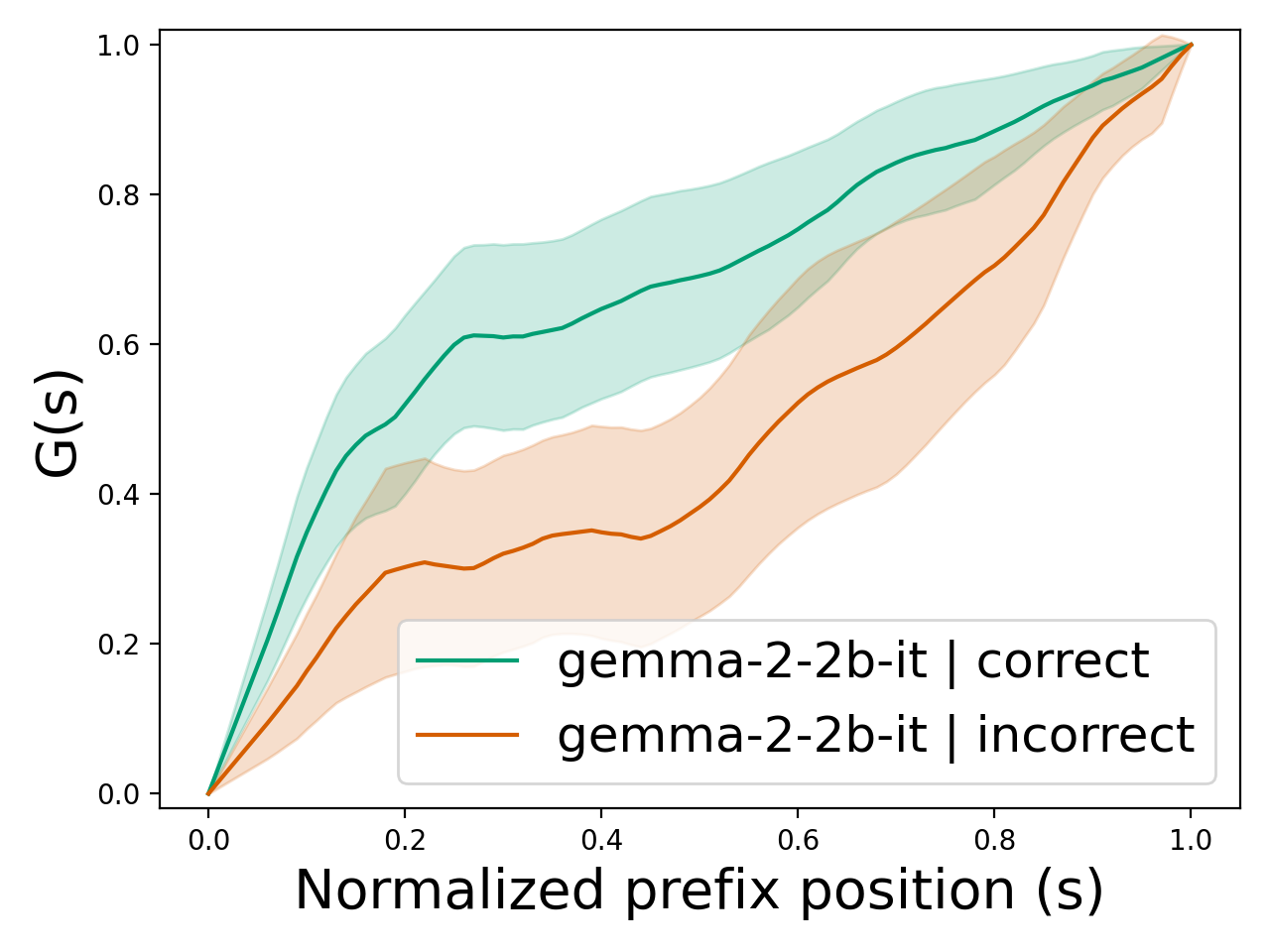}
    {Gemma-2-2B-IT}%
  \hfill
  \signaturepanel
    {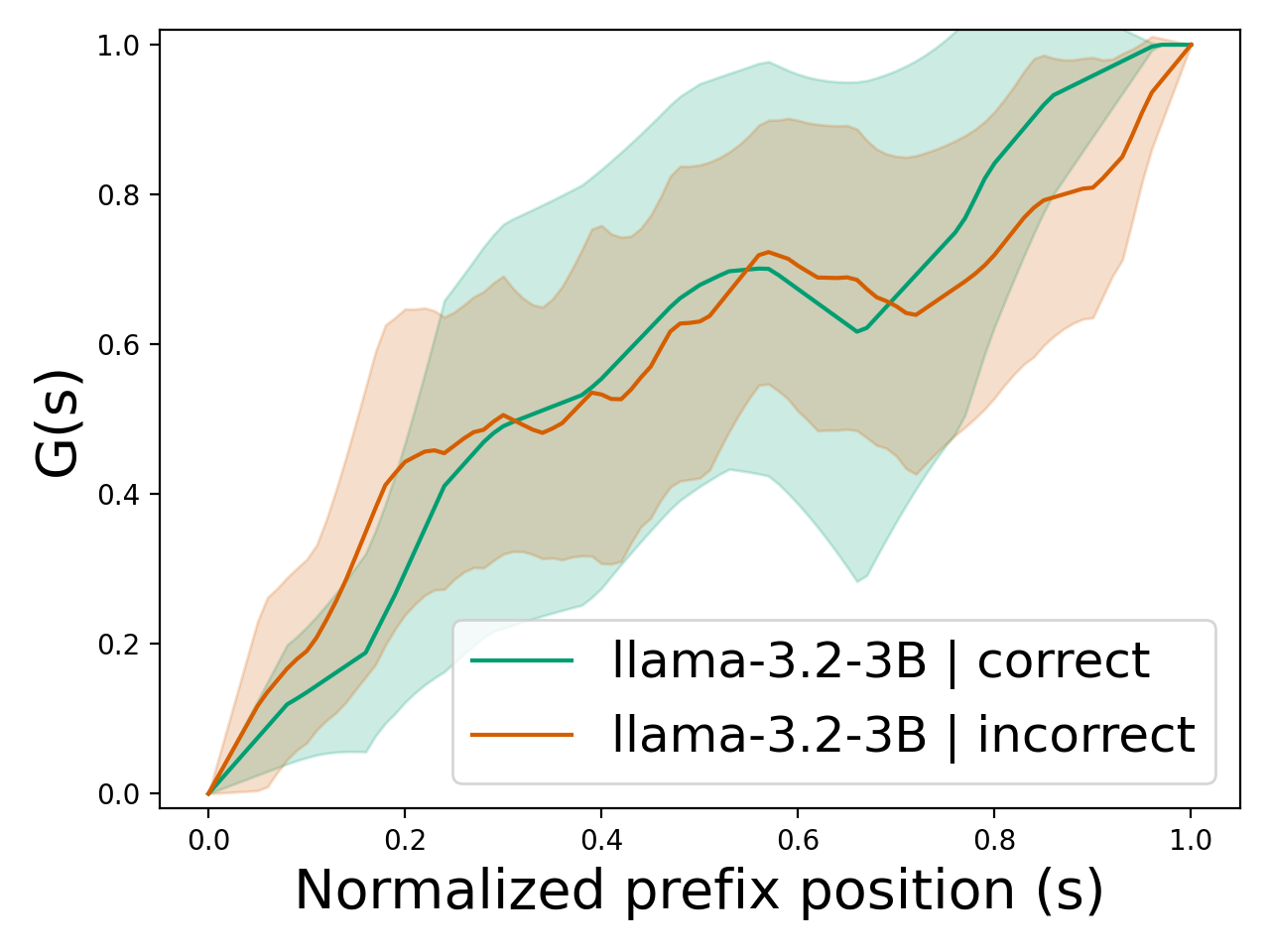}
    {LLaMA-3.2-3B}%
  \hfill
  \signaturepanel
    {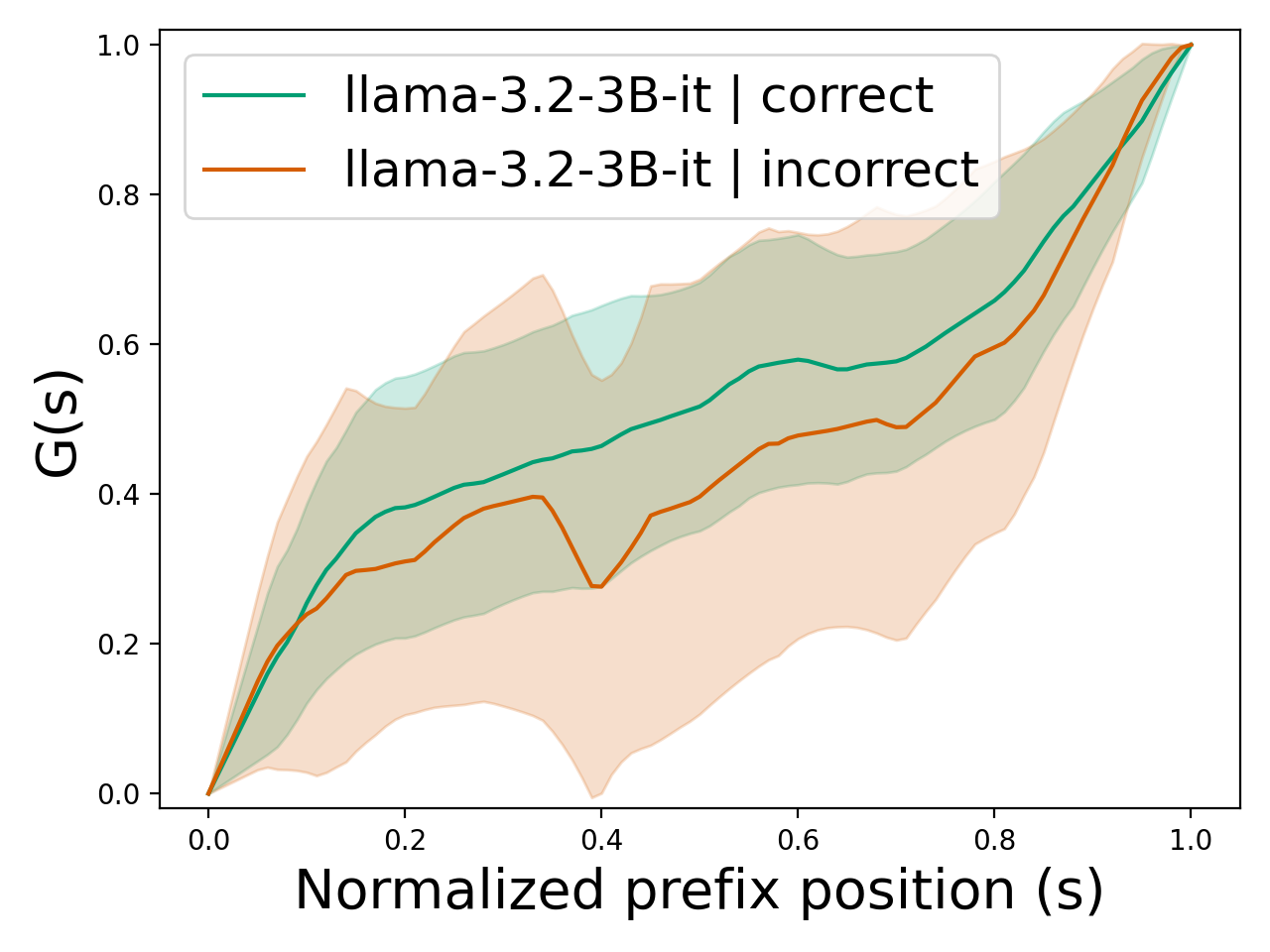}
    {LLaMA-3.2-3B-IT}%

  \par\medskip

  \signaturepanel
    {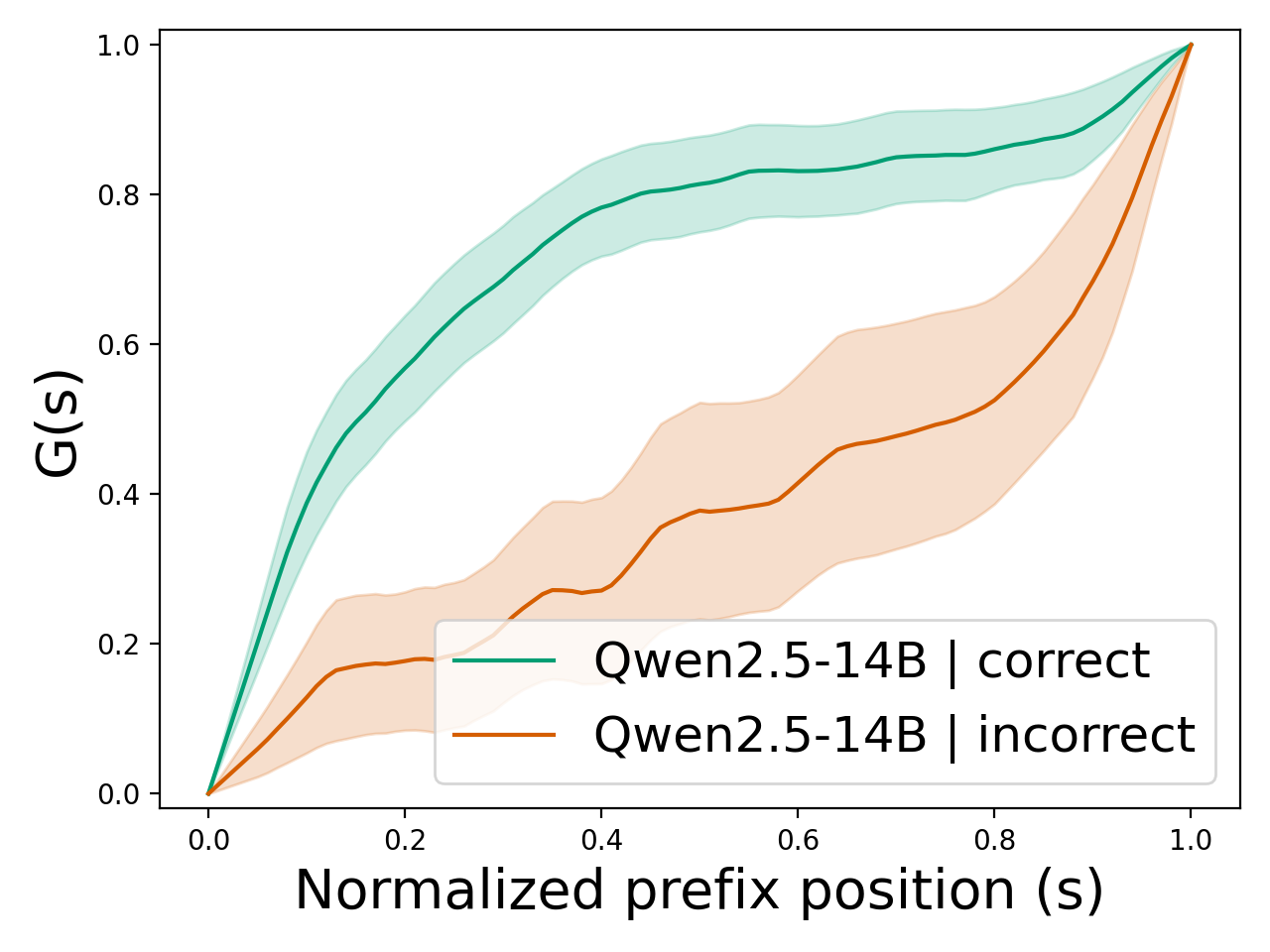}
    {Qwen2.5-14B}%
  \hfill
  \signaturepanel
    {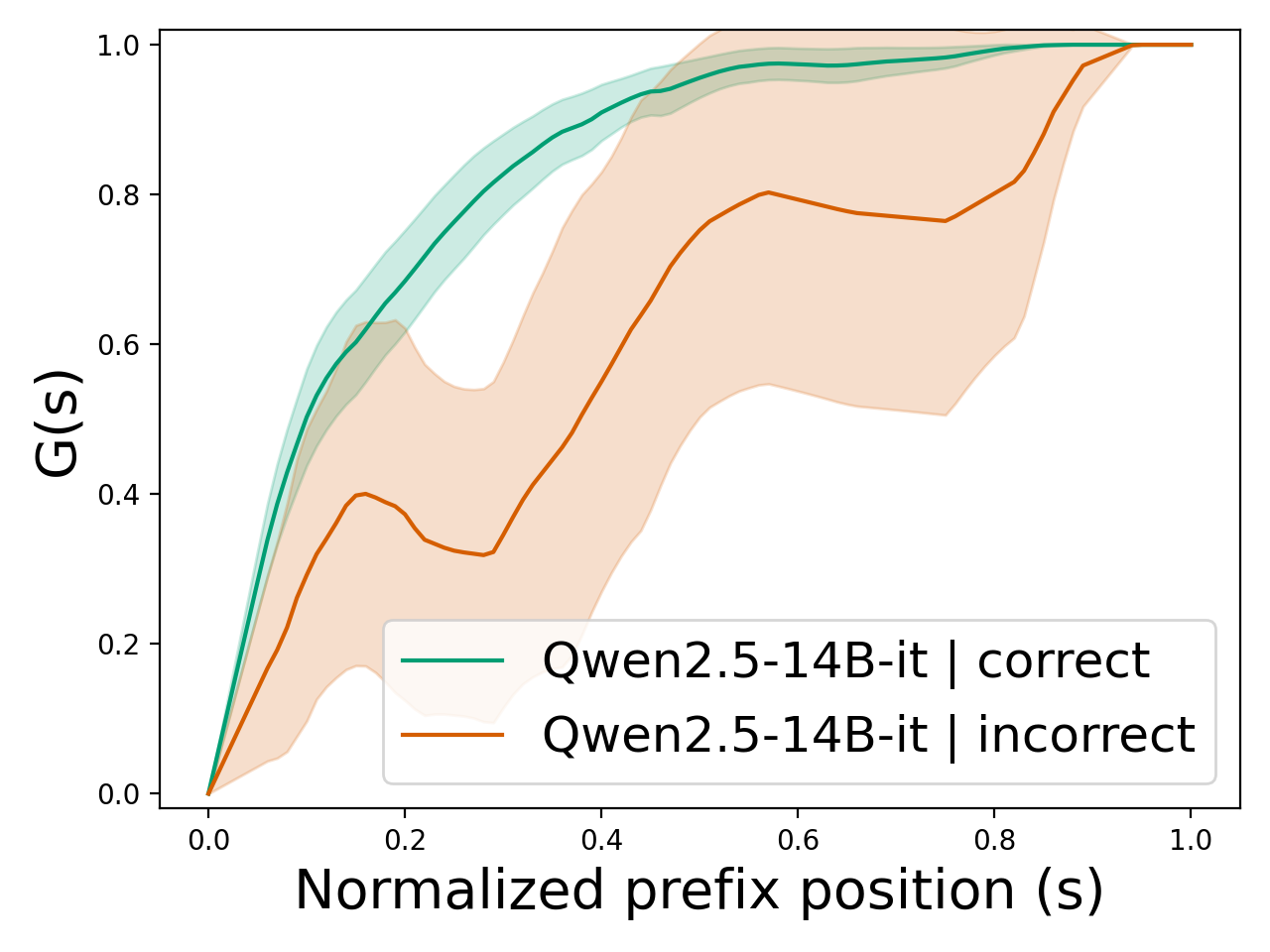}
    {Qwen2.5-14B-IT}%
  \hfill
  \signaturepanel
    {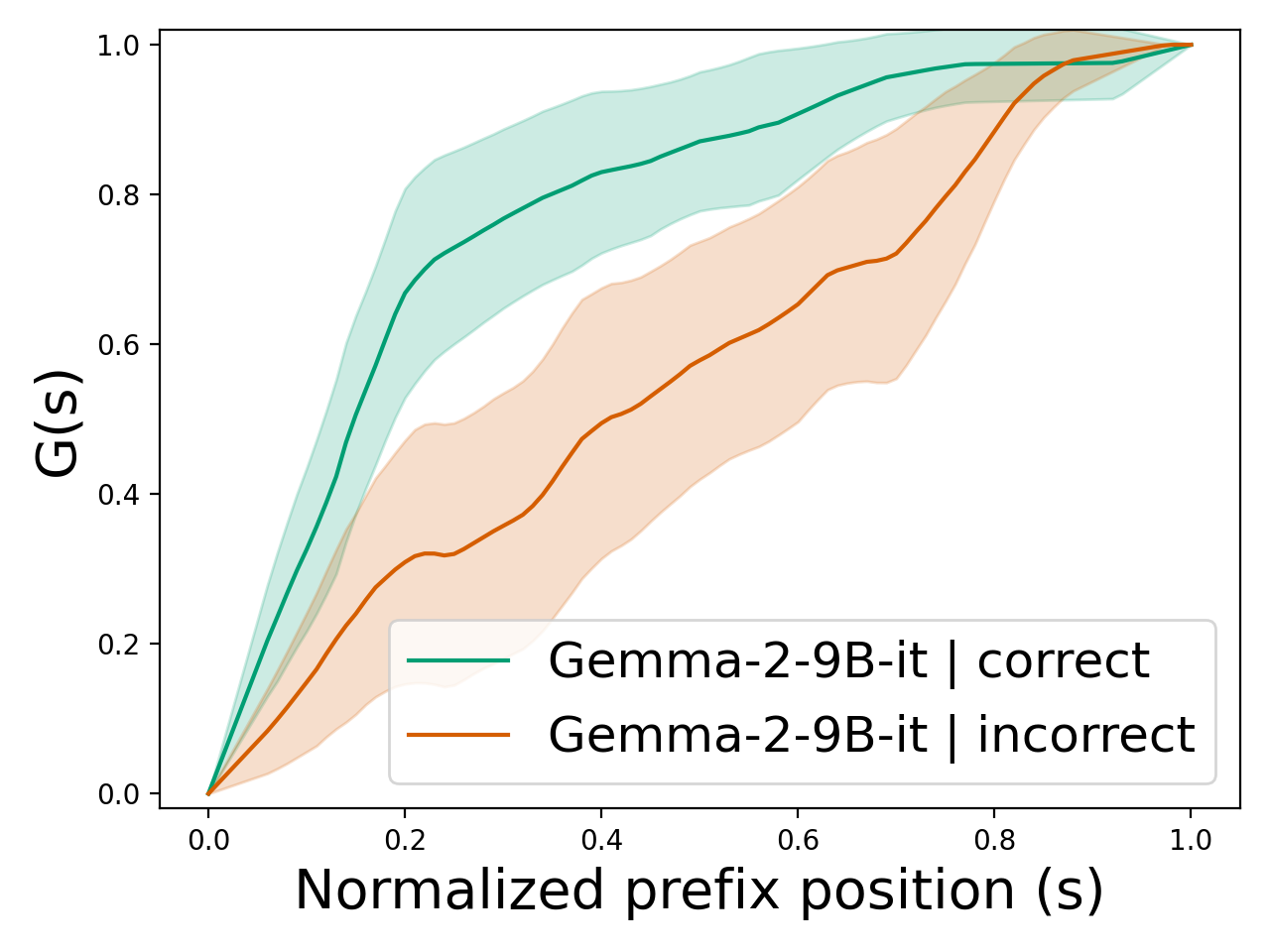}
    {Gemma-2-9B-IT}%
  \hfill
  \signaturepanel
    {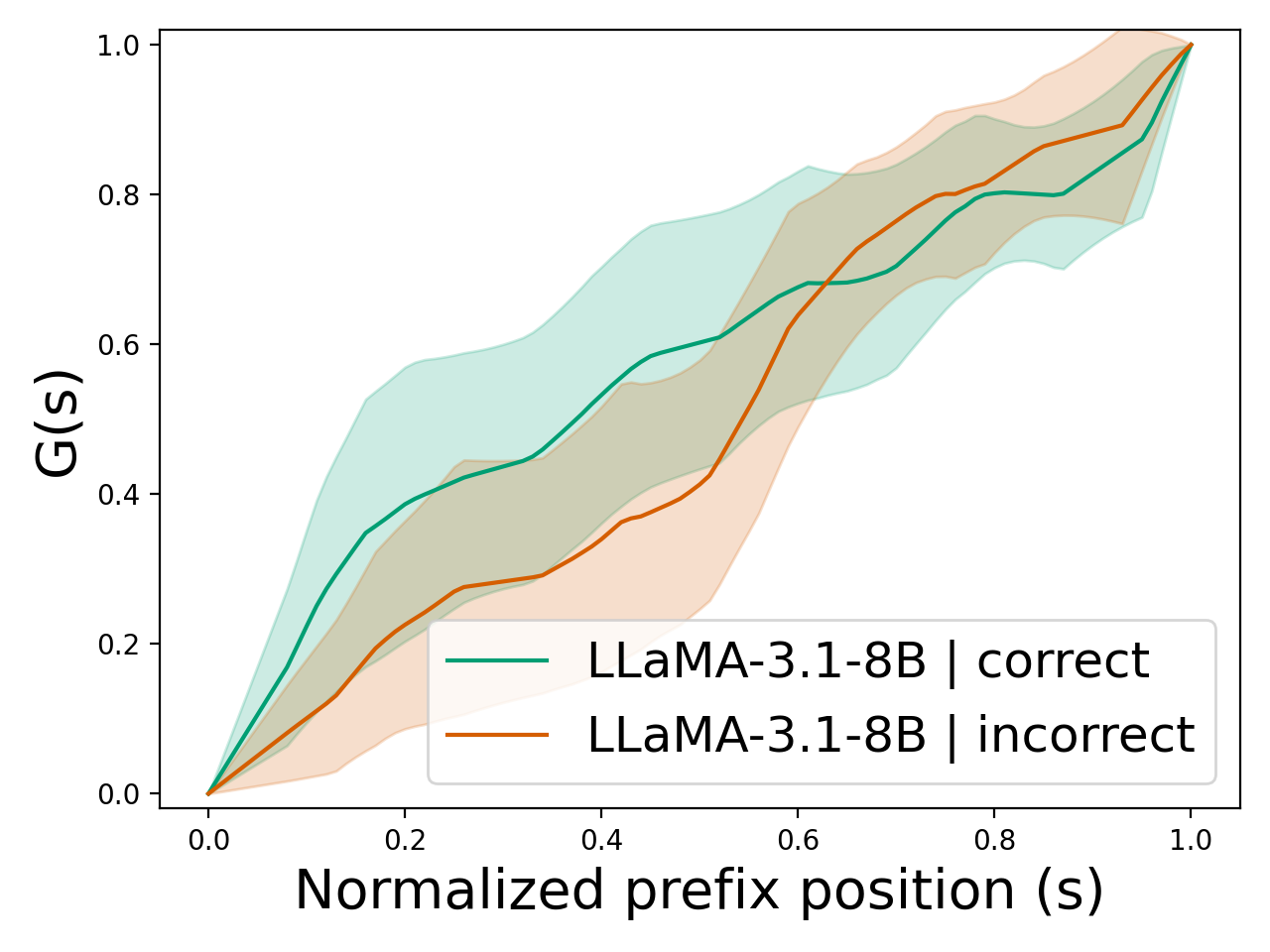}
    {LLaMA-3.1-8B}%

  \par\medskip

  \makebox[\linewidth][c]{%
    \signaturepanel
      {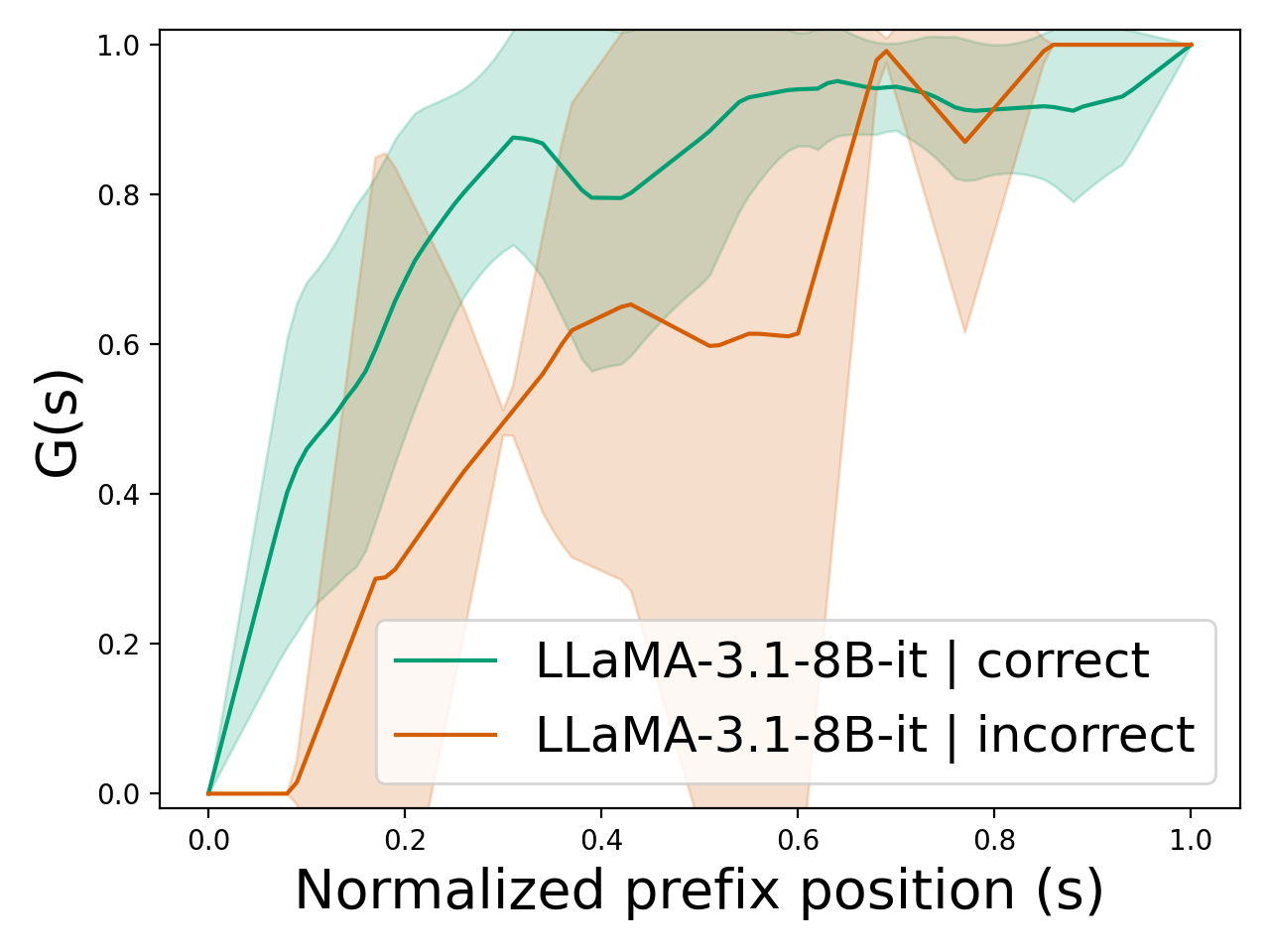}
      {LLaMA-3.1-8B-IT}%
  }

  \caption{\textbf{Early information accumulation for aligned
  model-dataset pairs on SVAMP.}
  Normalized cumulative information gain $G(s)$ versus relative prefix length $s$, split by correctness. Each panel corresponds to one model. Shaded regions denote 95\% bootstrap confidence intervals.}
  \label{fig:aligned-earlyinfo-svamp}
\end{figure*}

\begin{figure*}[htbp]
  \centering

  \signaturepanel
    {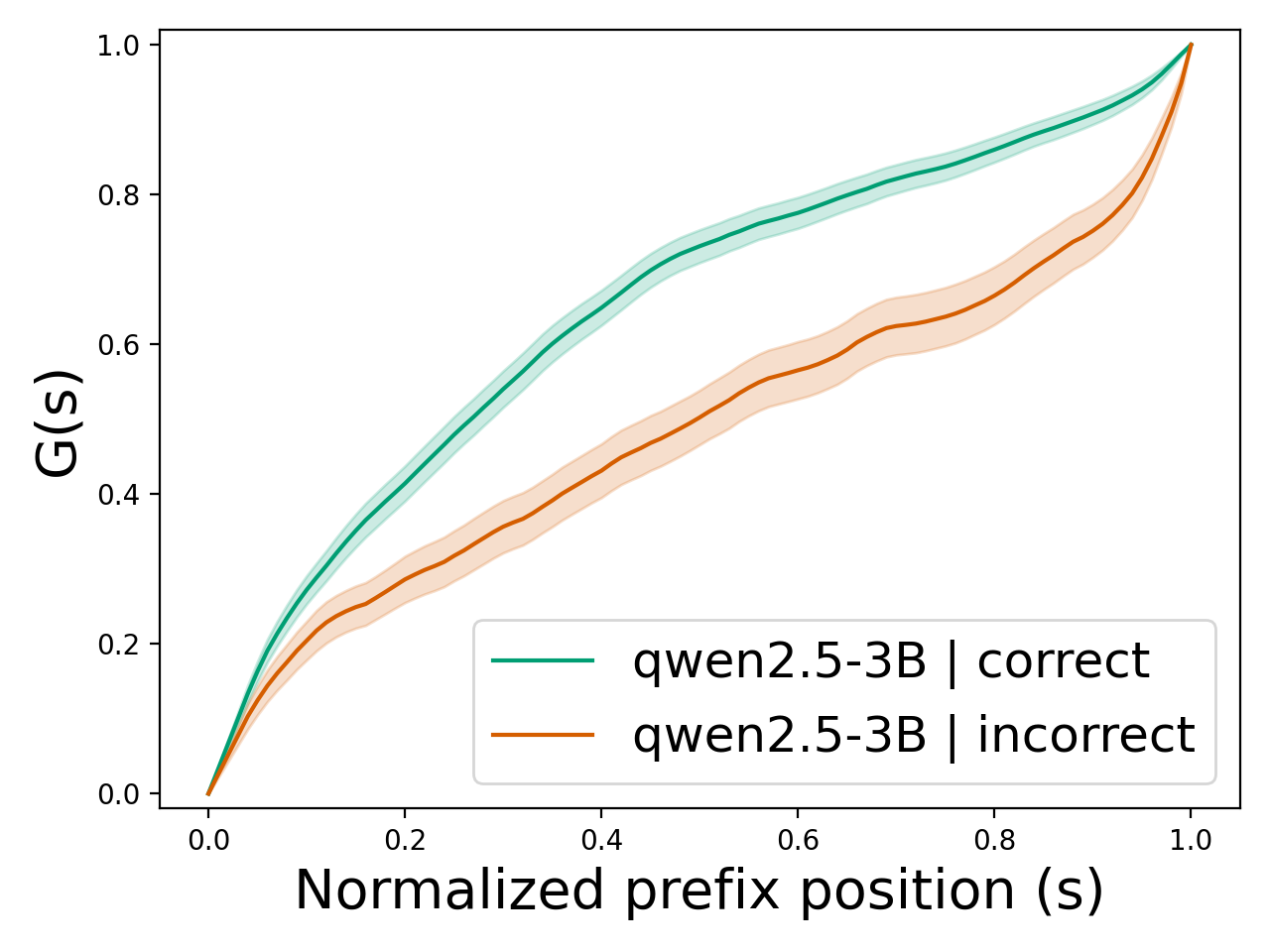}
    {Qwen2.5-3B}%
  \hfill
  \signaturepanel
    {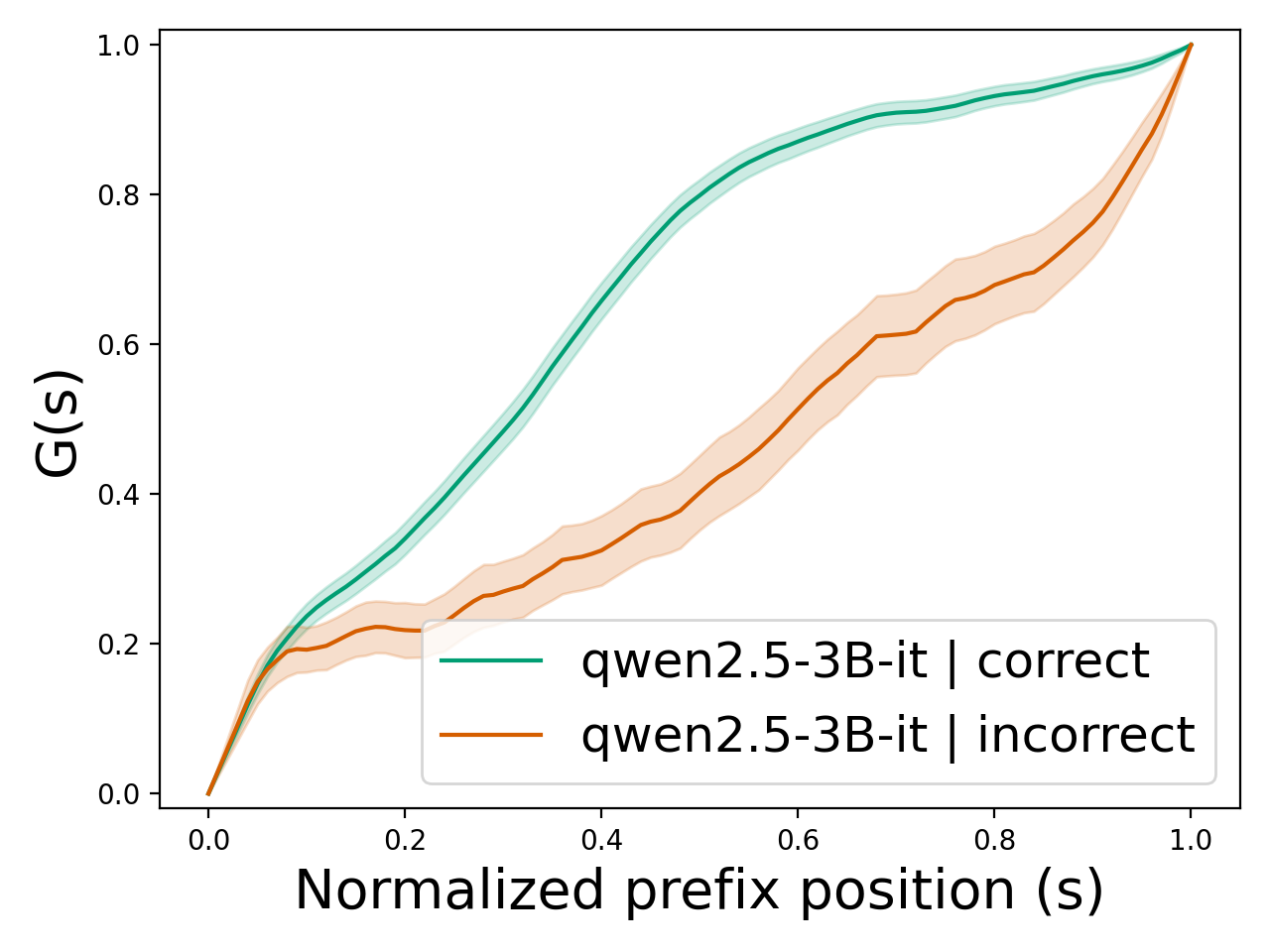}
    {Qwen2.5-3B-IT}%
  \hfill
  \signaturepanel
    {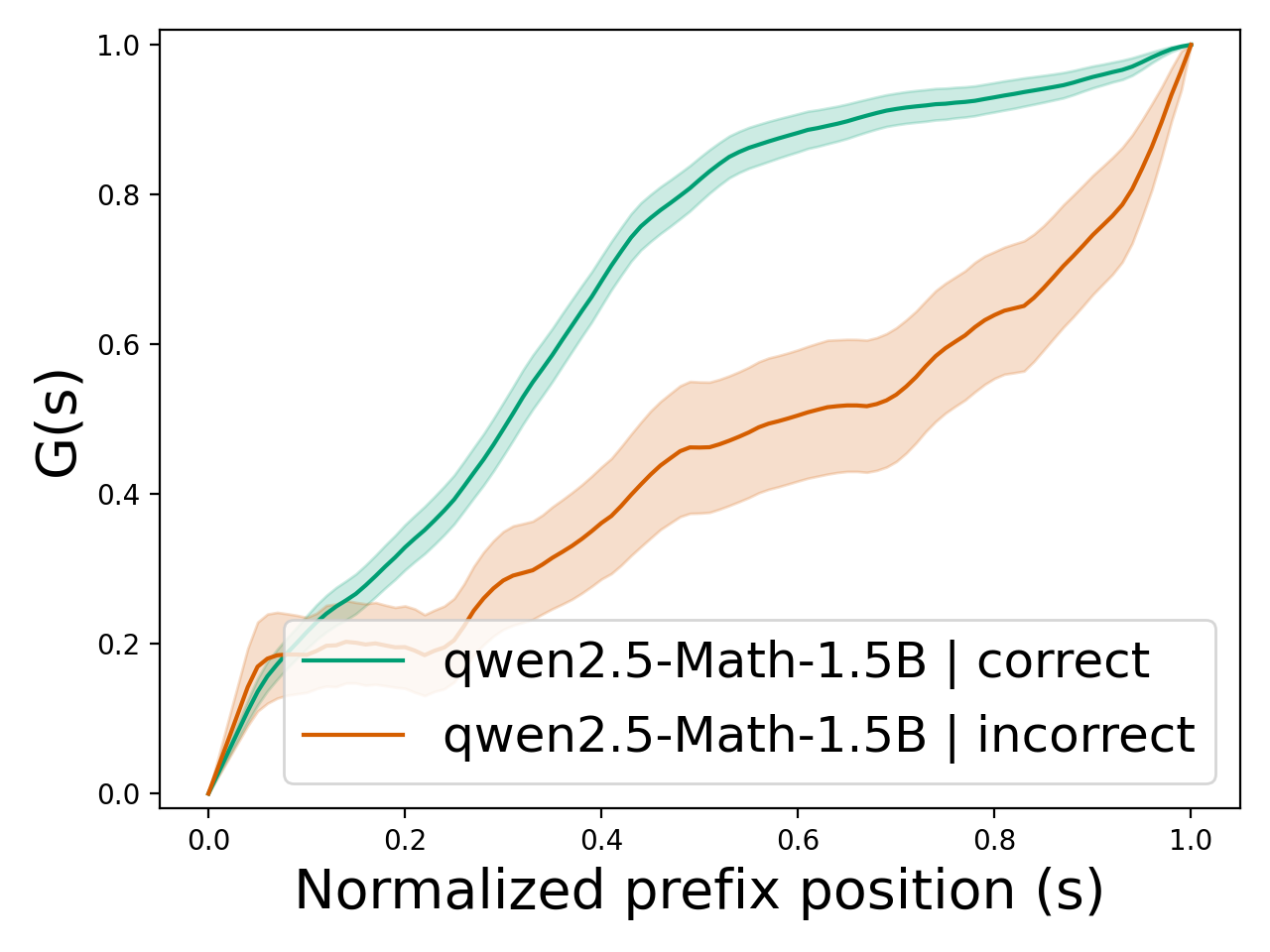}
    {Qwen2.5-Math-1.5B}%
  \hfill
  \signaturepanel
    {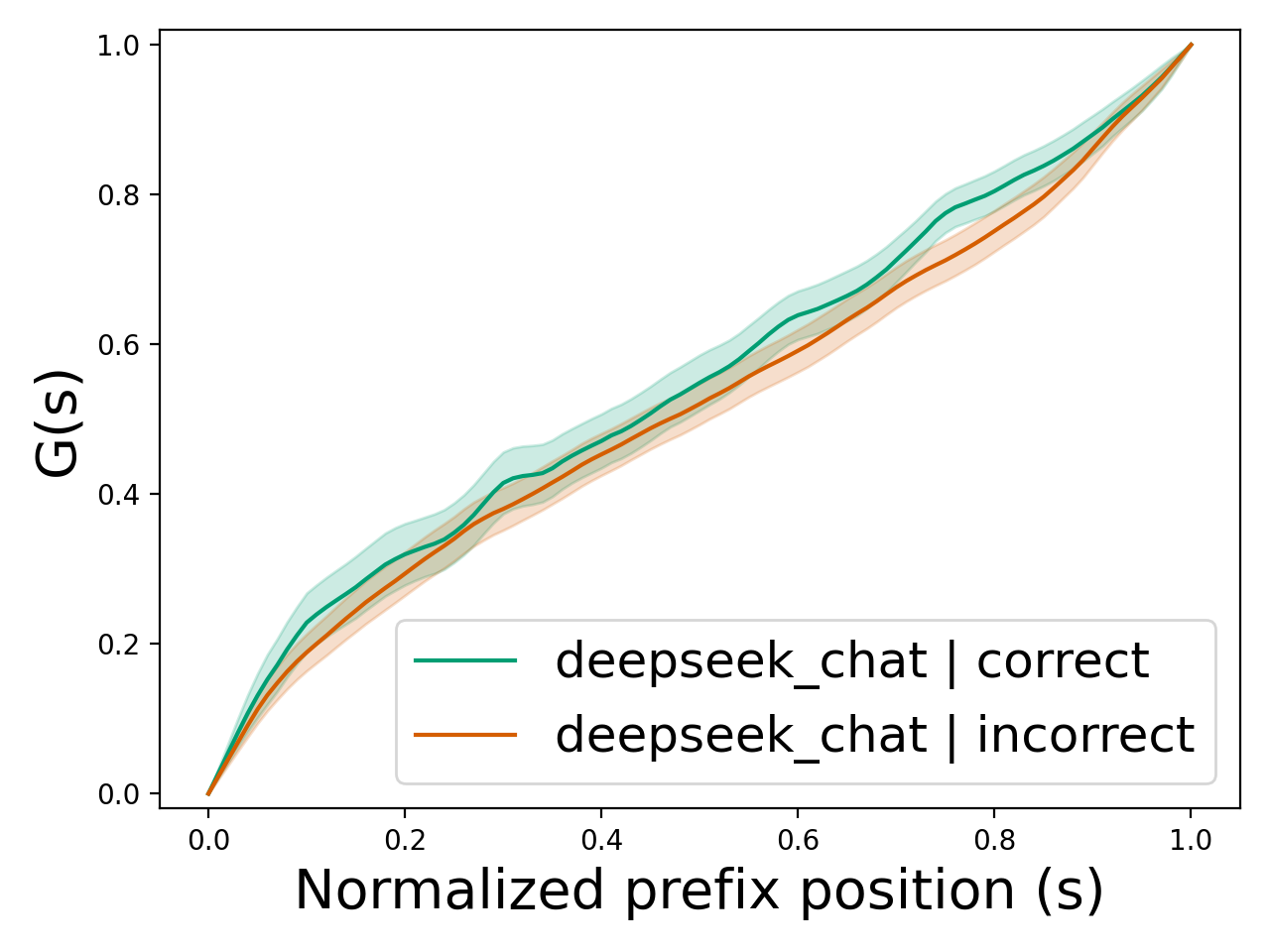}
    {DeepSeek-Chat-7B}%

  \par\medskip

  \signaturepanel
    {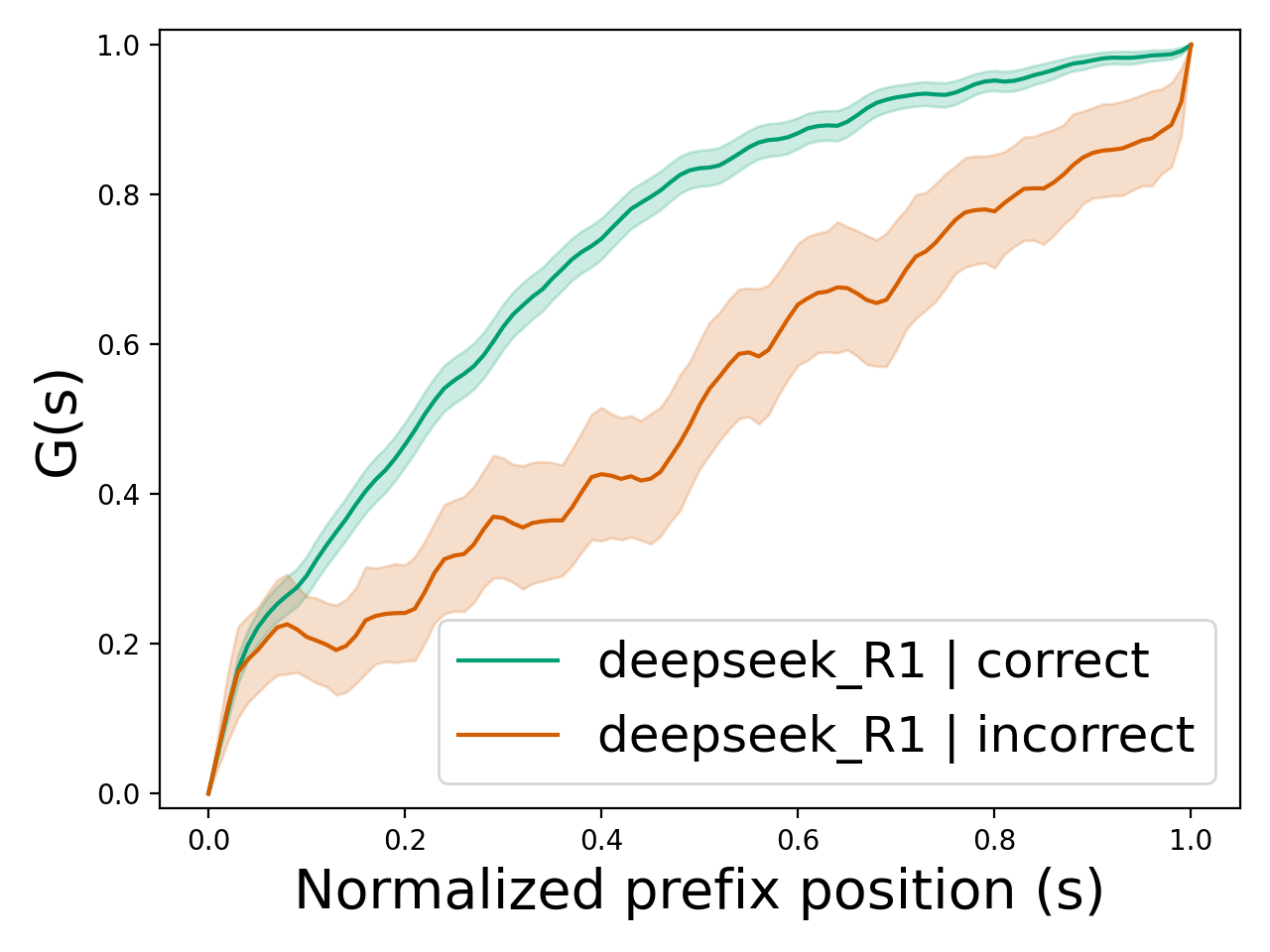}
    {DeepSeek-R1-Distilled}%
  \hfill
  \signaturepanel
    {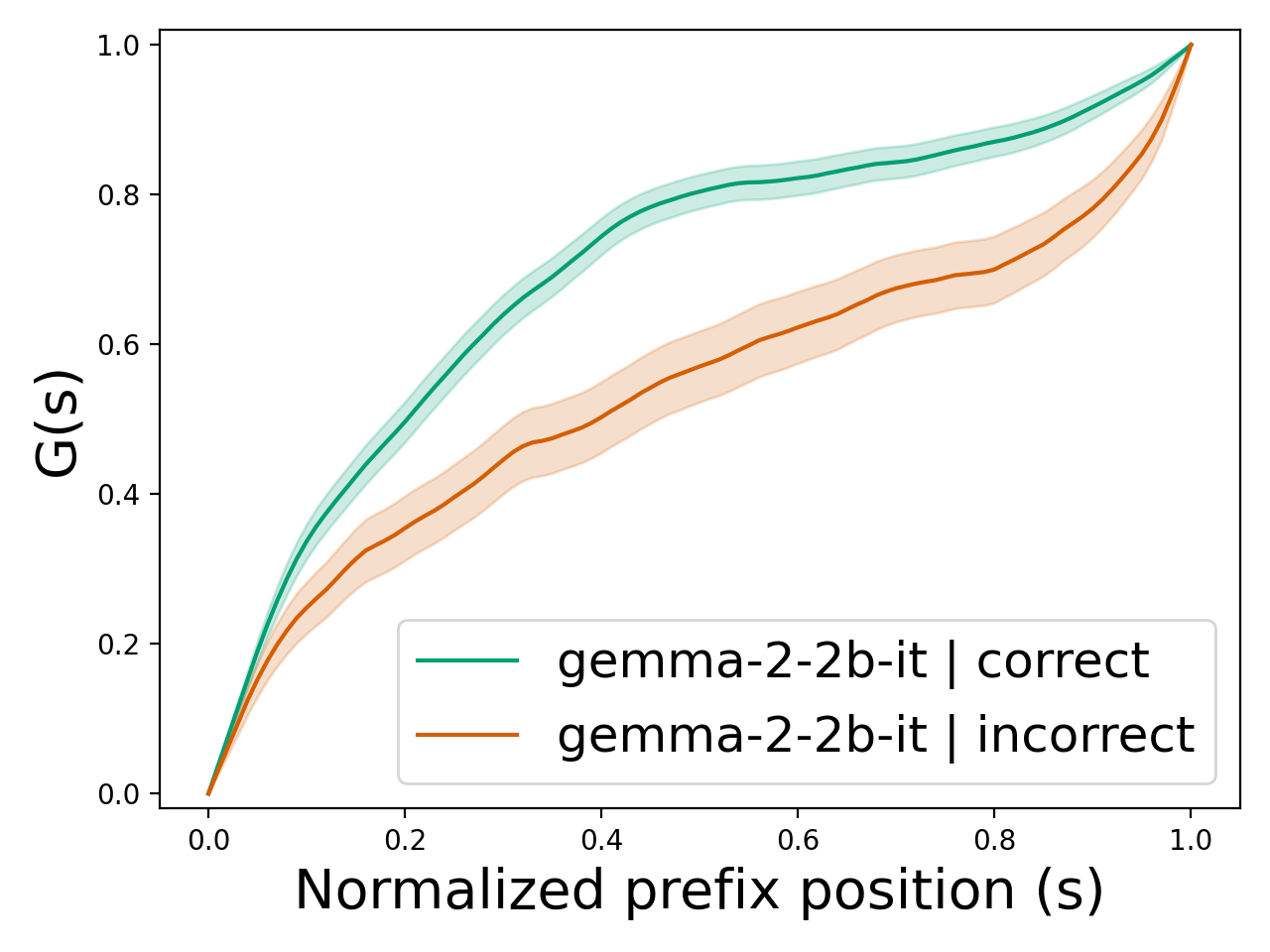}
    {Gemma-2-2B-IT}%
  \hfill
  \signaturepanel
    {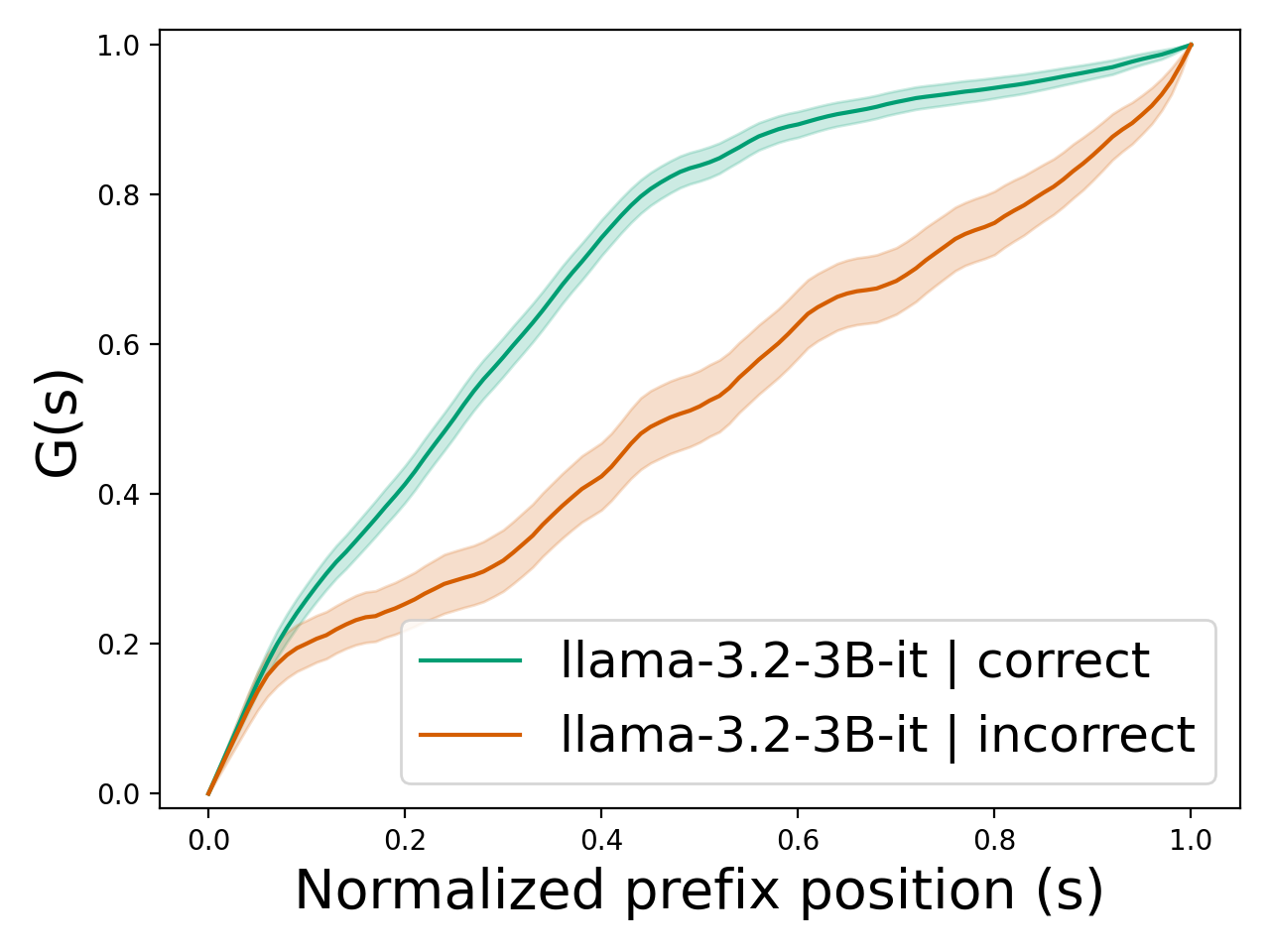}
    {LLaMA-3.2-3B-IT}%
  \hfill
  \signaturepanel
    {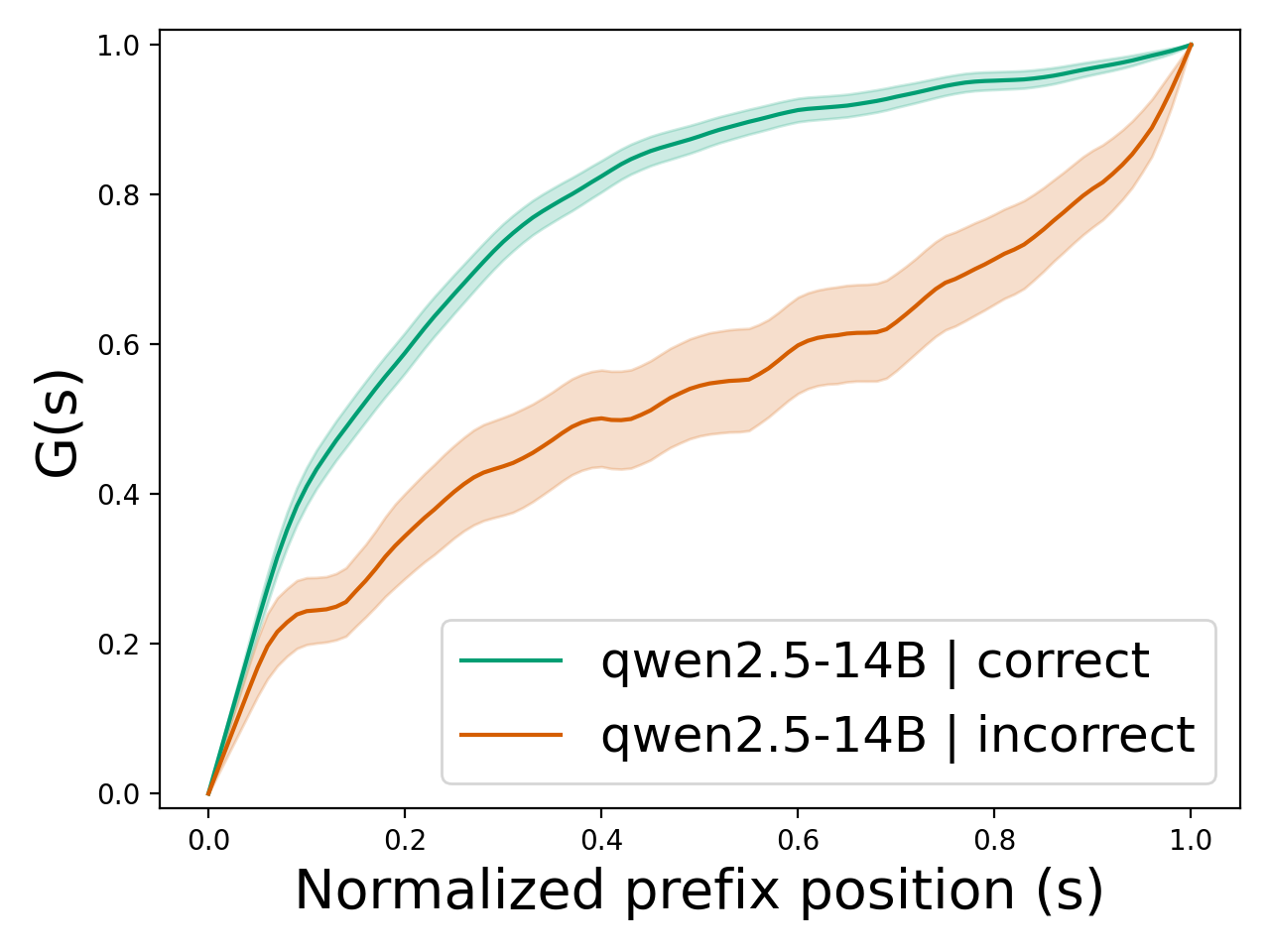}
    {Qwen2.5-14B}%

  \par\medskip

  \signaturepanel
    {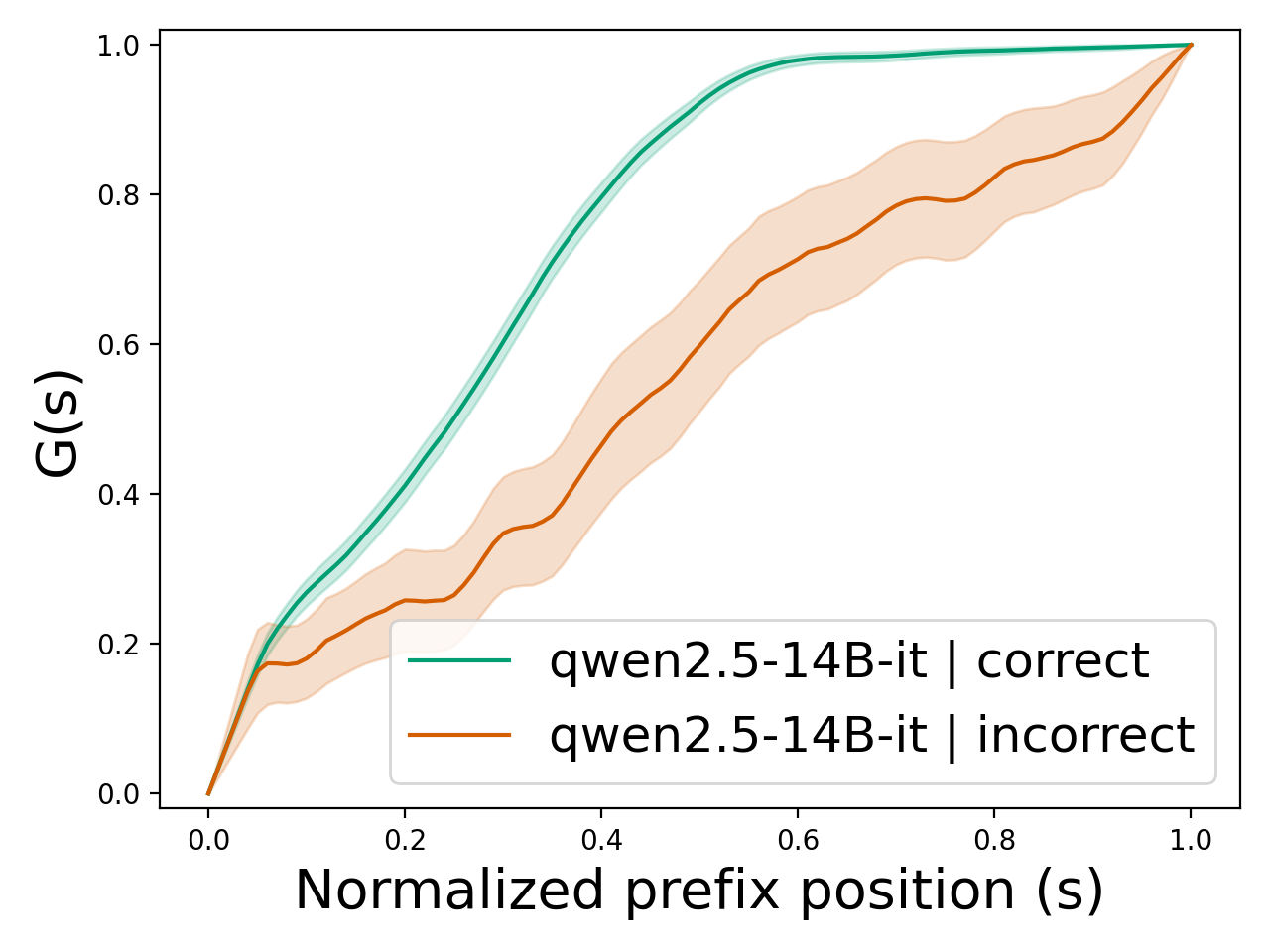}
    {Qwen2.5-14B-IT}%
  \hfill
  \signaturepanel
    {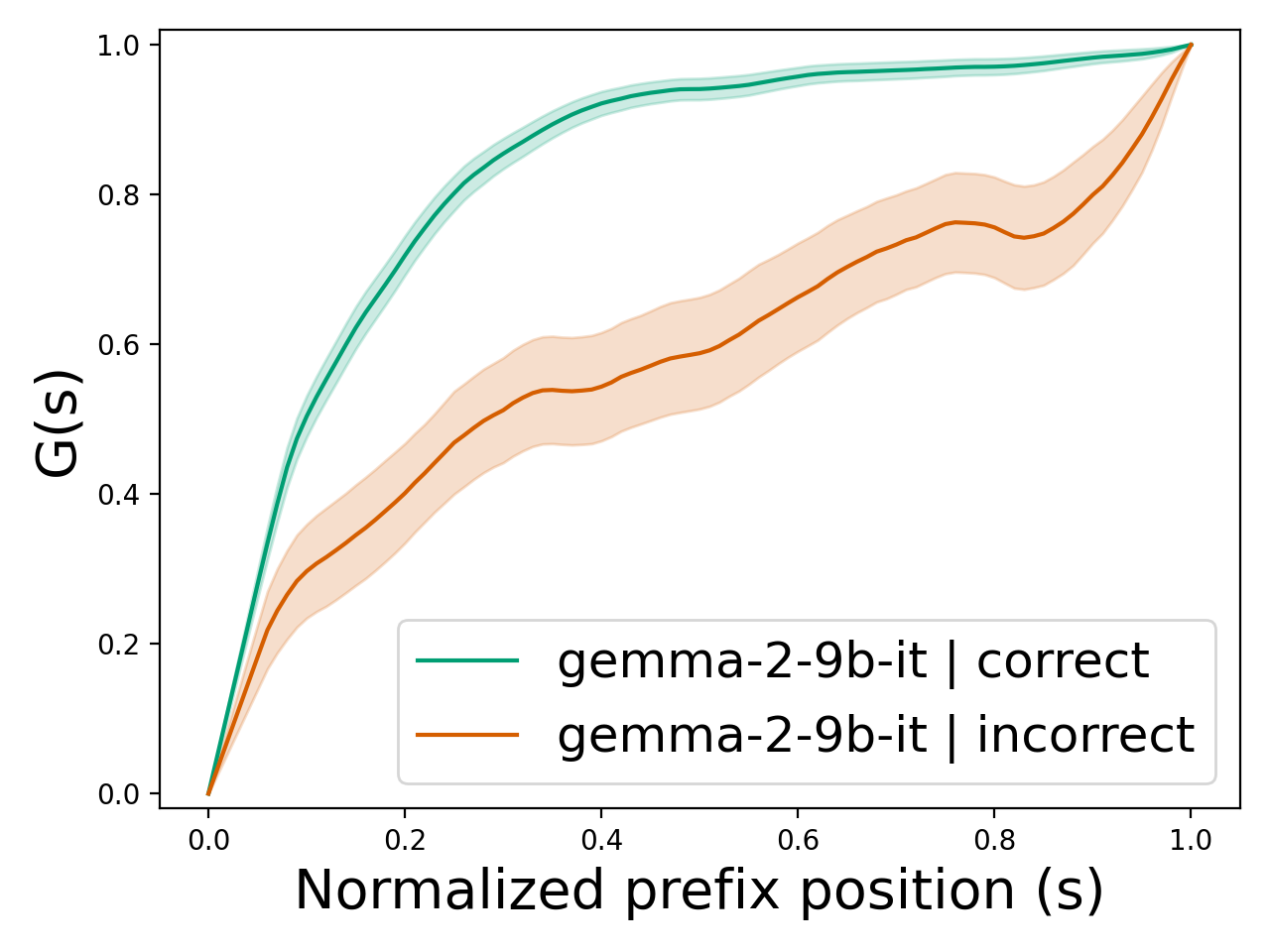}
    {Gemma-2-9B-IT}%
  \hfill
  \signaturepanel
    {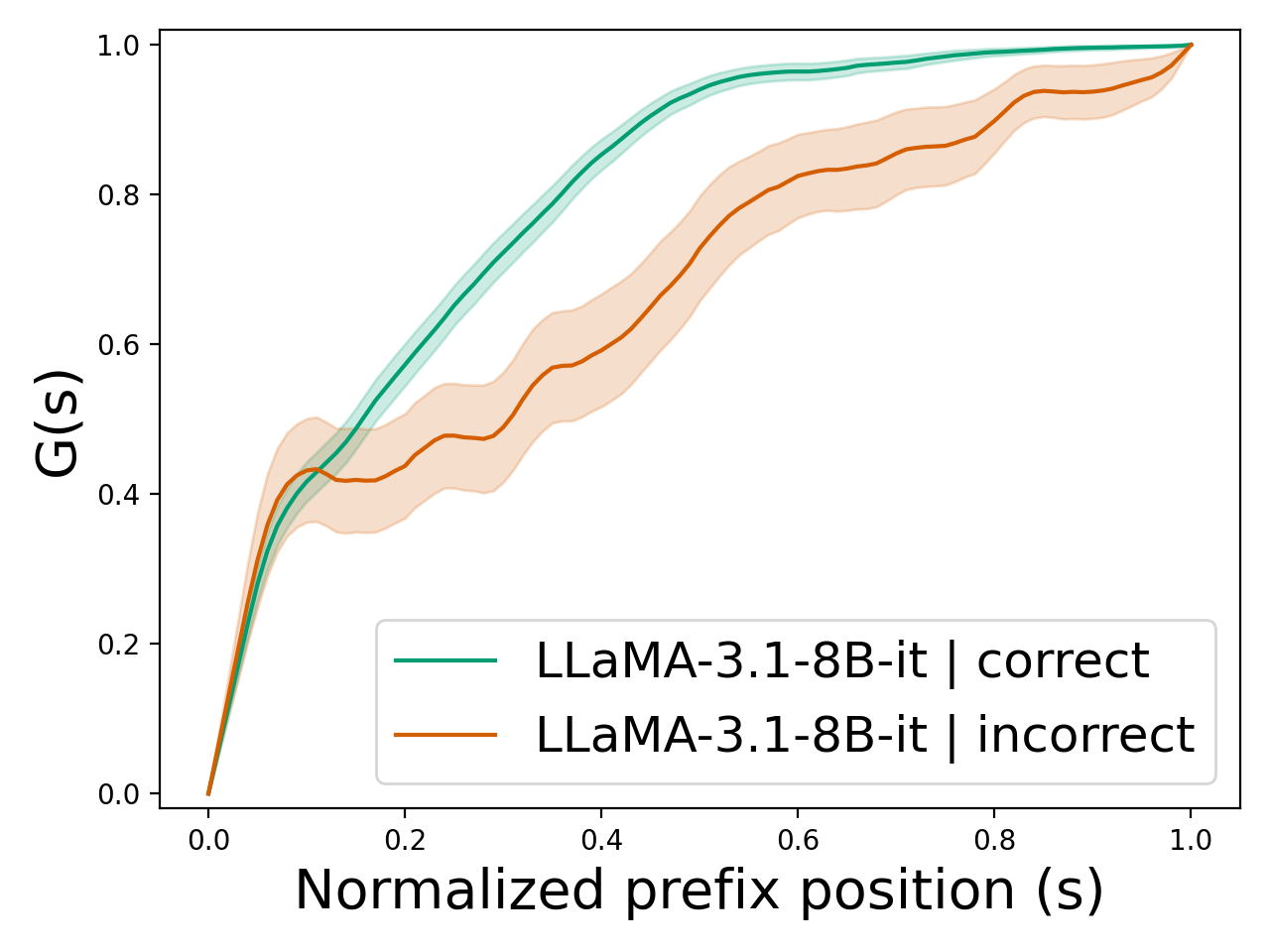}
    {LLaMA-3.1-8B-IT}%
  \hfill
  \signaturepanel
    {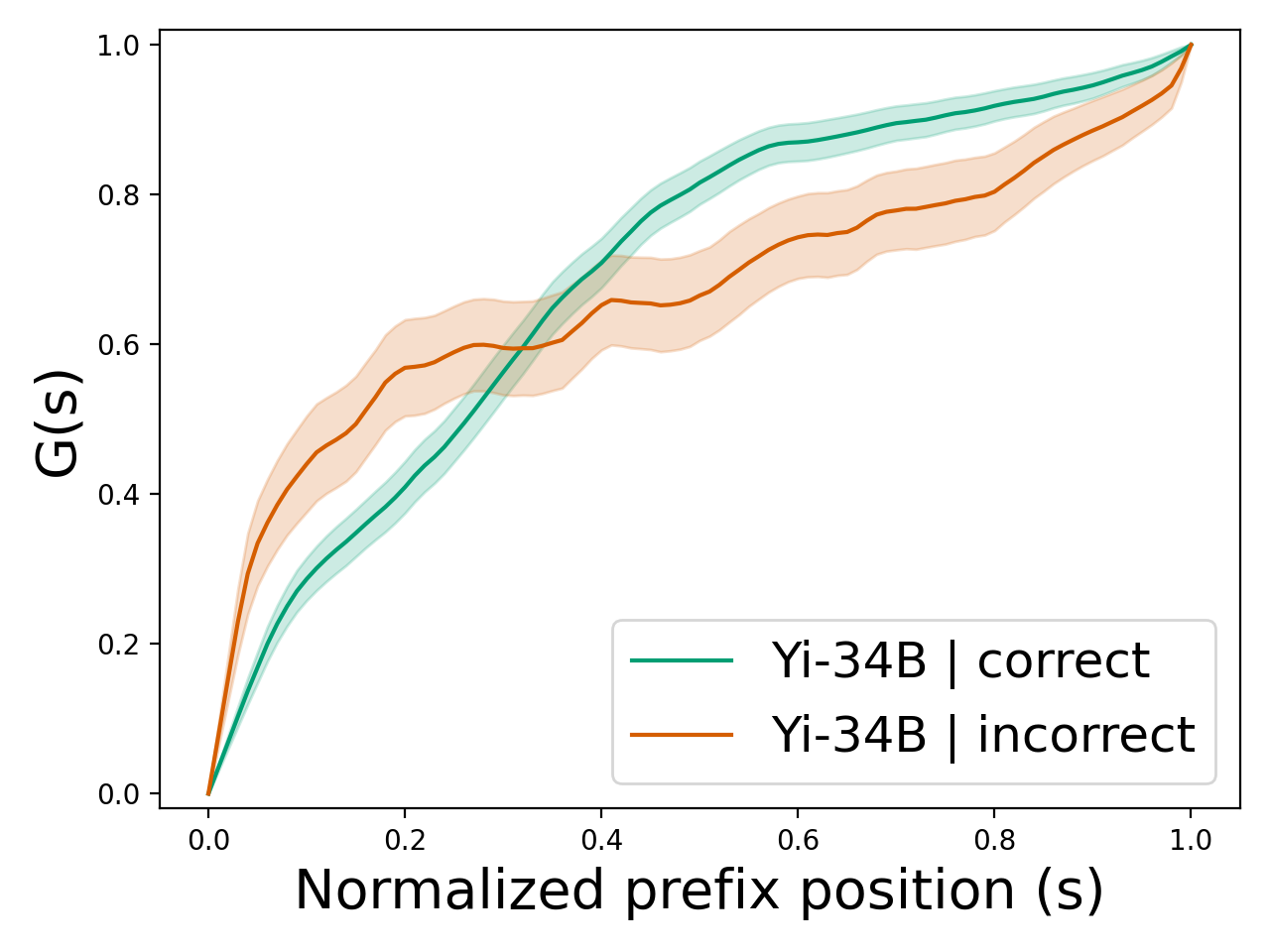}
    {Yi-1.5-34B}%

  \caption{\textbf{Early information accumulation for aligned
  model-dataset pairs on ARC.}
  Normalized cumulative information gain $G(s)$ versus relative prefix length $s$, split by correctness. Each panel corresponds to one model. Shaded regions denote 95\% bootstrap confidence intervals.}
  \label{fig:aligned-earlyinfo-arc}
\end{figure*}

\begin{figure*}[htbp]
  \centering

  \signaturepanel
    {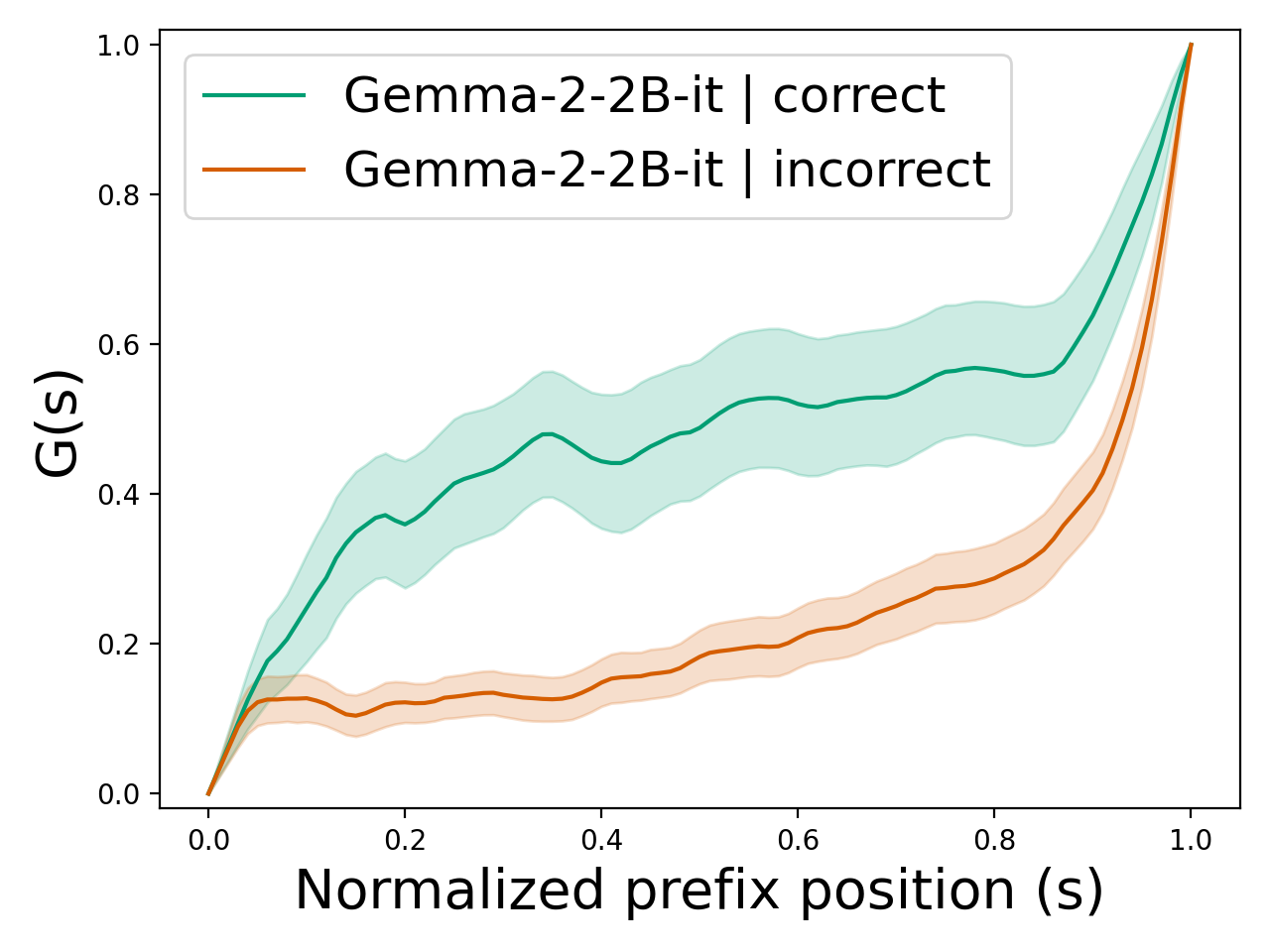}
    {Gemma-2-2B-IT}%
  \hfill
  \signaturepanel
    {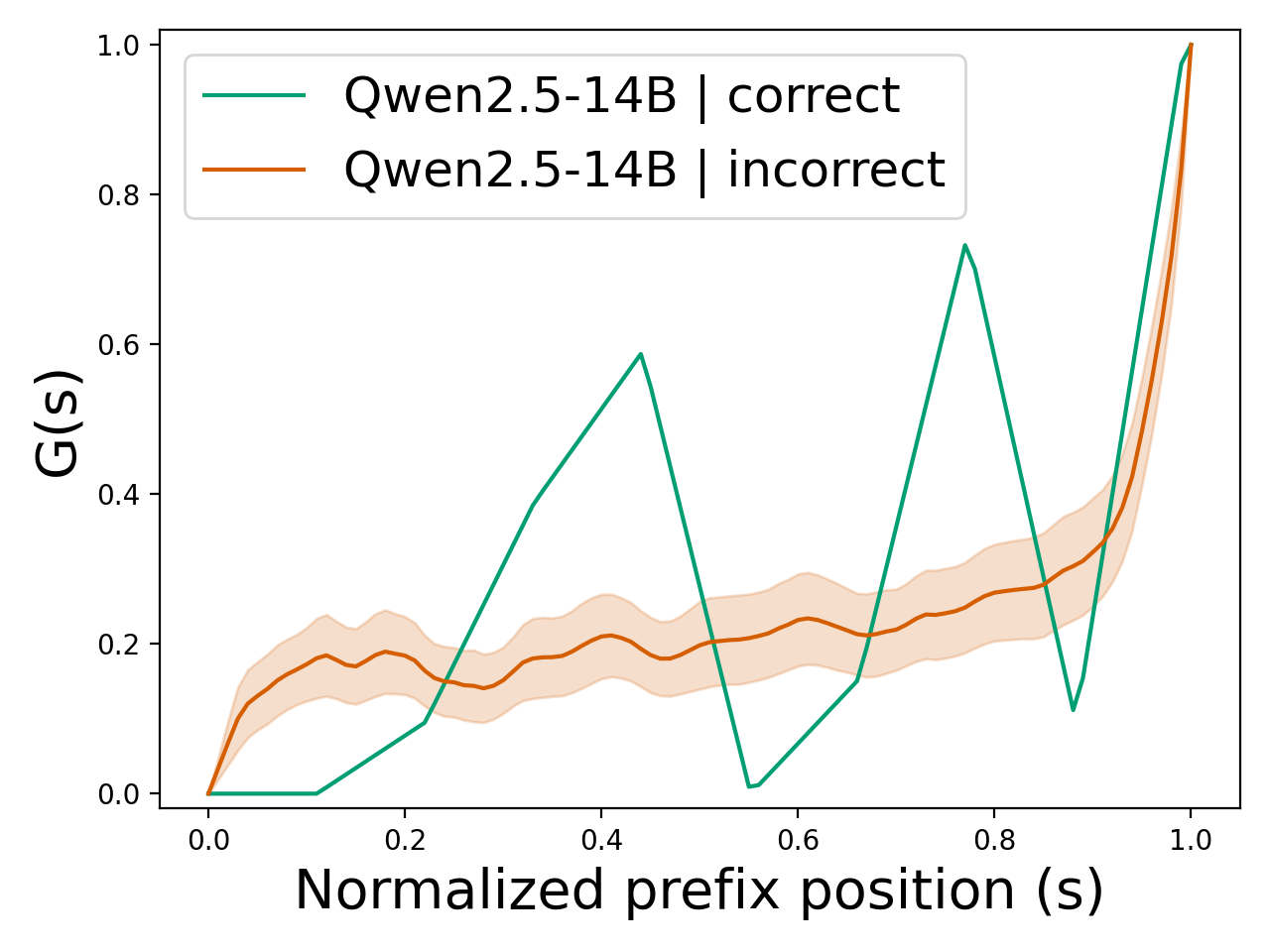}
    {Qwen2.5-14B}%
  \hfill
  \signaturepanel
    {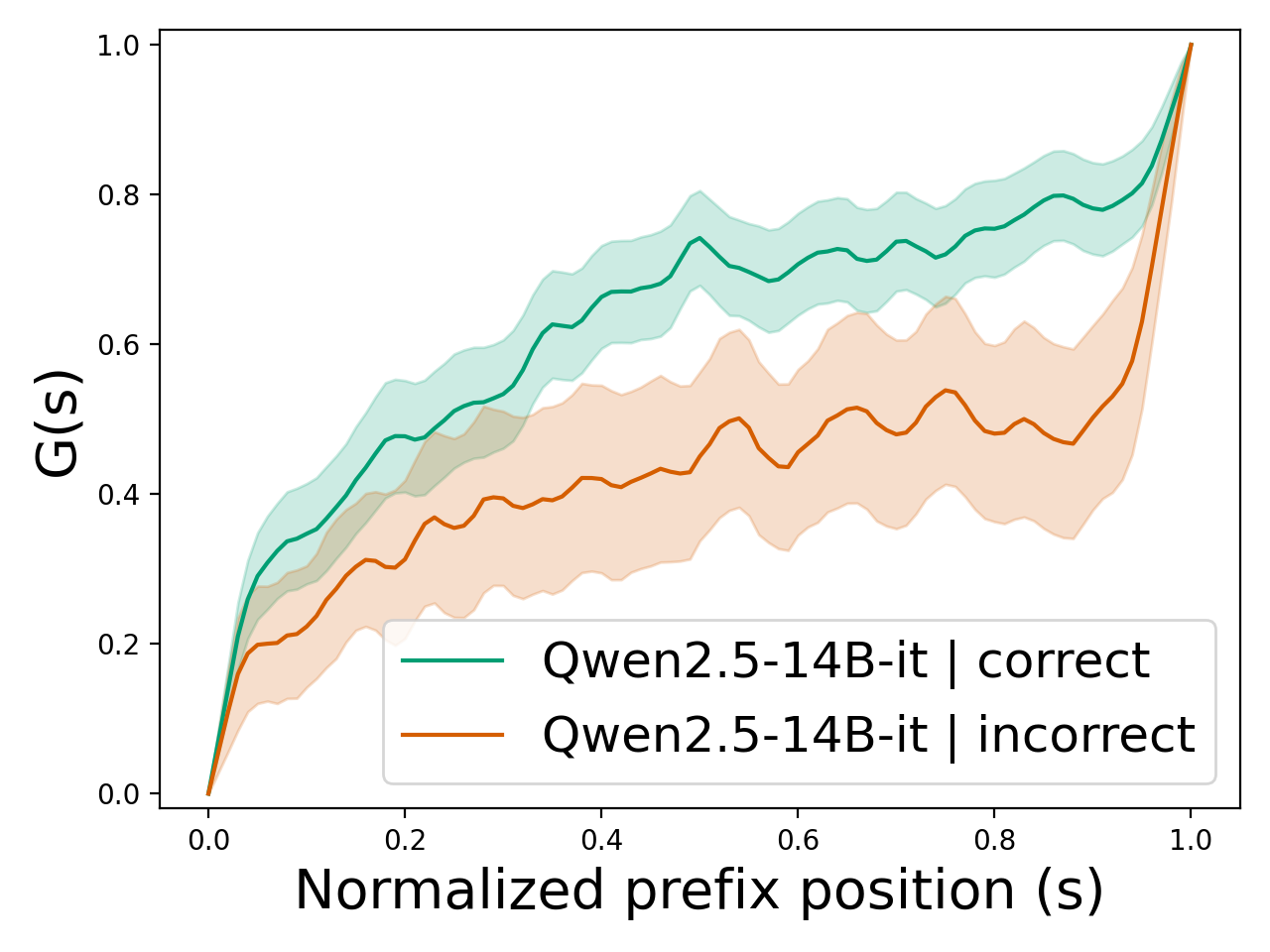}
    {Qwen2.5-14B-IT}%
  \hfill
  \signaturepanel
    {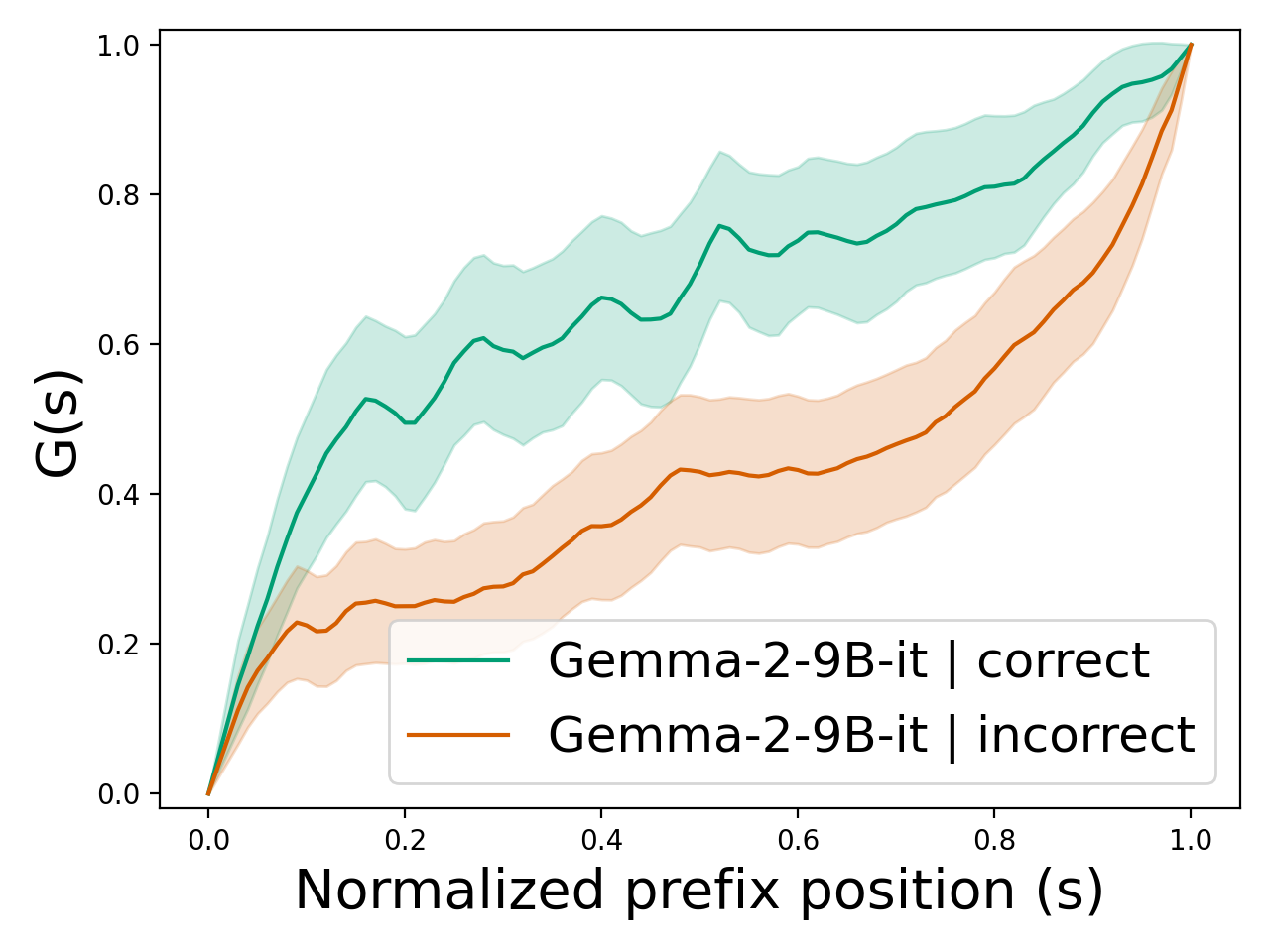}
    {Gemma-2-9B-IT}%

  \par\medskip

  \makebox[\linewidth][c]{%
    \signaturepanel
      {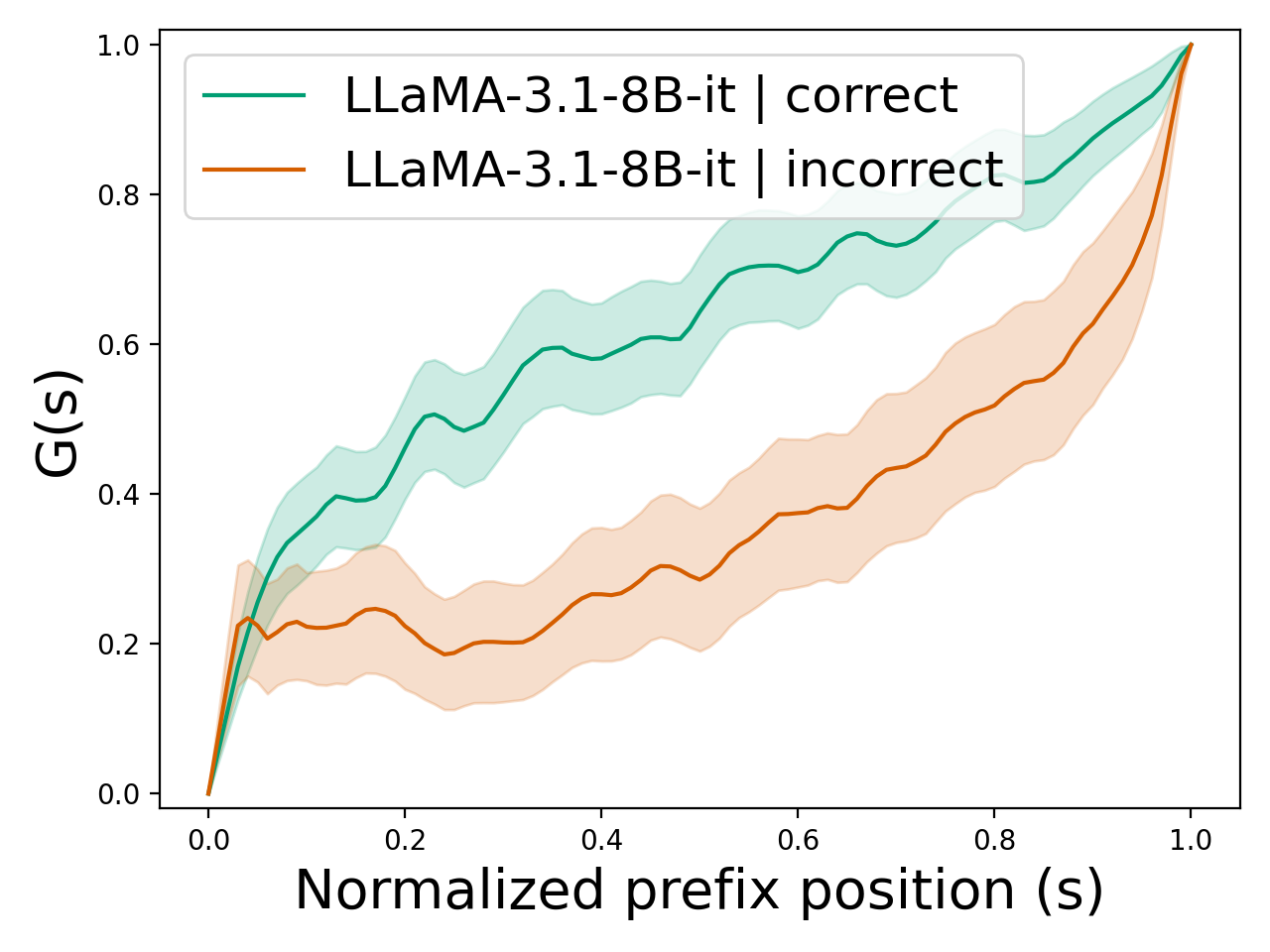}
      {LLaMA-3.1-8B-IT}%
  }

  \caption{\textbf{Early information accumulation for aligned
  model-dataset pairs on MATH.}
  Normalized cumulative information gain $G(s)$ versus relative prefix length $s$, split by correctness. Each panel corresponds to one model. Shaded regions denote 95\% bootstrap confidence intervals.}
  \label{fig:aligned-earlyinfo-math}
\end{figure*}

\paragraph{Non-aligned model-dataset pairs.}
Figures~\ref{fig:nonaligned-earlyinfo-gsm8k}-\ref{fig:nonaligned-earlyinfo-math} report normalized cumulative information gain for the model-dataset pairs classified as non-aligned.

\begin{figure*}[htbp]
  \centering

  \signaturepanel
    {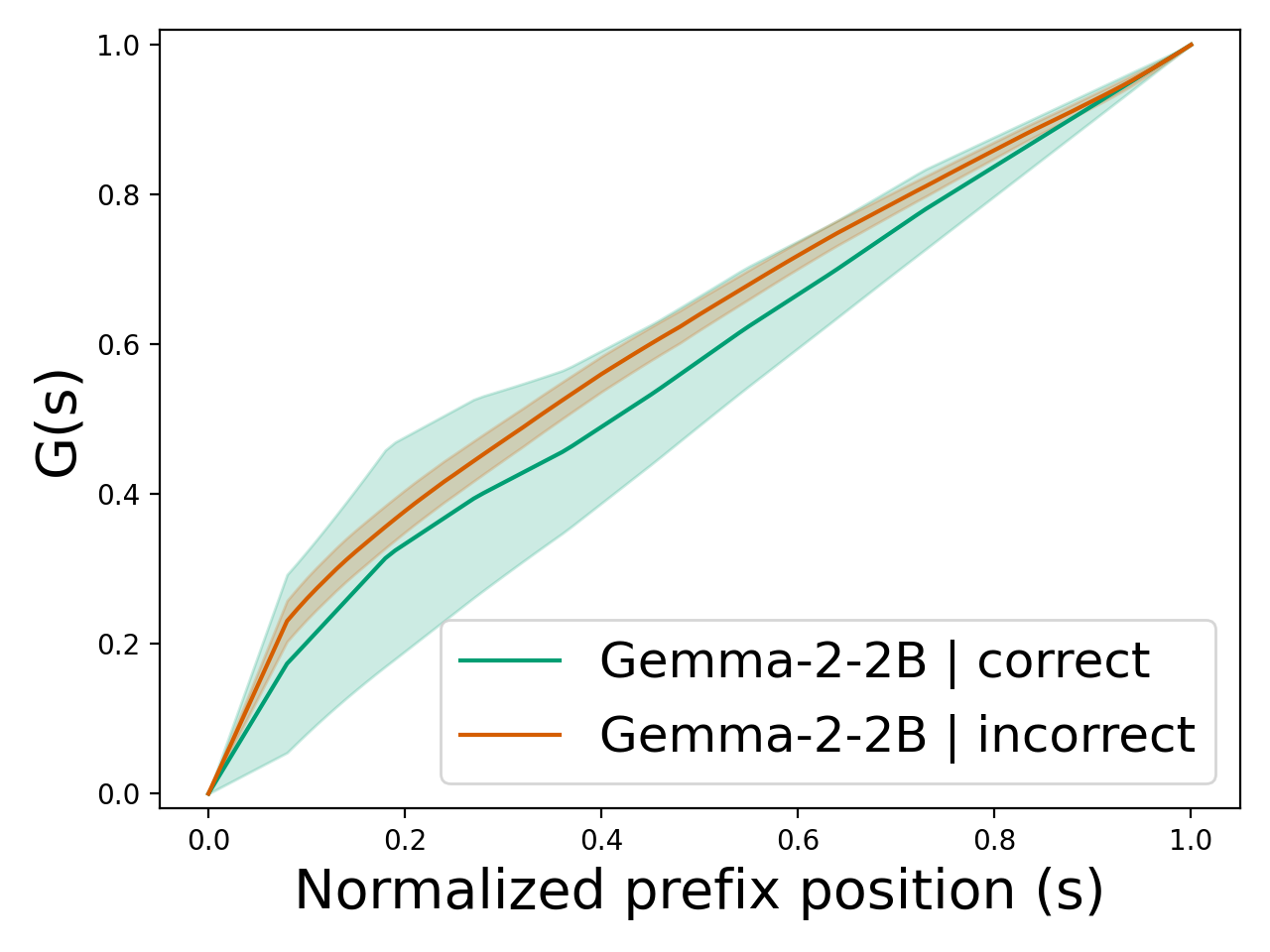}
    {Gemma-2-2B}%
  \hfill
  \signaturepanel
    {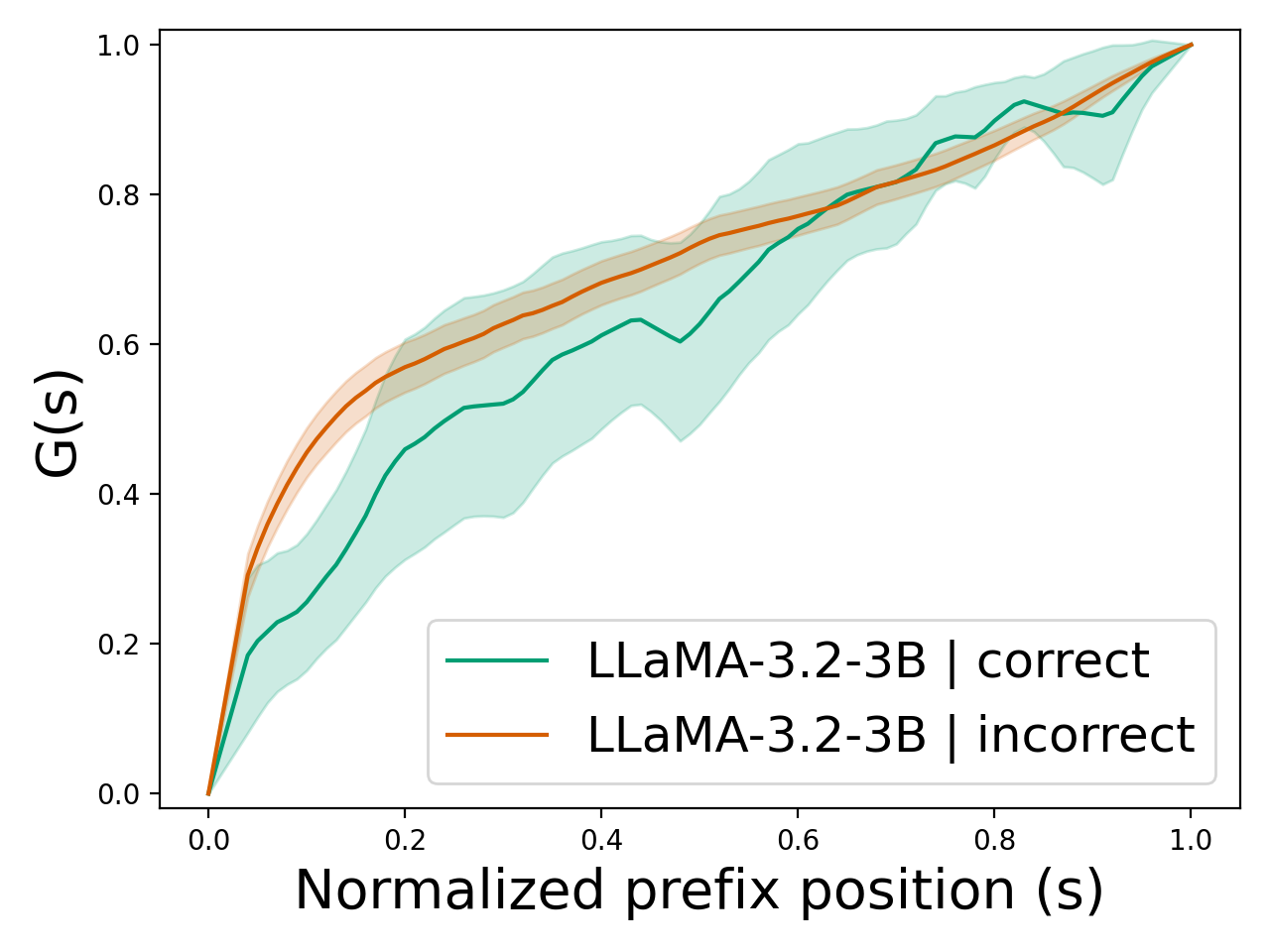}
    {LLaMA-3.2-3B}%
  \hfill
  \signaturepanel
    {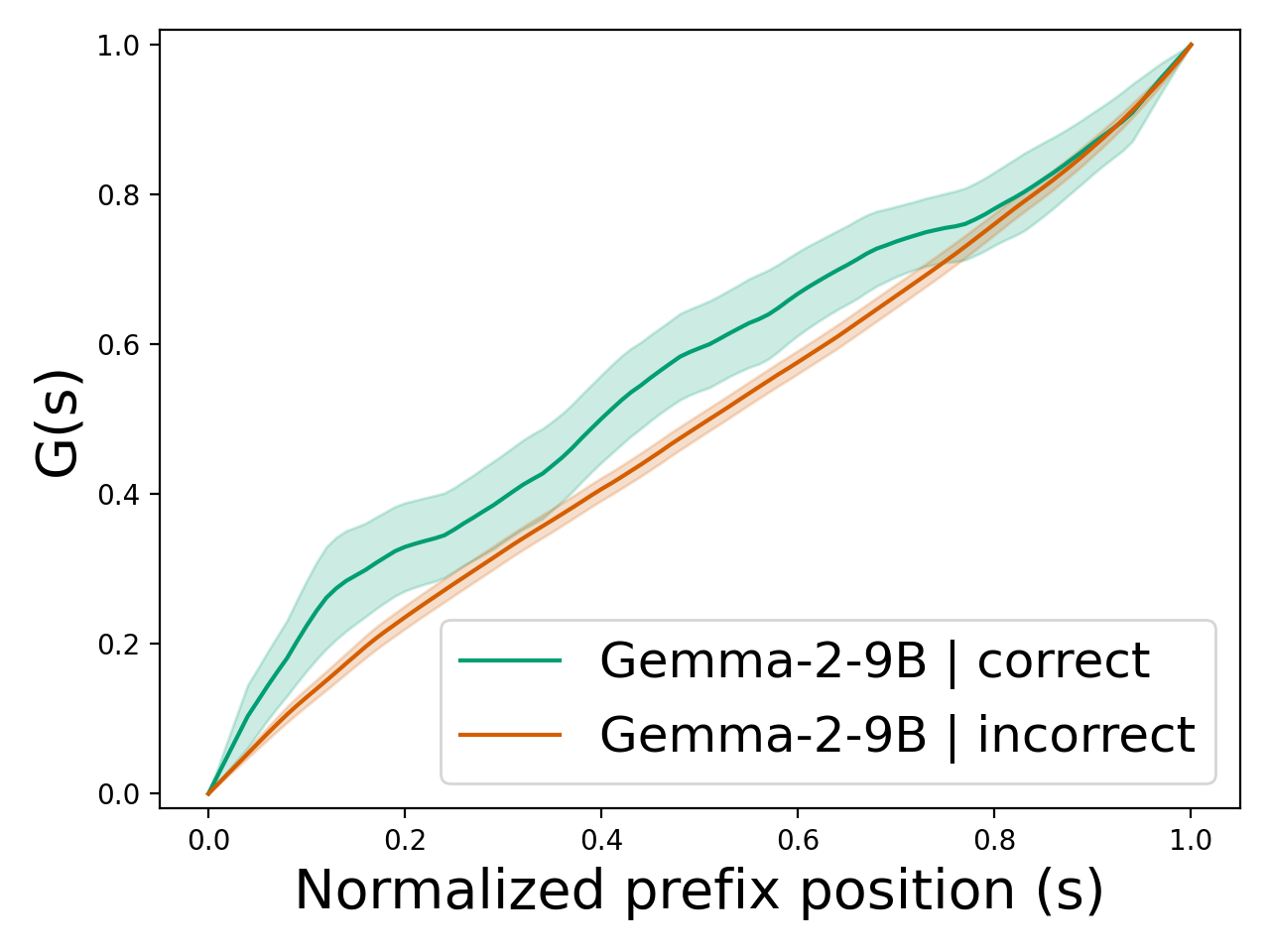}
    {Gemma-2-9B}%
  \hfill
  \signaturepanel
    {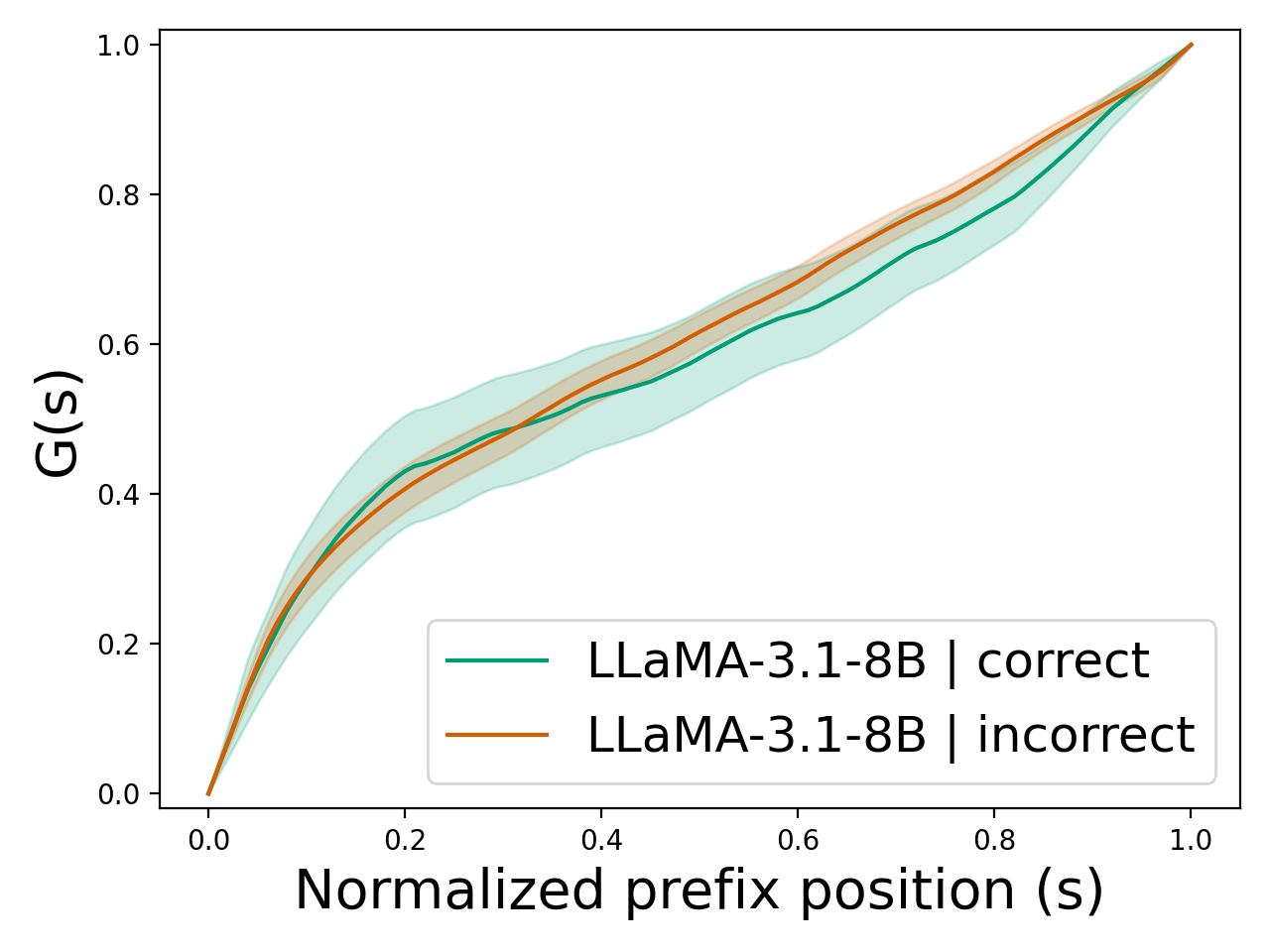}
    {LLaMA-3.1-8B}%

  \par\medskip

  \makebox[\linewidth][c]{%
    \signaturepanel
      {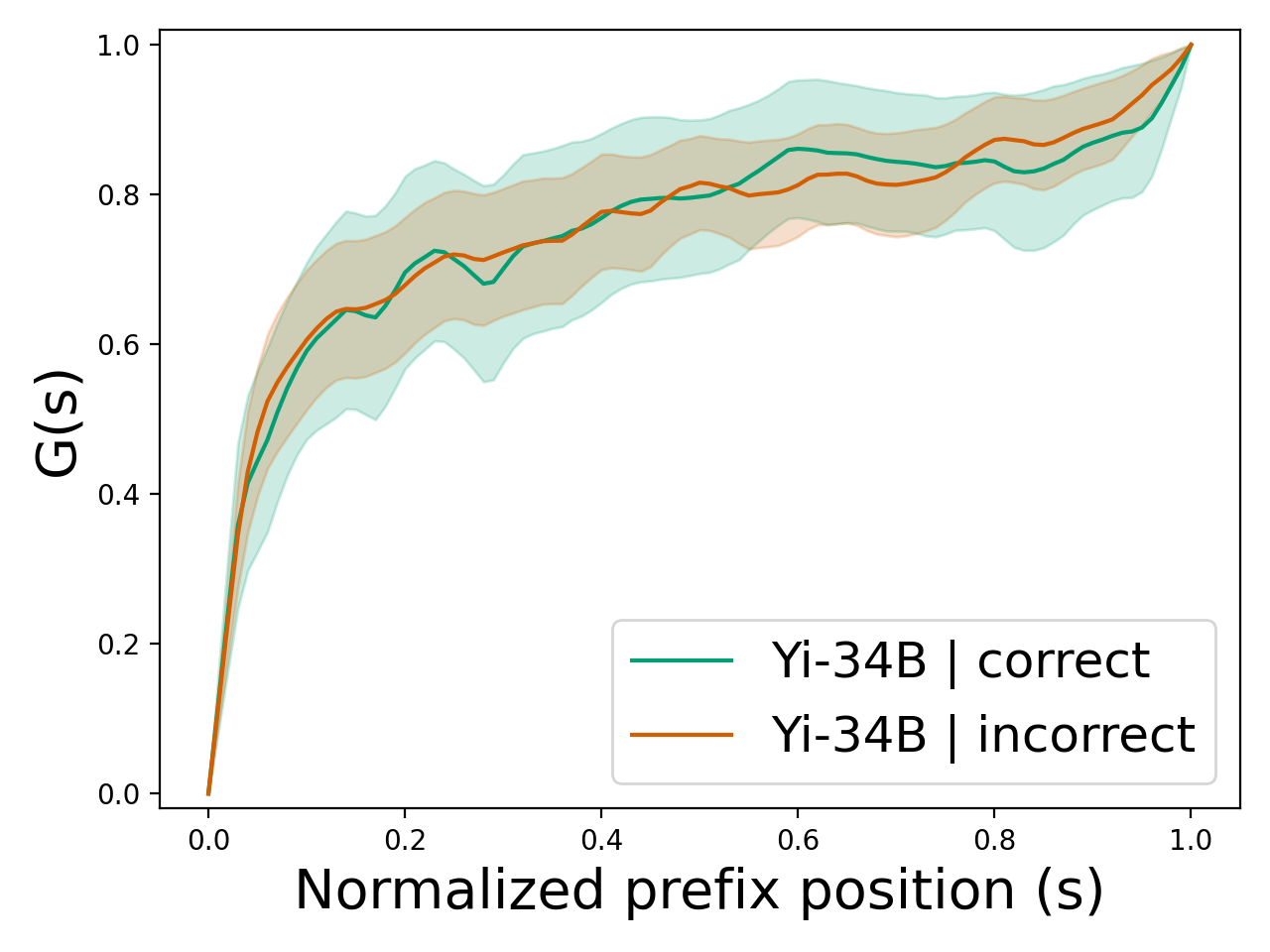}
      {Yi-1.5-34B}%
  }

  \caption{\textbf{Early information accumulation for non-aligned
  model-dataset pairs on GSM8K.}
  Normalized cumulative information gain $G(s)$ versus relative prefix length $s$, split by correctness. Each panel corresponds to one model. Shaded regions denote 95\% bootstrap confidence intervals.}
  \label{fig:nonaligned-earlyinfo-gsm8k}
\end{figure*}

\begin{figure*}[htbp]
  \centering

  \makebox[\linewidth][c]{%
    \signaturepanel
      {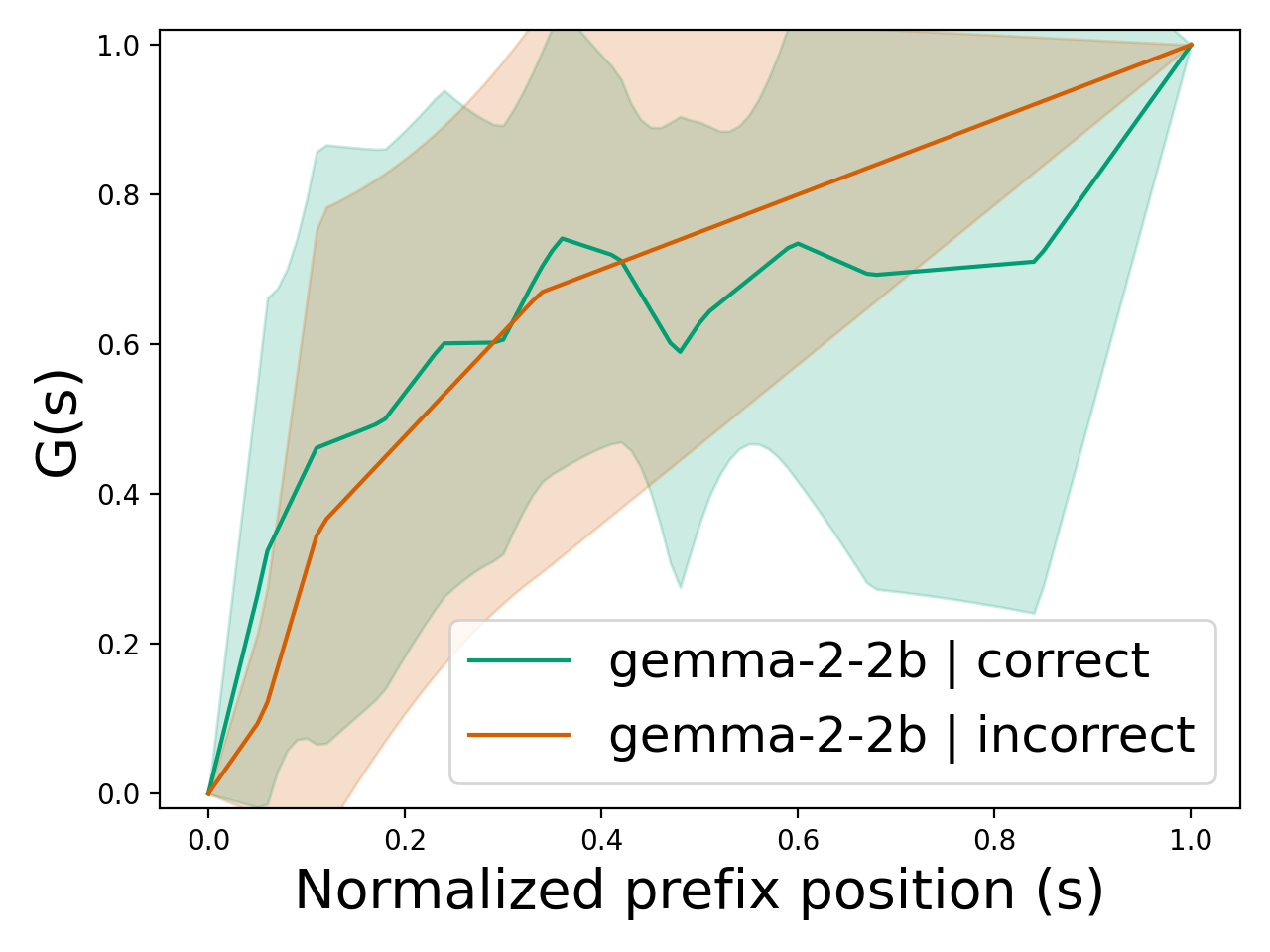}
      {Gemma-2-2B}%
    \hspace{0.015\linewidth}
    \signaturepanel
      {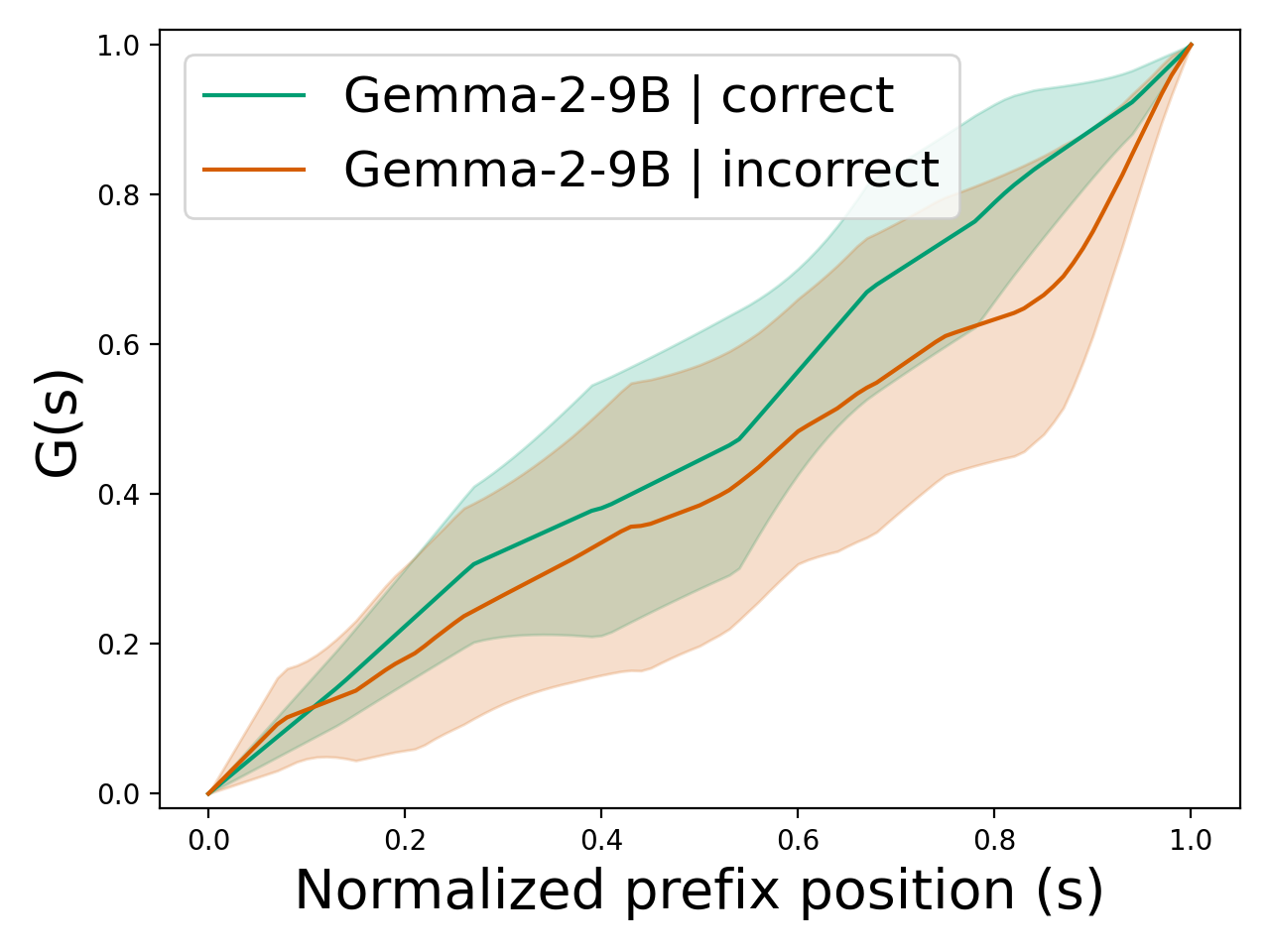}
      {Gemma-2-9B}%
    \hspace{0.015\linewidth}
    \signaturepanel
      {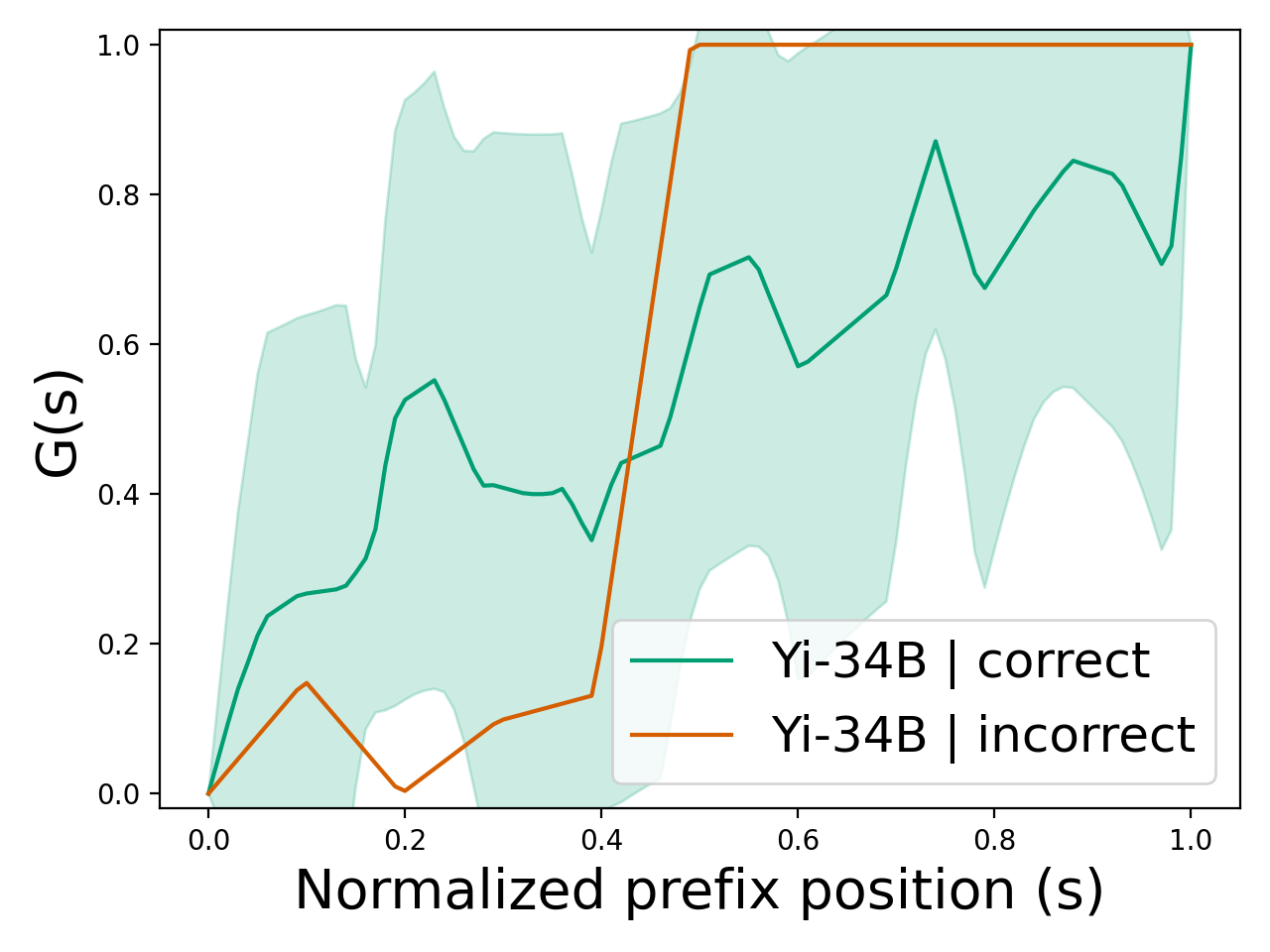}
      {Yi-1.5-34B}%
  }

  \caption{\textbf{Early information accumulation for non-aligned
  model-dataset pairs on SVAMP.}
  Normalized cumulative information gain $G(s)$ versus relative prefix length $s$, split by correctness. Each panel corresponds to one model. Shaded regions denote 95\% bootstrap confidence intervals.}
  \label{fig:nonaligned-earlyinfo-svamp}
\end{figure*}

\begin{figure*}[htbp]
  \centering

  \signaturepanel
    {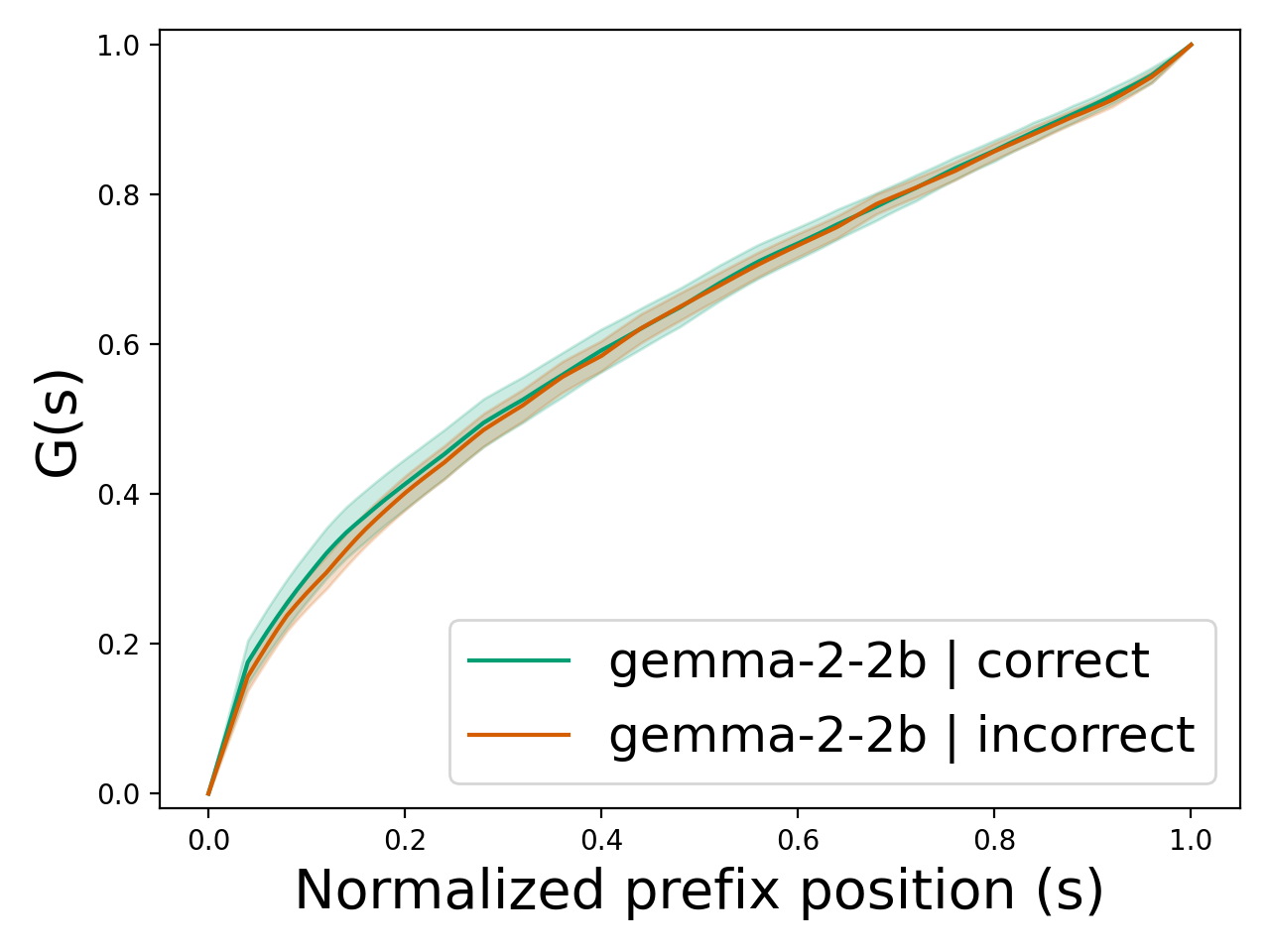}
    {Gemma-2-2B}%
  \hfill
  \signaturepanel
    {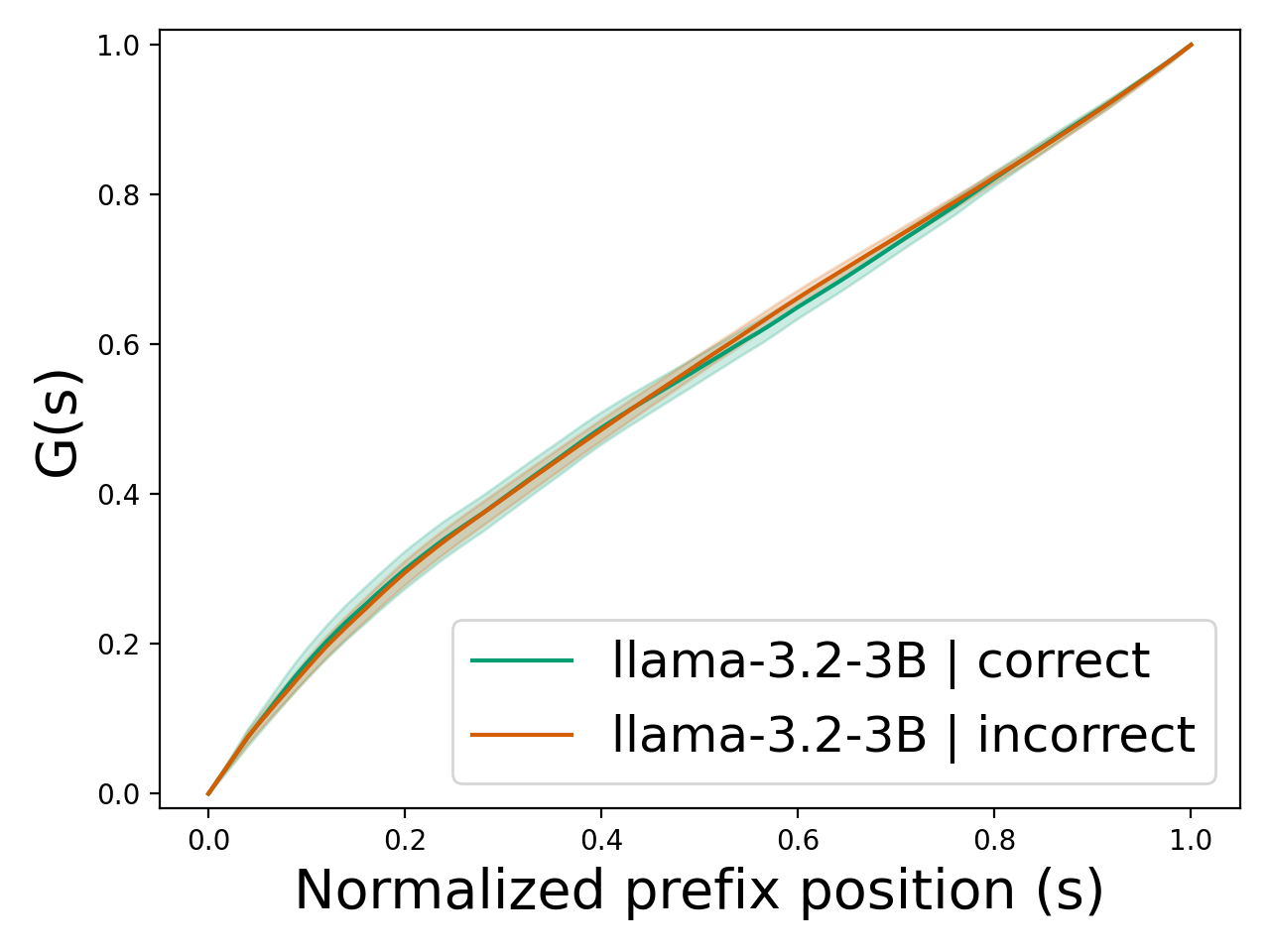}
    {LLaMA-3.2-3B}%
  \hfill
  \signaturepanel
    {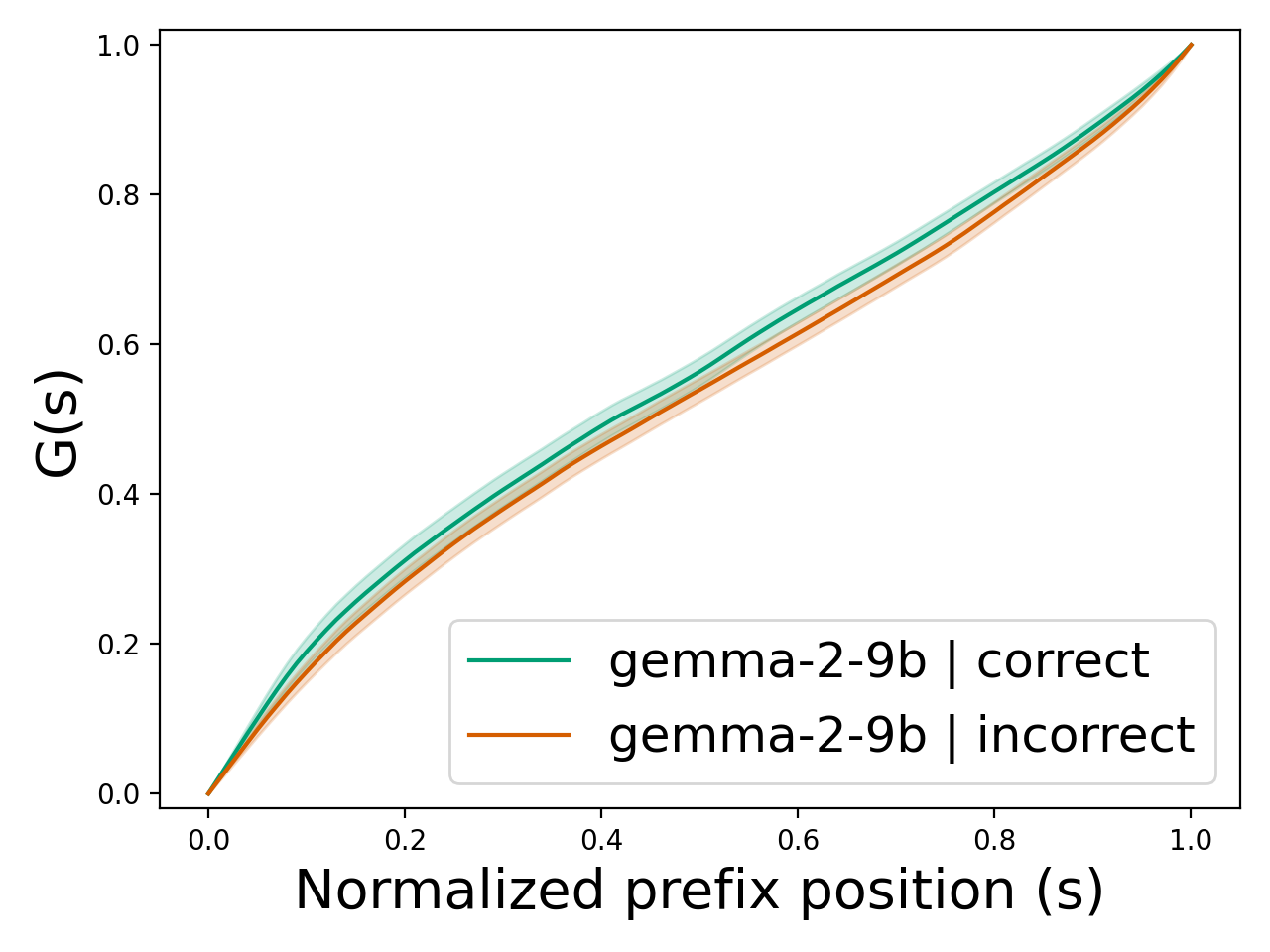}
    {Gemma-2-9B}%
  \hfill
  \signaturepanel
    {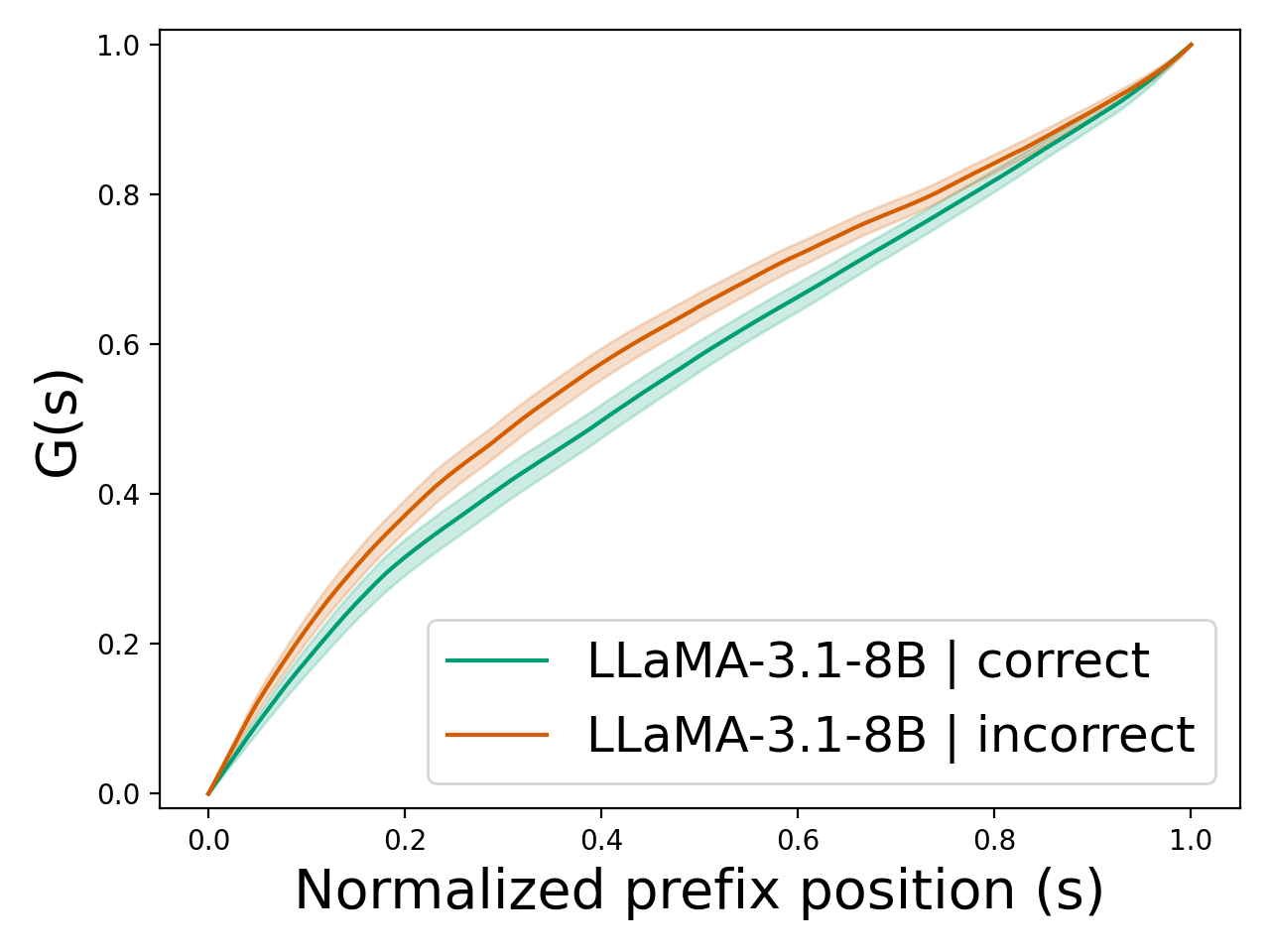}
    {LLaMA-3.1-8B}%

  \caption{\textbf{Early information accumulation for non-aligned
  model-dataset pairs on ARC.}
  Normalized cumulative information gain $G(s)$ versus relative prefix length $s$, split by correctness. Each panel corresponds to one model. Shaded regions denote 95\% bootstrap confidence intervals.}
  \label{fig:nonaligned-earlyinfo-arc}
\end{figure*}

\begin{figure*}[htbp]
  \centering

  \signaturepanel
    {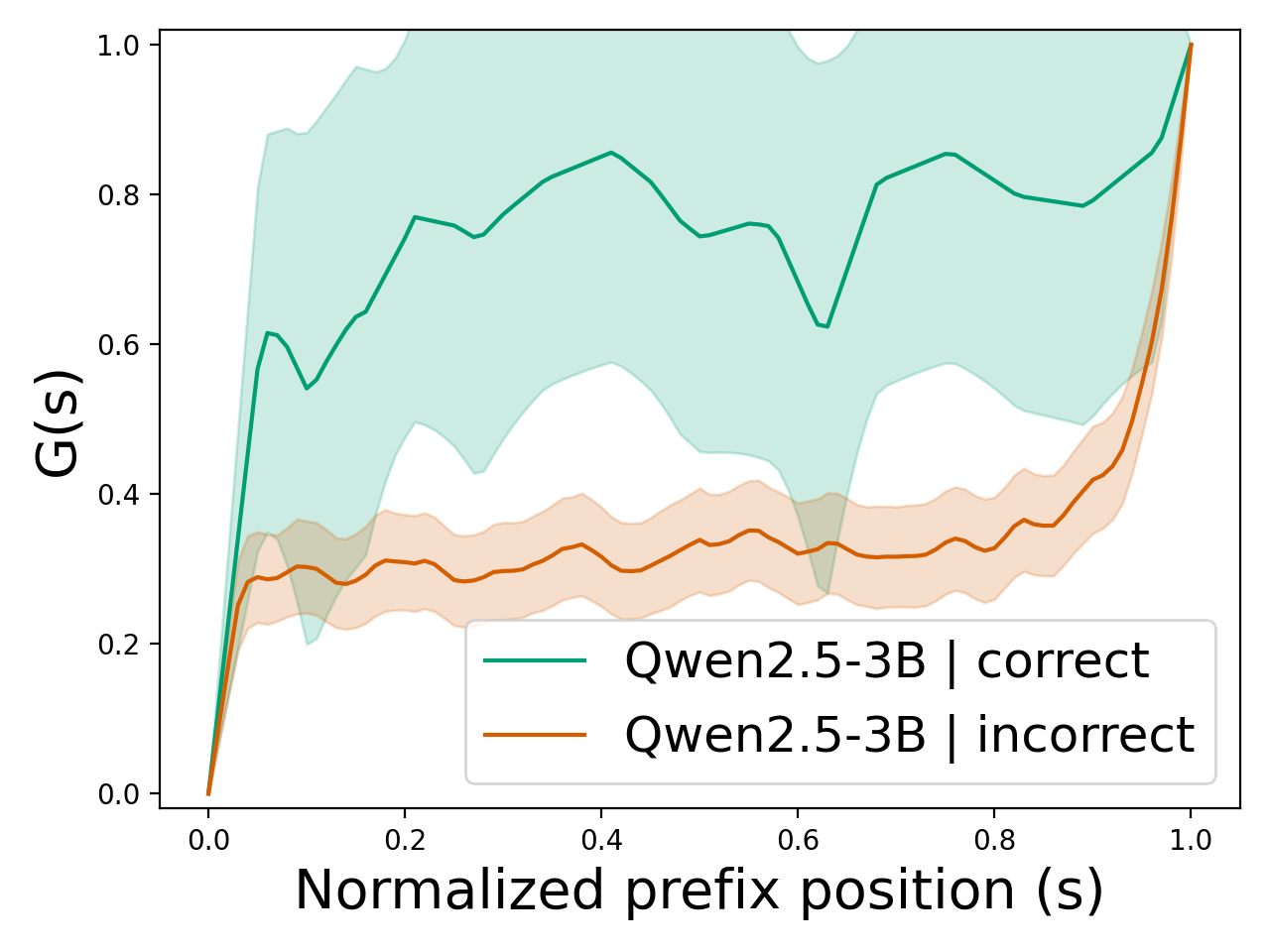}
    {Qwen2.5-3B}%
  \hfill
  \signaturepanel
    {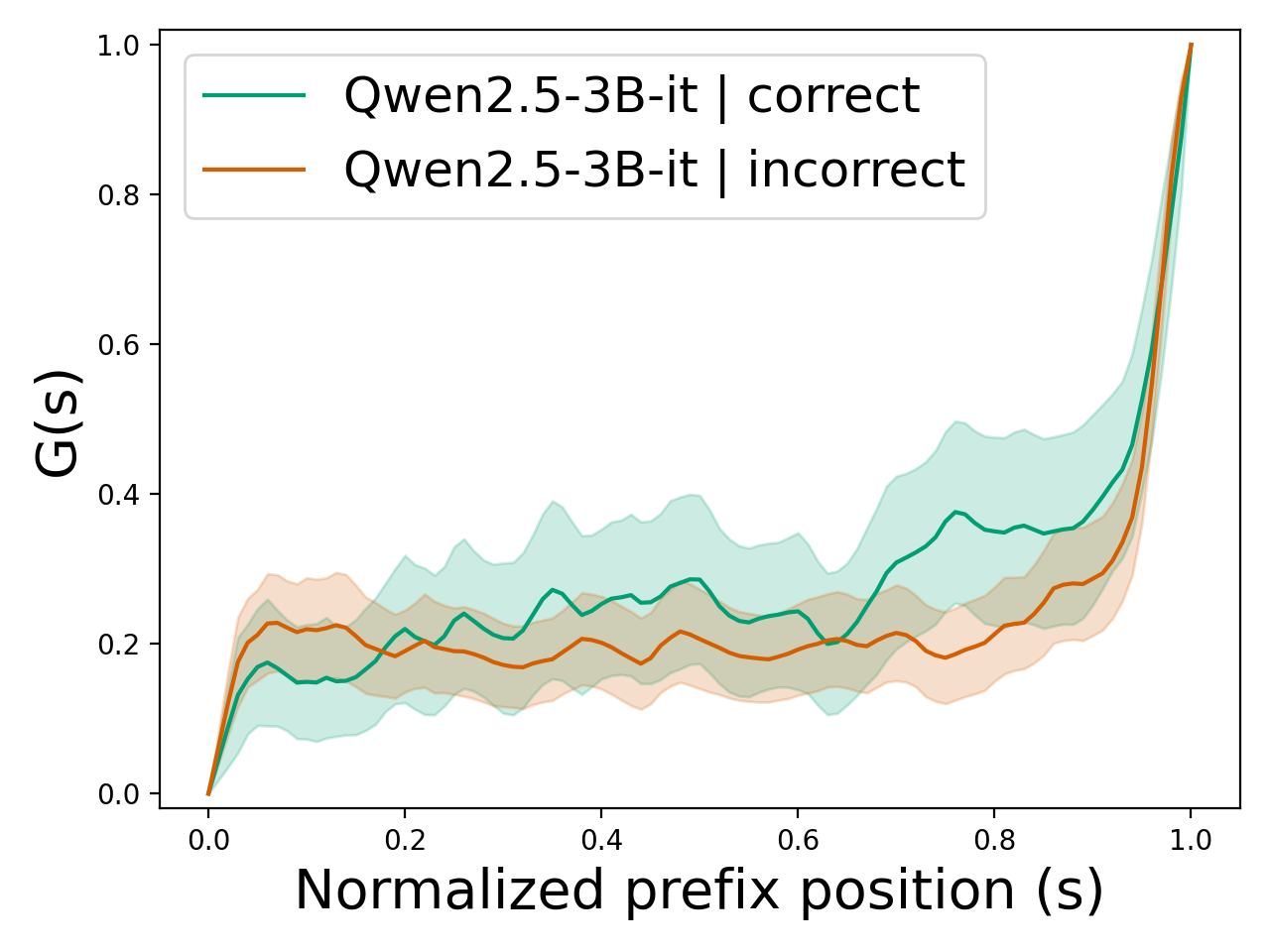}
    {Qwen2.5-3B-IT}%
  \hfill
  \signaturepanel
    {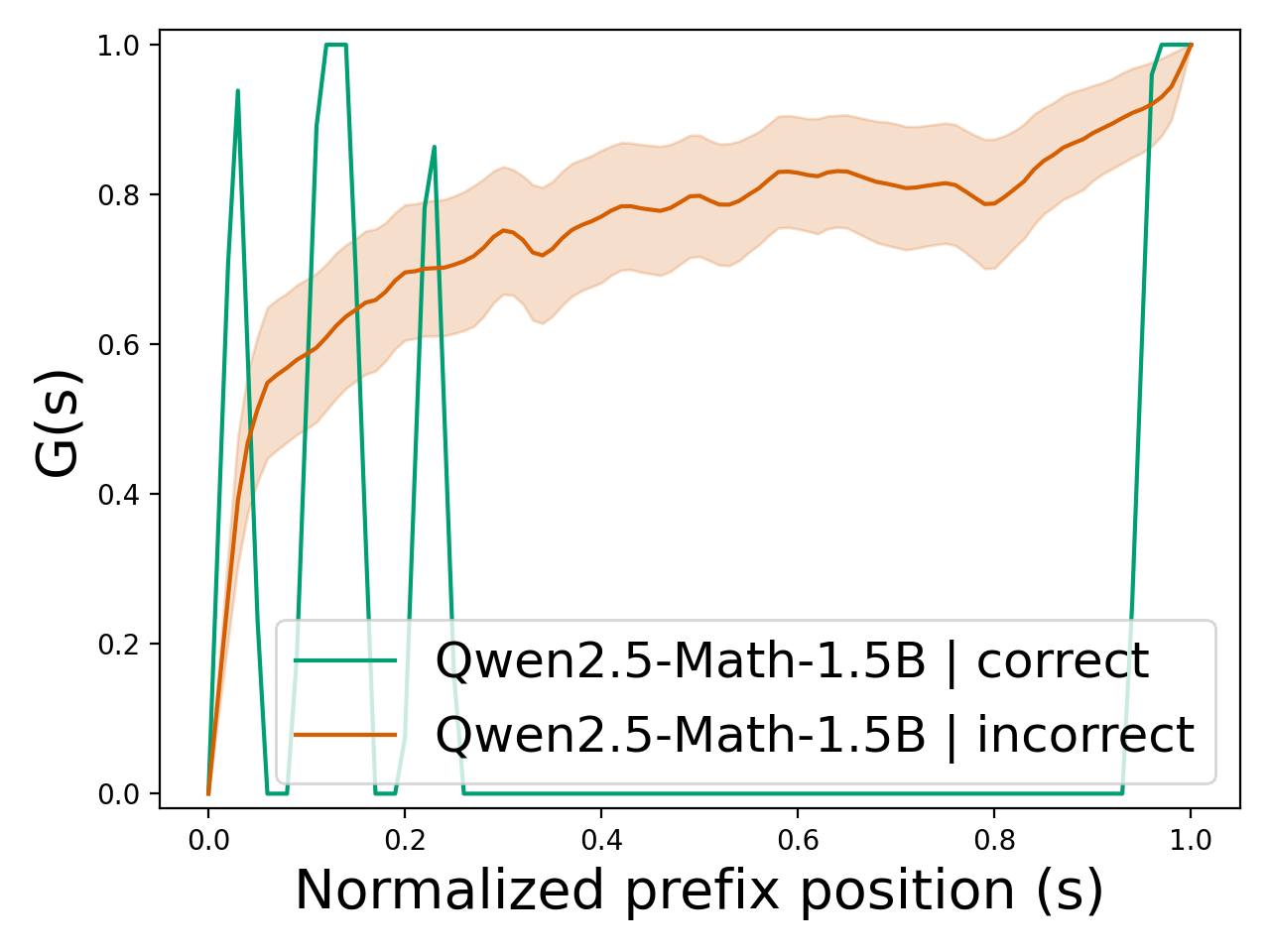}
    {Qwen2.5-Math-1.5B}%
  \hfill
  \signaturepanel
    {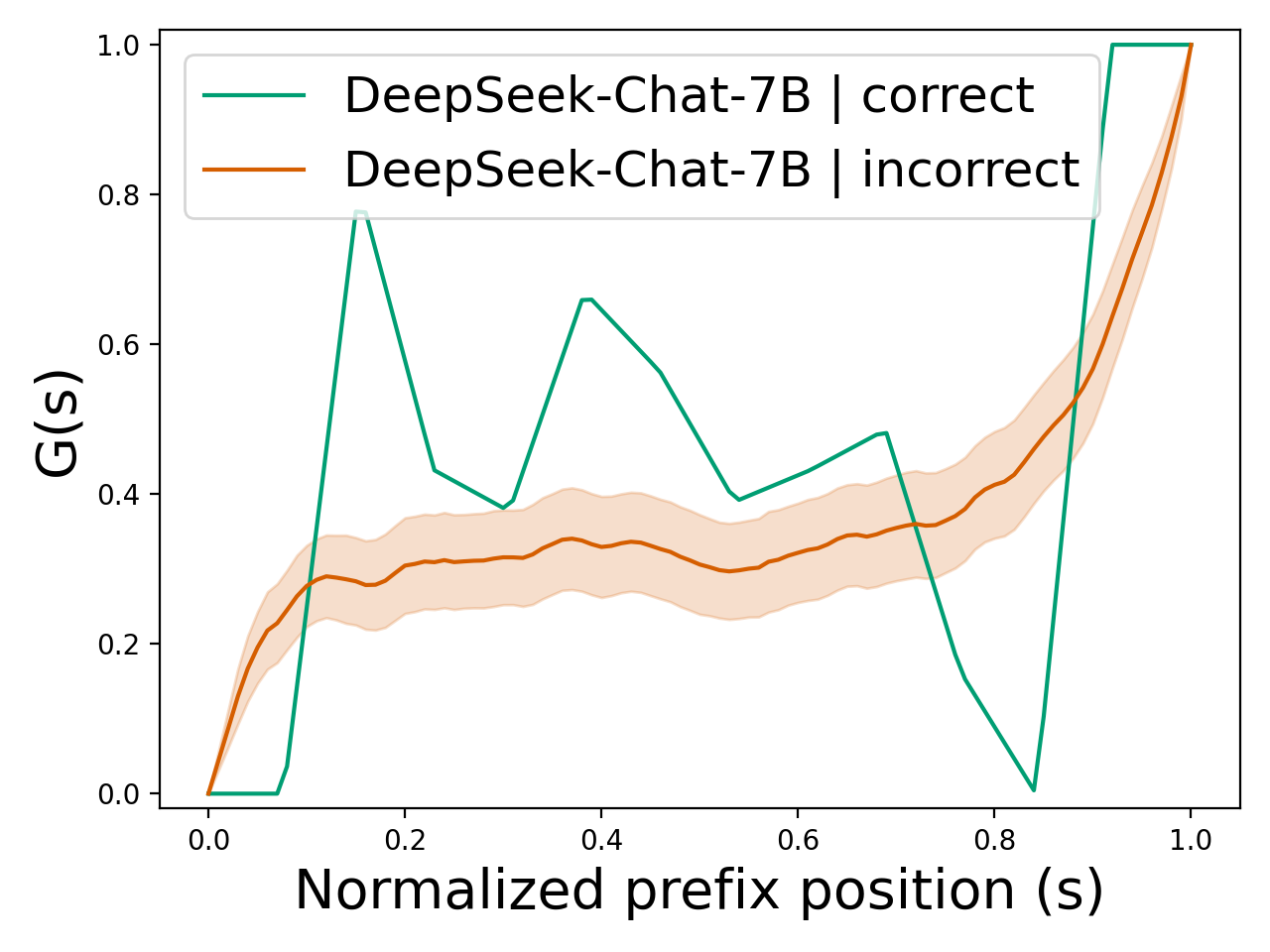}
    {DeepSeek-Chat-7B}%

  \par\medskip

  \signaturepanel
    {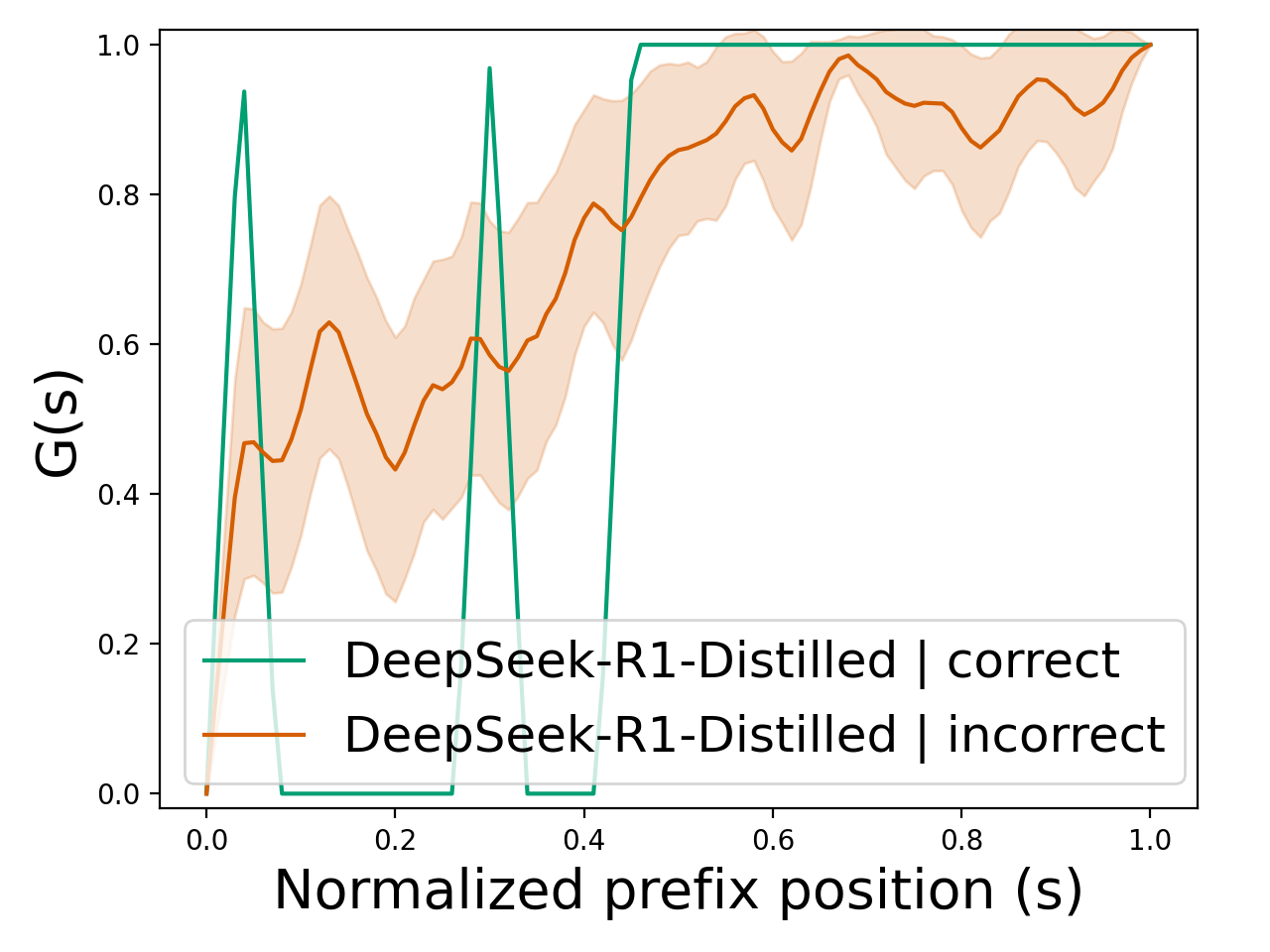}
    {DeepSeek-R1-Distilled}%
  \hfill
  \signaturepanel
    {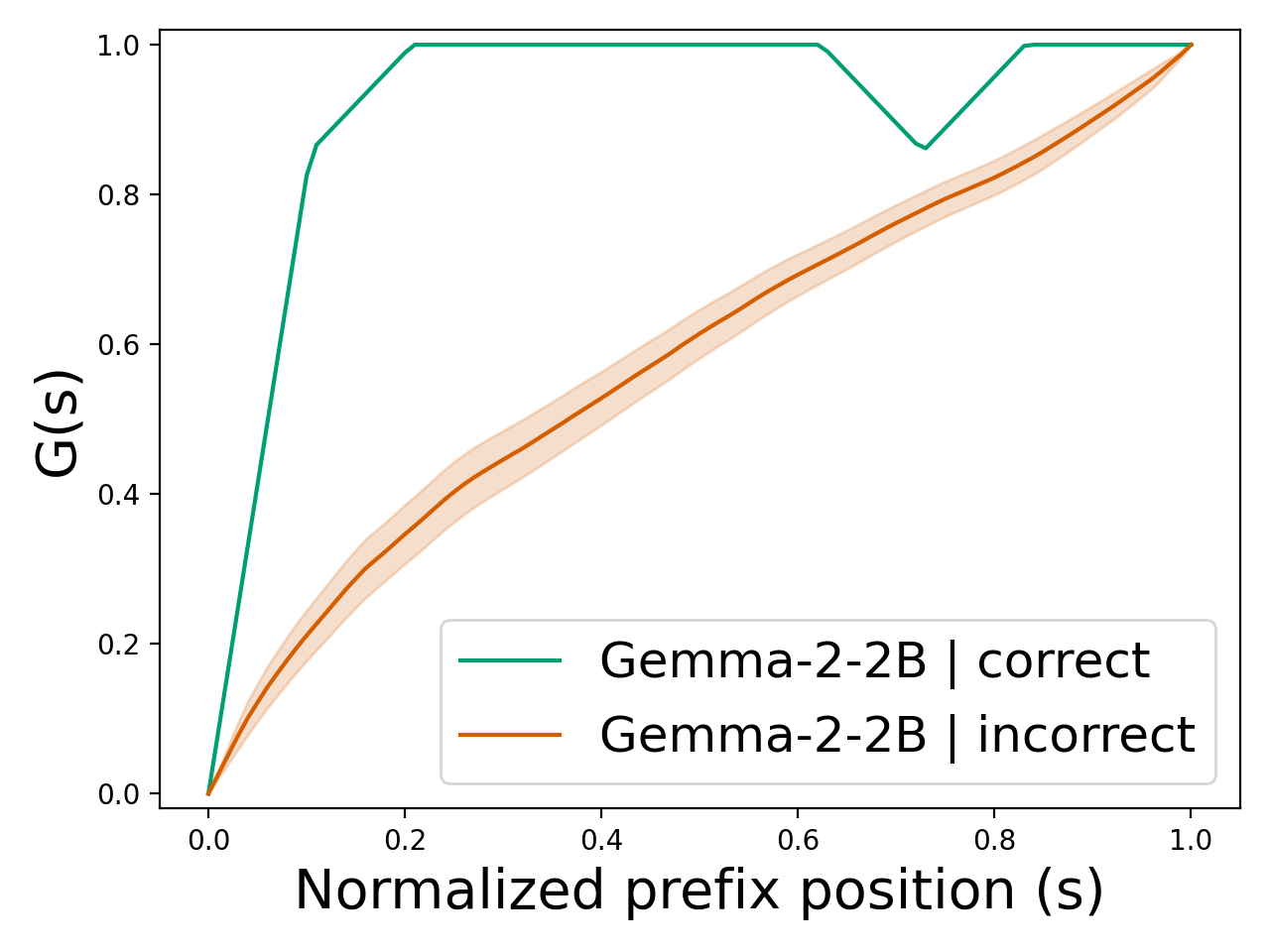}
    {Gemma-2-2B}%
  \hfill
  \signaturepanel
    {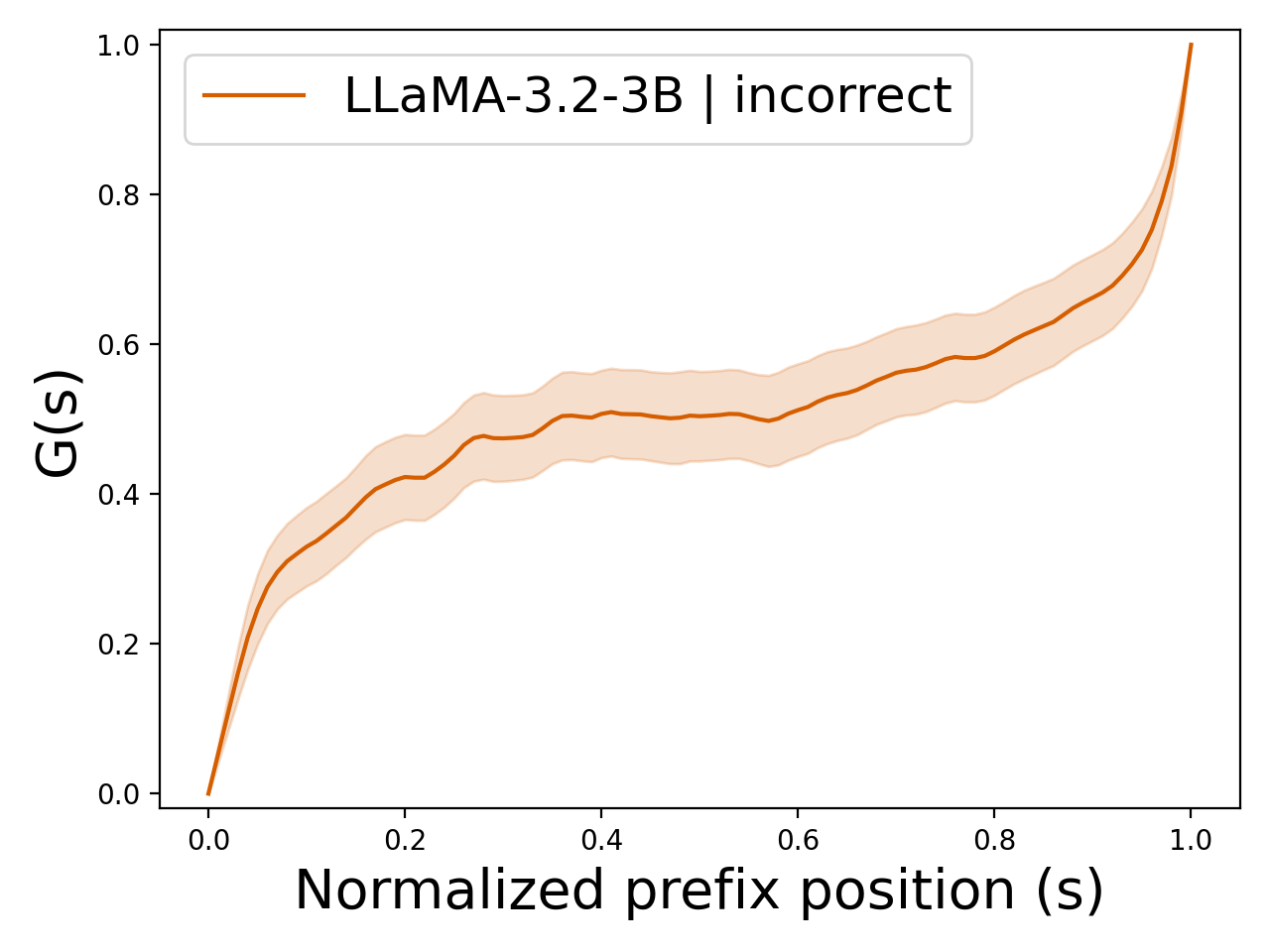}
    {LLaMA-3.2-3B}%
  \hfill
  \signaturepanel
    {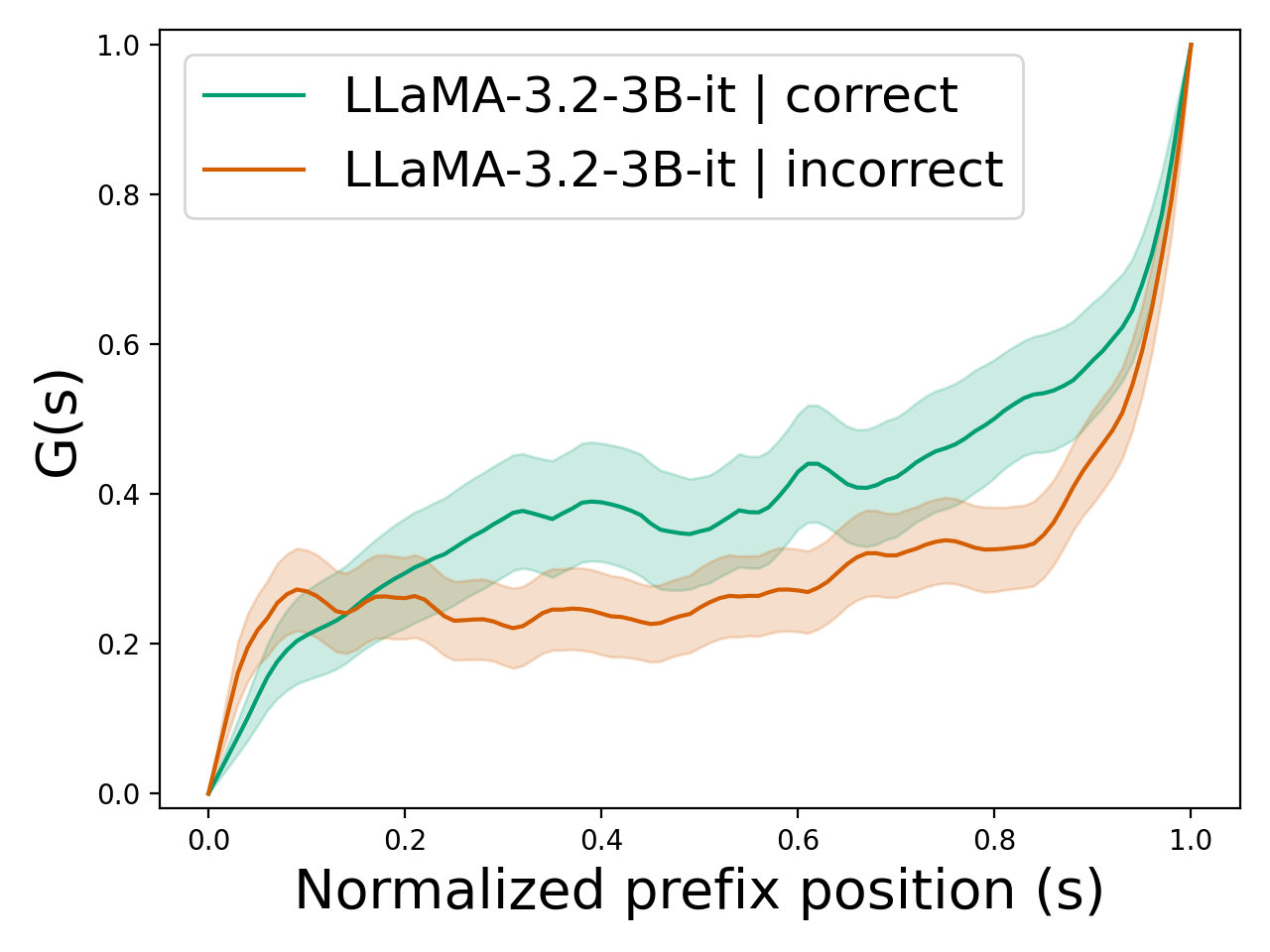}
    {LLaMA-3.2-3B-IT}%

  \par\medskip

  \makebox[\linewidth][c]{%
    \signaturepanel
      {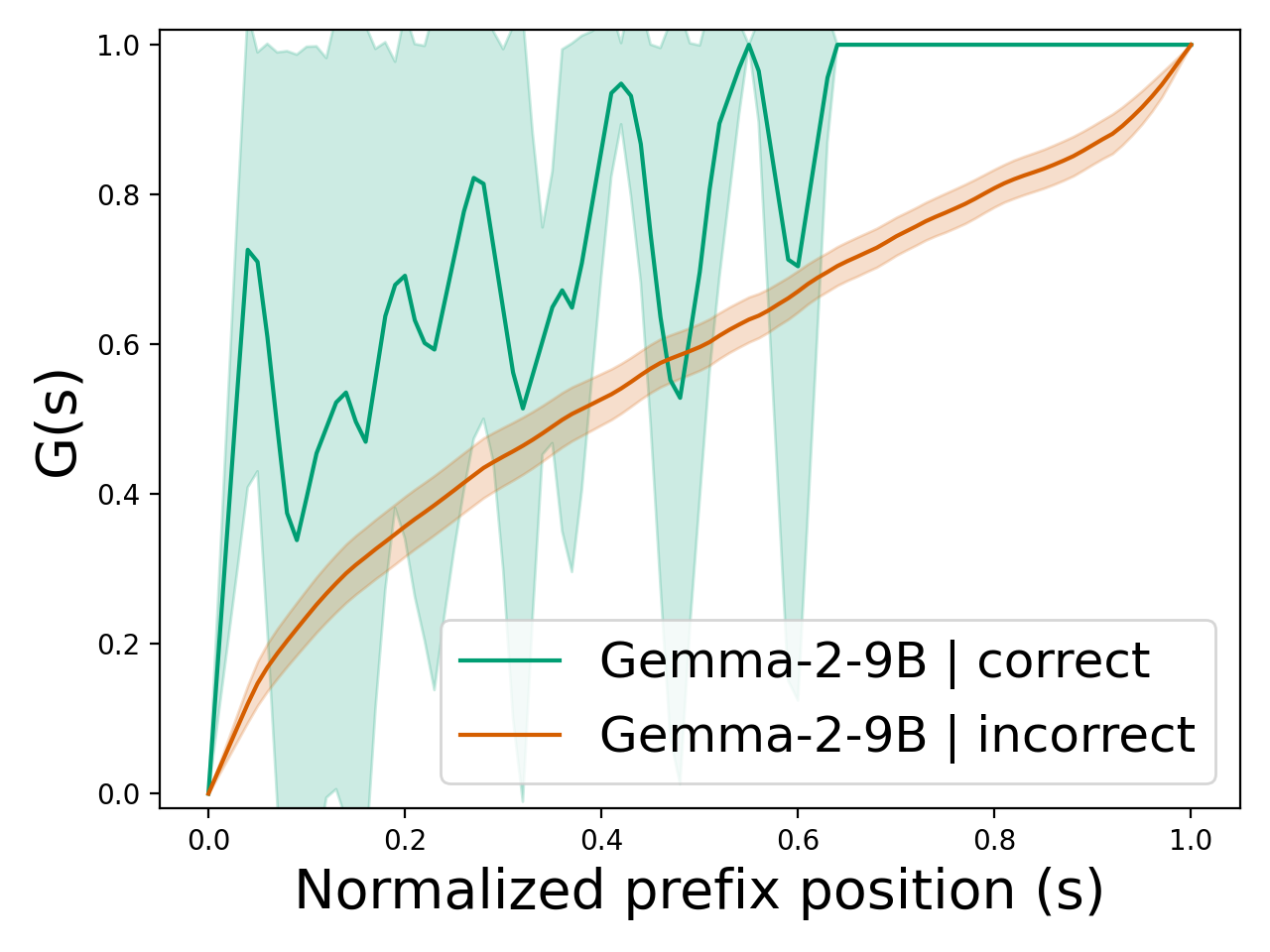}
      {Gemma-2-9B}%
    \hspace{0.015\linewidth}
    \signaturepanel
      {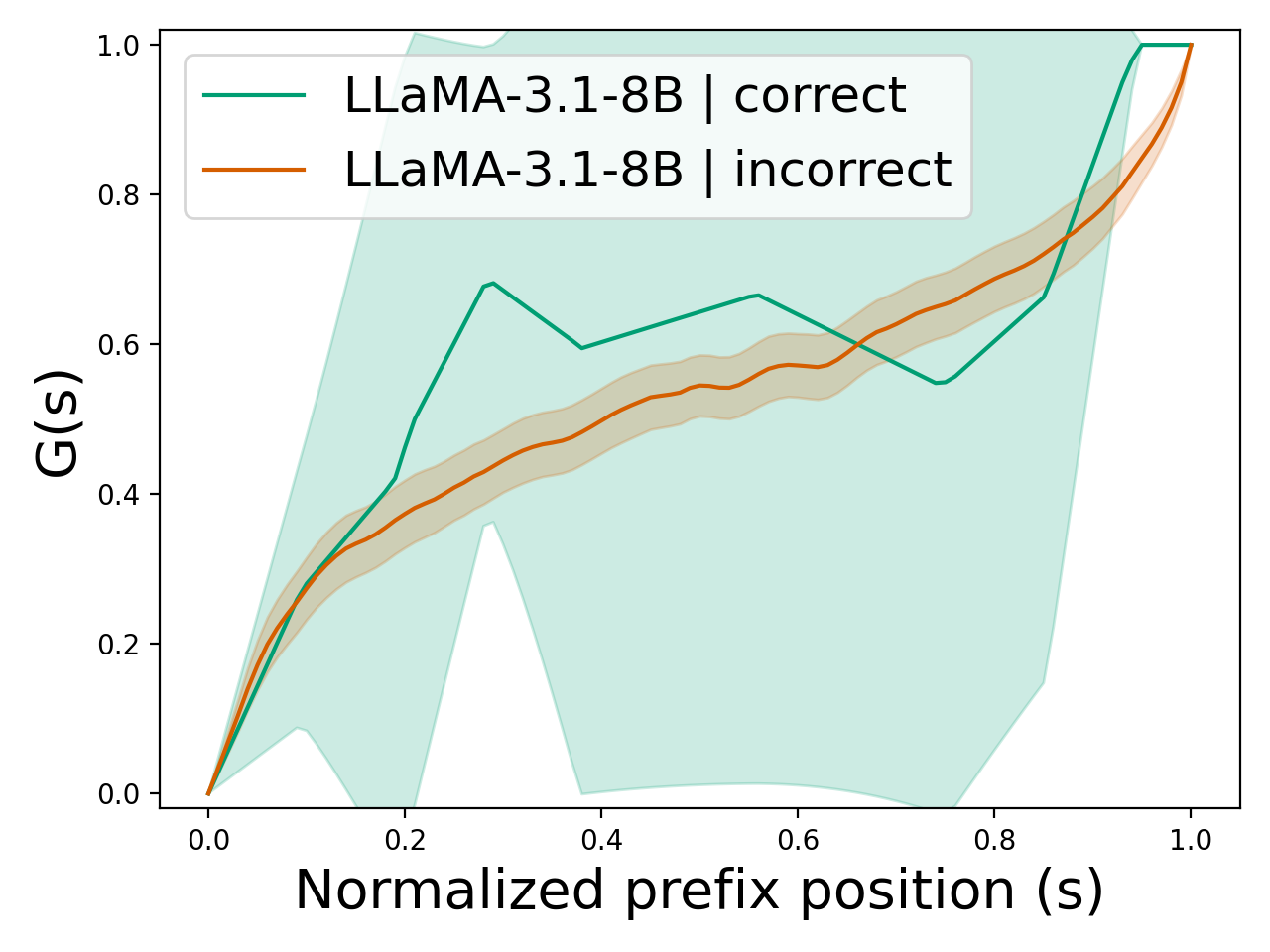}
      {LLaMA-3.1-8B}%
    \hspace{0.015\linewidth}
    \signaturepanel
      {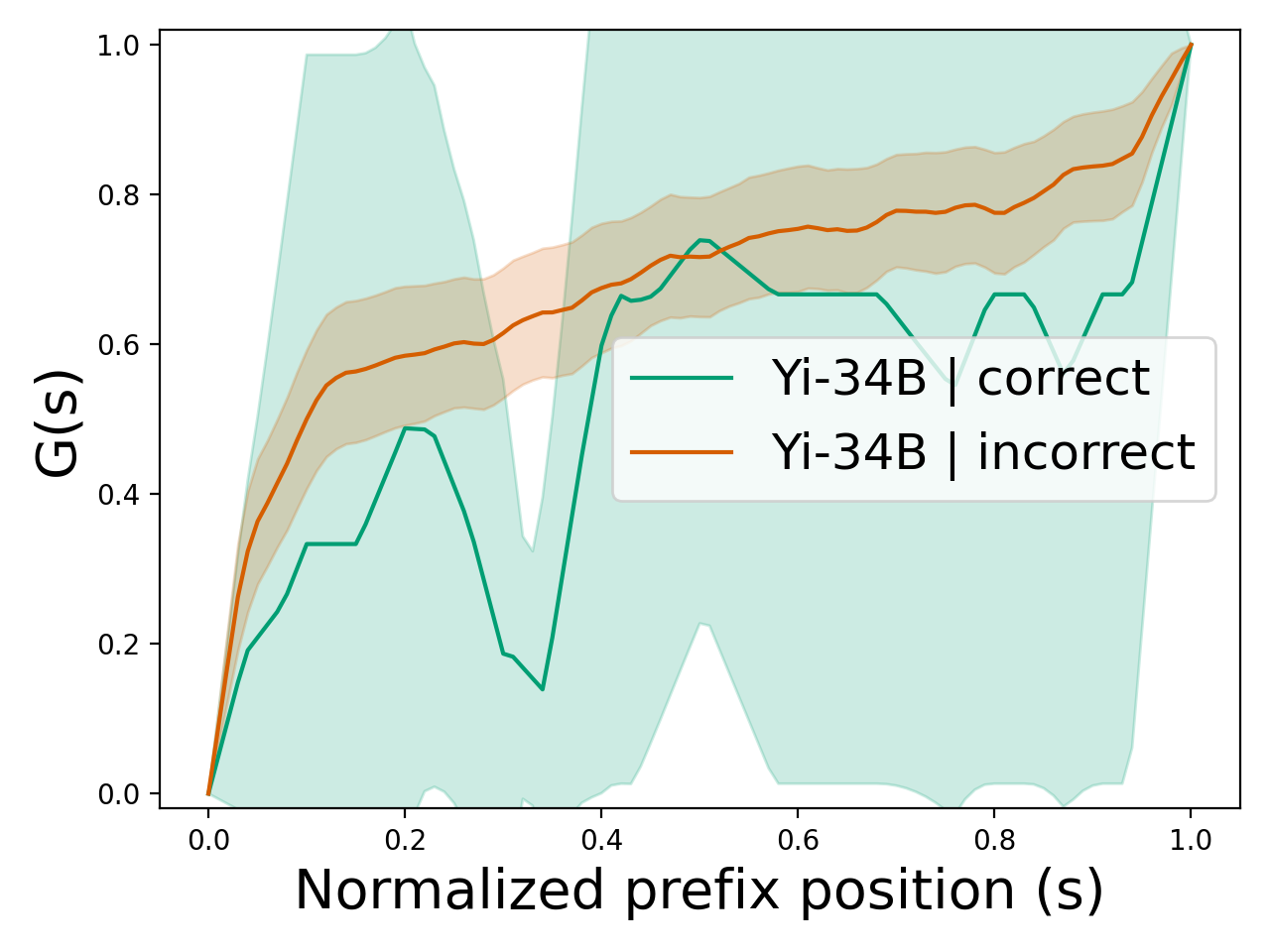}
      {Yi-1.5-34B}%
  }

  \caption{\textbf{Early information accumulation for non-aligned
  model-dataset pairs on MATH.}
  Normalized cumulative information gain $G(s)$ versus relative prefix length $s$, split by correctness. Each panel corresponds to one model. Shaded regions denote 95\% bootstrap confidence intervals.}
  \label{fig:nonaligned-earlyinfo-math}
\end{figure*}

\subsection{Early separability}
\label{subsec:appendix-separability}

Early separability is measured using the AUC obtained when conditional answer entropy at relative prefix length $s$ is used to distinguish correct from incorrect traces. The horizontal reference line denotes chance performance, $\mathrm{AUC}=0.5$, and shaded regions denote 95\% bootstrap confidence intervals.

\paragraph{Aligned model-dataset pairs.}
Figures~\ref{fig:aligned-separability-gsm8k}-\ref{fig:aligned-separability-math} report the separability curves for the model-dataset pairs classified as aligned. Each figure contains all aligned models for the indicated dataset.

\begin{figure}[htbp]
  \centering
  \includegraphics[width=0.72\linewidth]{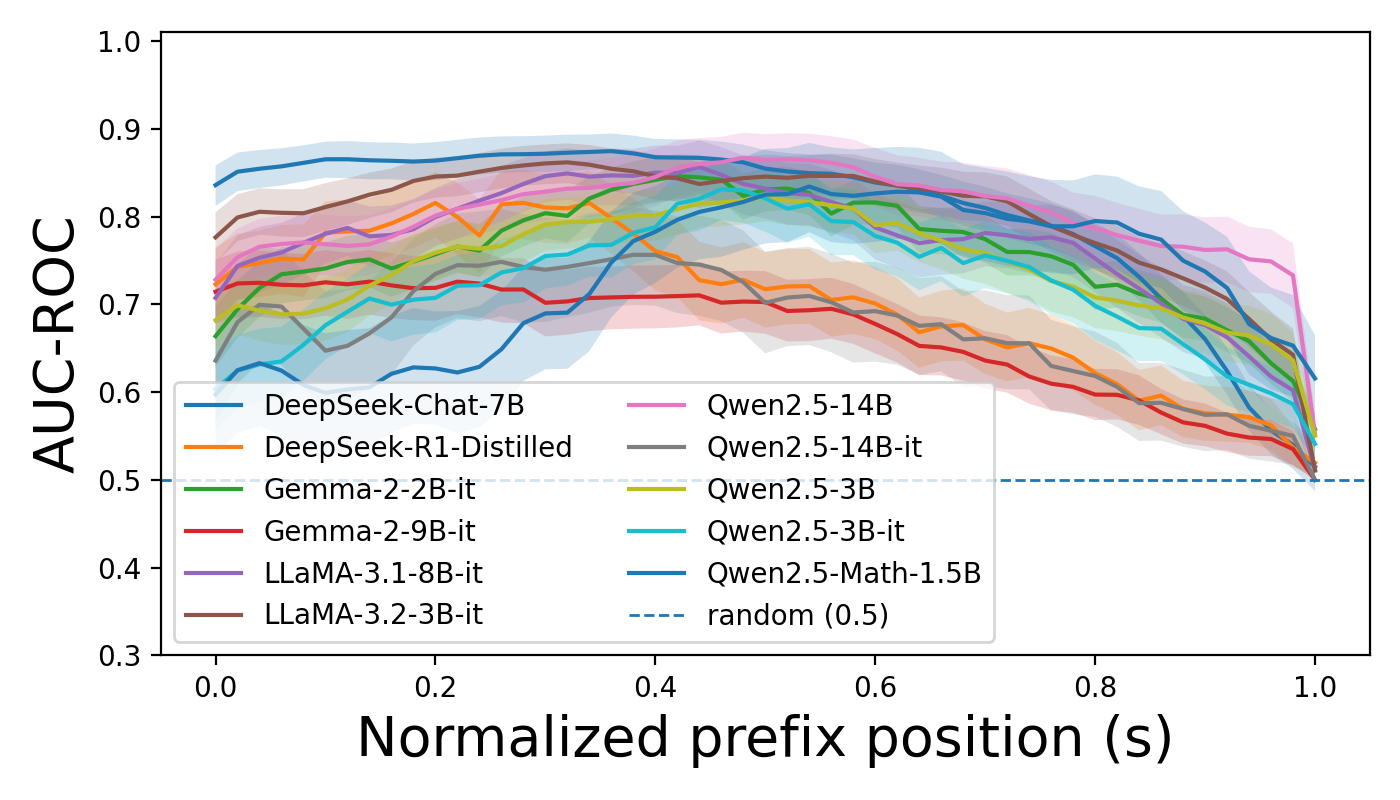}
  \caption{\textbf{Early separability for aligned model-dataset pairs on GSM8K.}
  AUC for using conditional answer entropy to distinguish correct from incorrect traces as a function of relative prefix length $s$.}
  \label{fig:aligned-separability-gsm8k}
\end{figure}

\begin{figure}[htbp]
  \centering
  \includegraphics[width=0.72\linewidth]{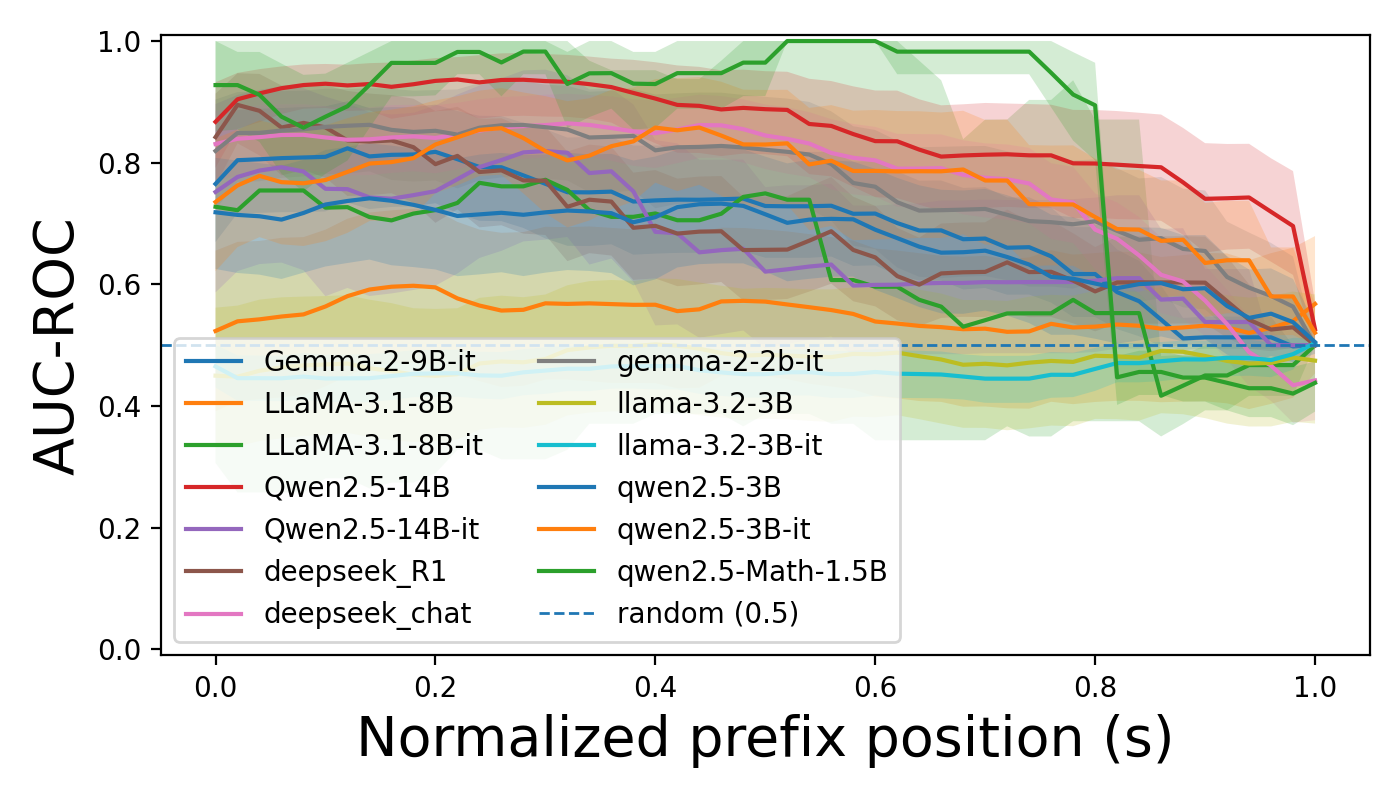}
  \caption{\textbf{Early separability for aligned model-dataset pairs on SVAMP.}
  AUC for using conditional answer entropy to distinguish correct from incorrect traces as a function of relative prefix length $s$.}
  \label{fig:aligned-separability-svamp}
\end{figure}

\begin{figure}[htbp]
  \centering
  \includegraphics[width=0.72\linewidth]
    {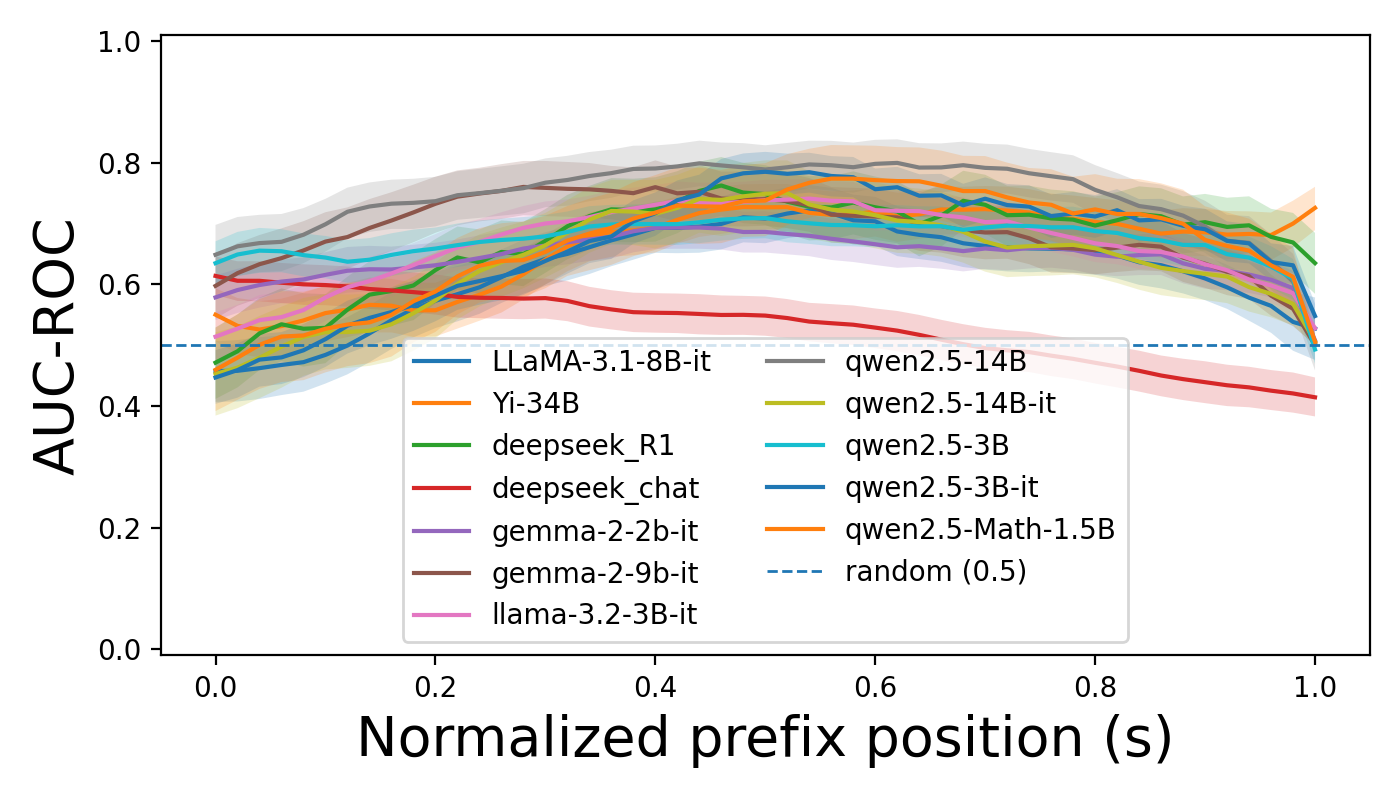}
  \caption{\textbf{Early separability for aligned model-dataset pairs
  on ARC.}
  AUC for using conditional answer entropy to distinguish correct from
  incorrect traces as a function of relative prefix length $s$.
  }
  \label{fig:aligned-separability-arc}
\end{figure}

\begin{figure}[htbp]
  \centering
  \includegraphics[width=0.72\linewidth]
    {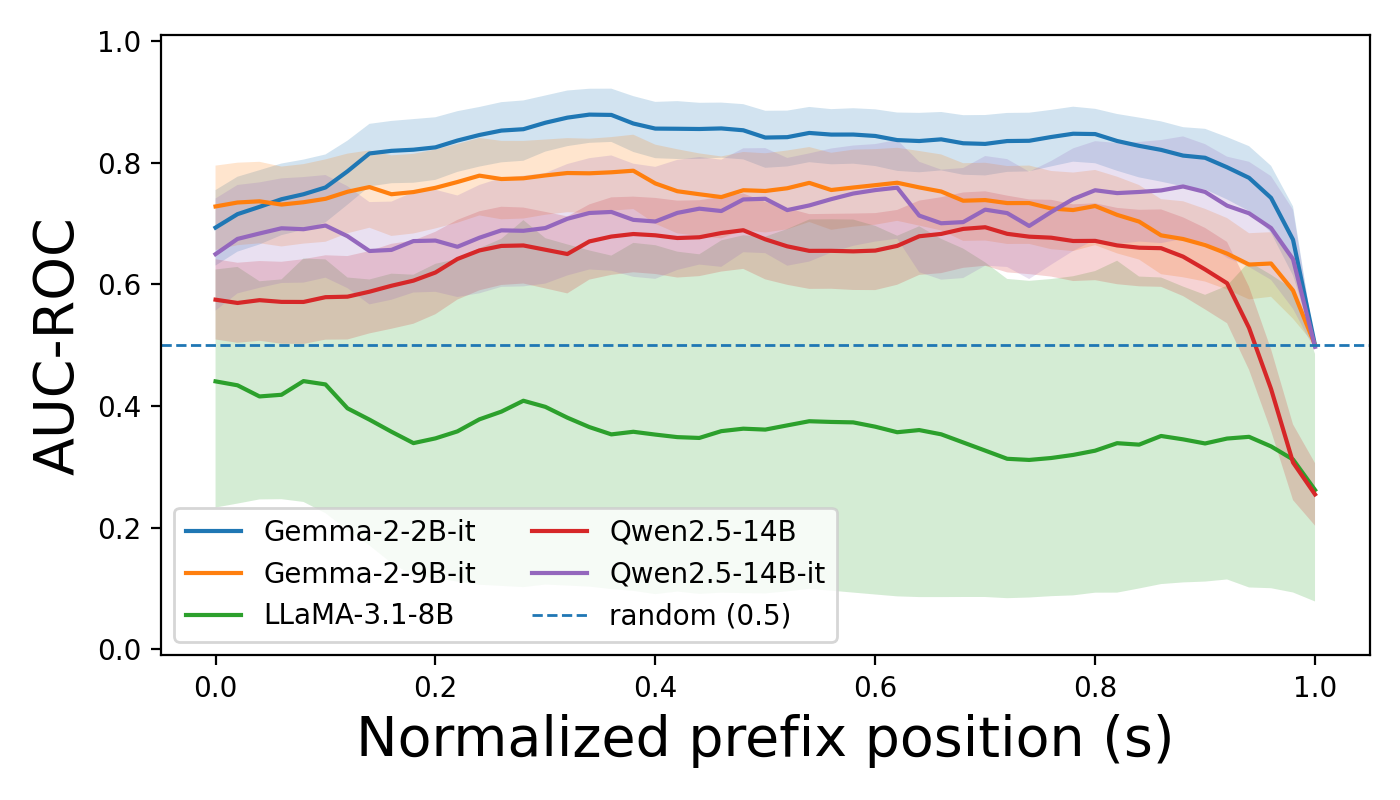}
  \caption{\textbf{Early separability for aligned model-dataset pairs
  on MATH.}
  AUC for using conditional answer entropy to distinguish correct from
  incorrect traces as a function of relative prefix length $s$.}
  \label{fig:aligned-separability-math}
\end{figure}

\paragraph{Non-aligned model-dataset pairs.}
Figures~\ref{fig:nonaligned-separability-gsm8k}-\ref{fig:nonaligned-separability-math} report the separability curves for the model-dataset pairs classified as non-aligned. Each figure contains all non-aligned models for the indicated dataset.

\begin{figure}[htbp]
  \centering
  \includegraphics[width=0.72\linewidth]
    {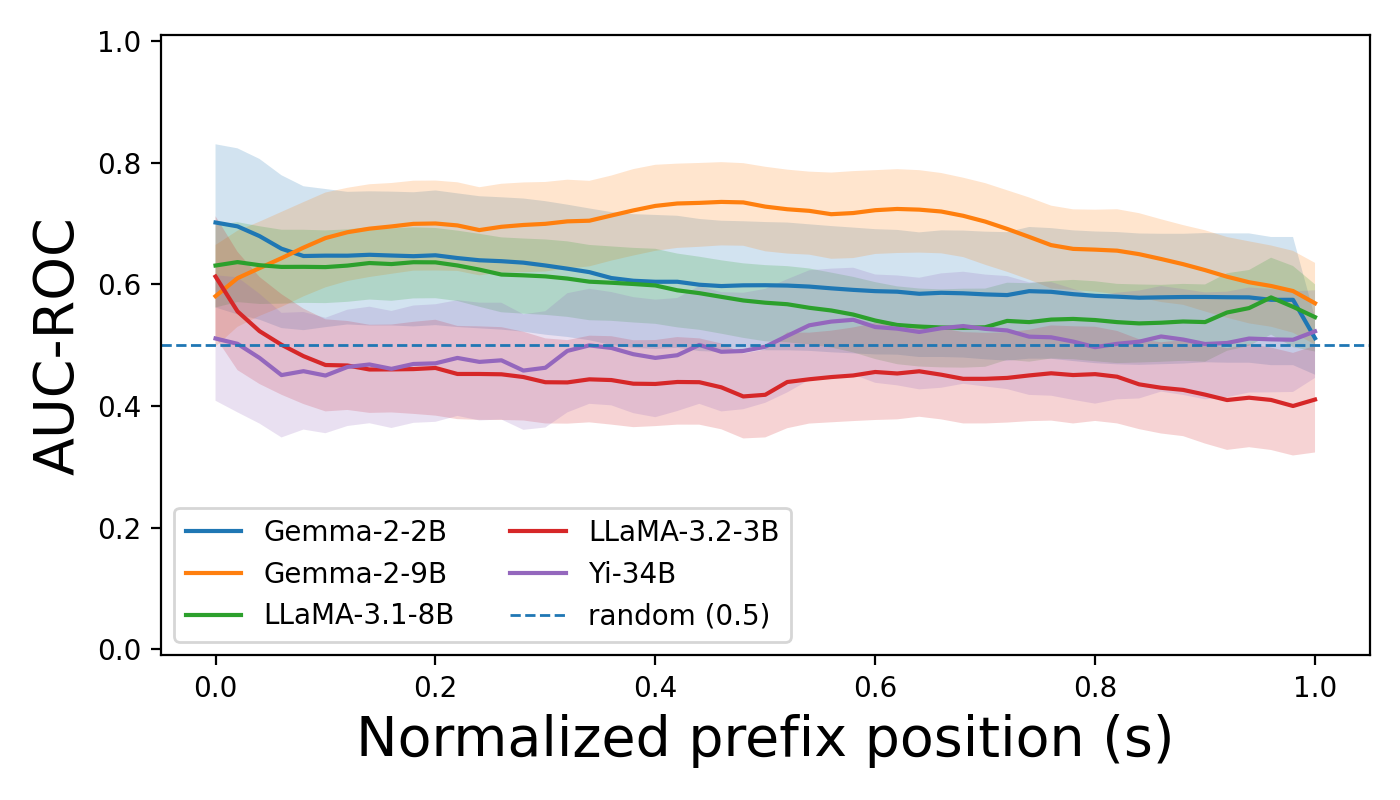}
  \caption{\textbf{Early separability for non-aligned model-dataset
  pairs on GSM8K.}
  AUC for using conditional answer entropy to distinguish correct from incorrect traces as a function of relative prefix length $s$.}
  \label{fig:nonaligned-separability-gsm8k}
\end{figure}

\begin{figure}[htbp]
  \centering
  \includegraphics[width=0.72\linewidth]
    {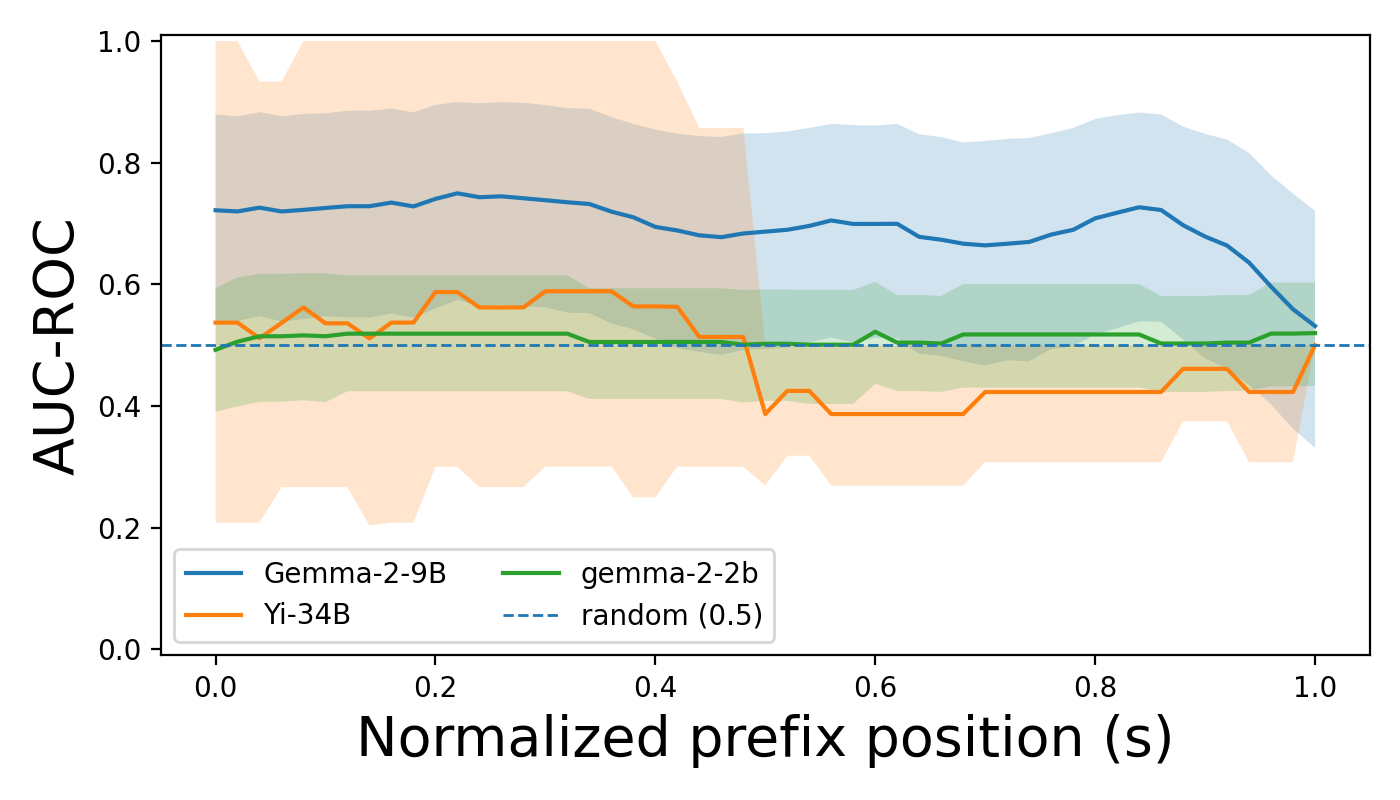}
  \caption{\textbf{Early separability for non-aligned model-dataset pairs on SVAMP.}
  AUC for using conditional answer entropy to distinguish correct from incorrect traces as a function of relative prefix length $s$.}
  \label{fig:nonaligned-separability-svamp}
\end{figure}

\begin{figure}[htbp]
  \centering
  \includegraphics[width=0.72\linewidth]
    {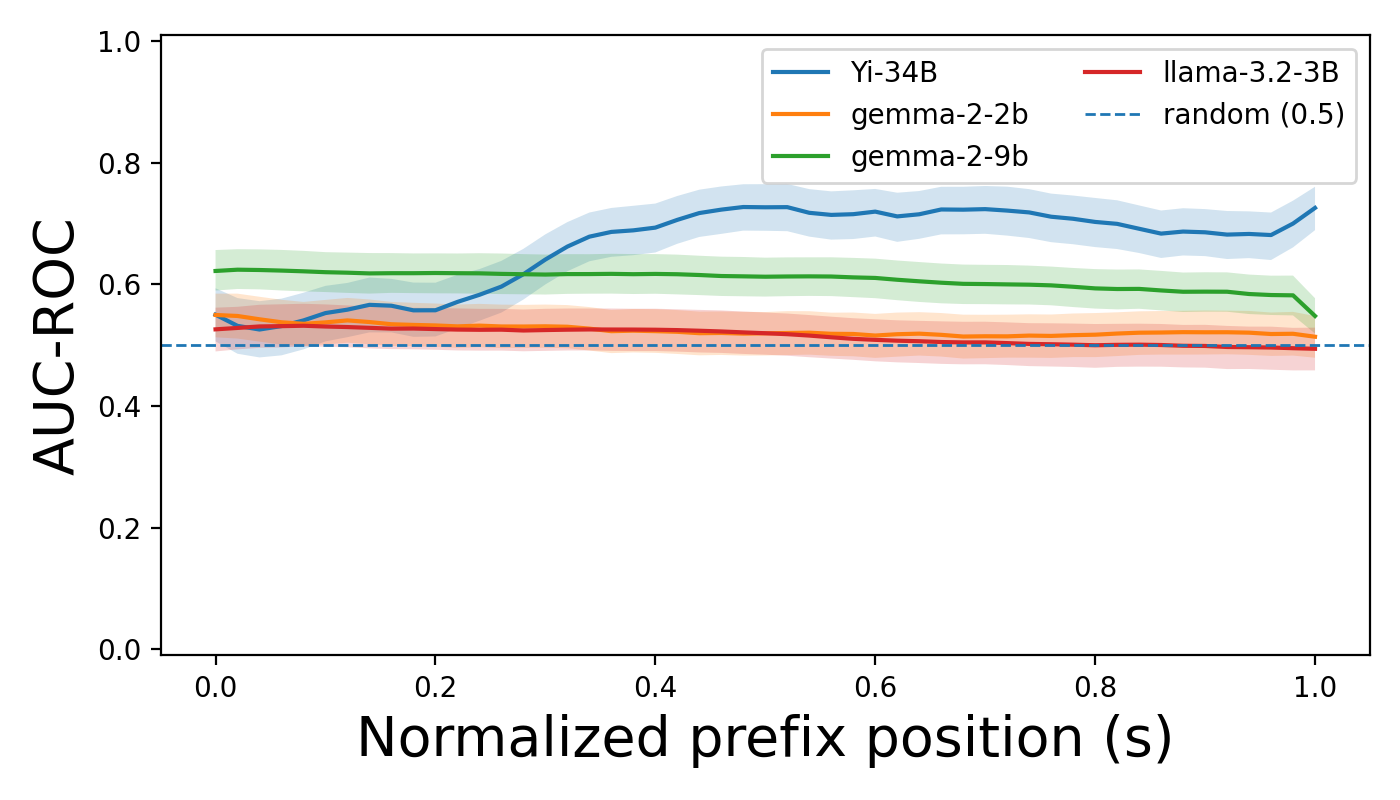}
  \caption{\textbf{Early separability for non-aligned model-dataset pairs on ARC.}
  AUC for using conditional answer entropy to distinguish correct from incorrect traces as a function of relative prefix length $s$.}
  \label{fig:nonaligned-separability-arc}
\end{figure}

\begin{figure}[htbp]
  \centering
  \includegraphics[width=0.72\linewidth]
    {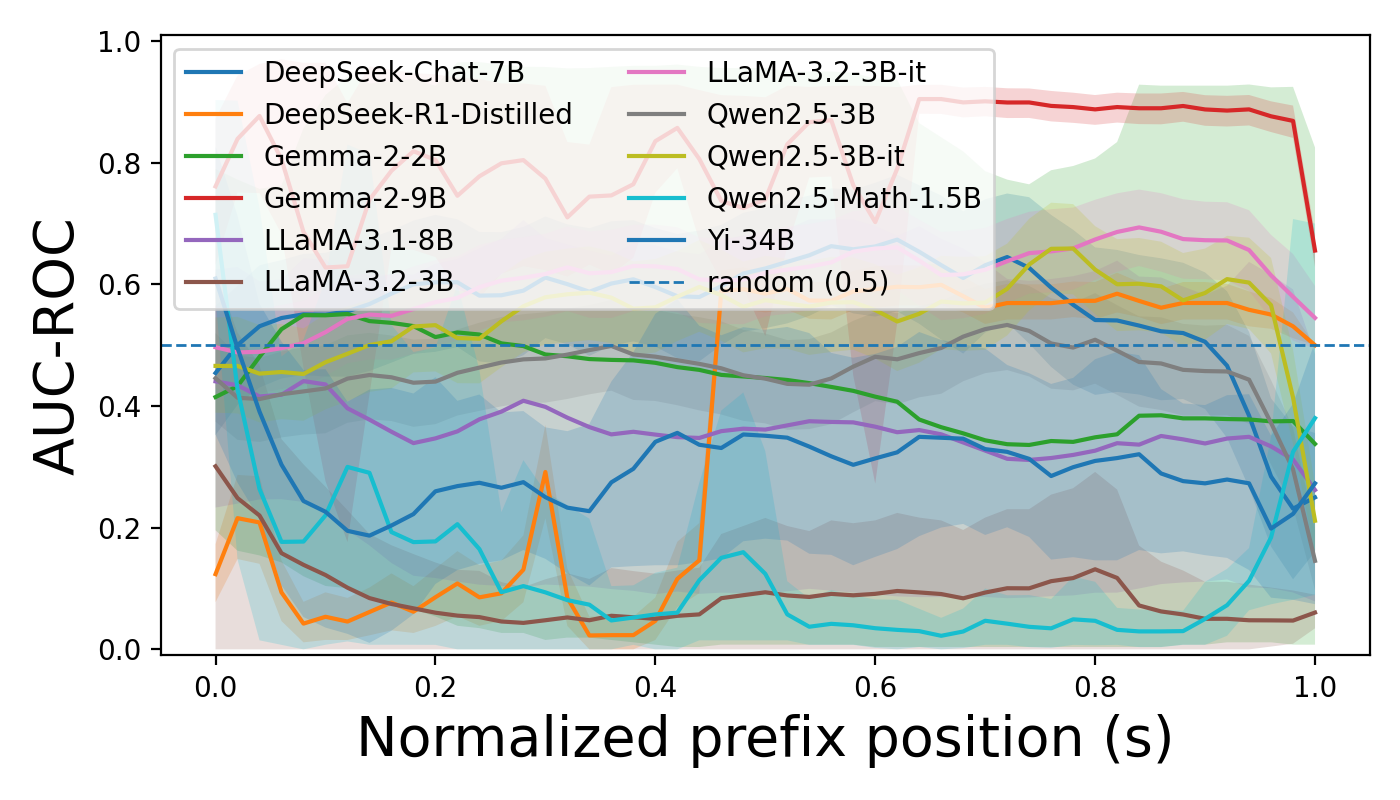}
  \caption{\textbf{Early separability for non-aligned model-dataset pairs on MATH.}
  AUC for using conditional answer entropy to distinguish correct from incorrect traces as a function of relative prefix length $s$.}
  \label{fig:nonaligned-separability-math}
\end{figure}

\section{Further results}
\label{app:further_results}

\subsection{Causal intervention: supervised fine-tuning induces entropy-correctness alignment}
\label{app:lora-sft-sia}

The results in Section~\ref{sec:empirical_validation} show that instruction-tuned and reasoning-optimized models exhibit stronger entropy-correctness alignment than their base counterparts. To more directly test whether reasoning supervision itself induces entropy-correctness alignment, we conduct a controlled intervention experiment in which we start from a fixed base model and track
\(\rho_{\mathrm{SIA}}\) throughout supervised fine-tuning.

Concretely, we initialize from Gemma-2-2B base and perform LoRA-based supervised fine-tuning on
the GSM8K training split, validating throughout on the GSM8K test split. The model is trained to produce GSM8K chain-of-thought rationales followed by the final answer in the standard \texttt{\#\#\#\# <number>} format. We use prompt masking, so the loss is applied only to the completion tokens, and include an explicit EOS token in the supervised target. LoRA adapters are applied to the attention and MLP projection matrices \((q,k,v,o,\mathrm{up},\mathrm{down},\mathrm{gate})\), with rank \(r=16\), scaling \(\alpha=32\), dropout \(0.05\), learning rate \(10^{-4}\), and three training epochs. We save intermediate checkpoints every 25 training steps and evaluate each checkpoint on the GSM8K test split using the same SIA-alignment protocol as in the main experiments.

Figure~\ref{fig:lora_sia_training} shows the resulting trajectory of \(\rho_{\mathrm{SIA}}\) over training. Early in training, \(\rho_{\mathrm{SIA}}\) is negative, indicating that reductions in conditional answer entropy are not yet reliably aligned with increased probability of the gold answer. As supervised fine-tuning proceeds, however, \(\rho_{\mathrm{SIA}}\) rises steadily and becomes positive, remaining positive over later checkpoints. 

\begin{figure*}[t]
    \centering
    \includegraphics[width=0.75\linewidth]{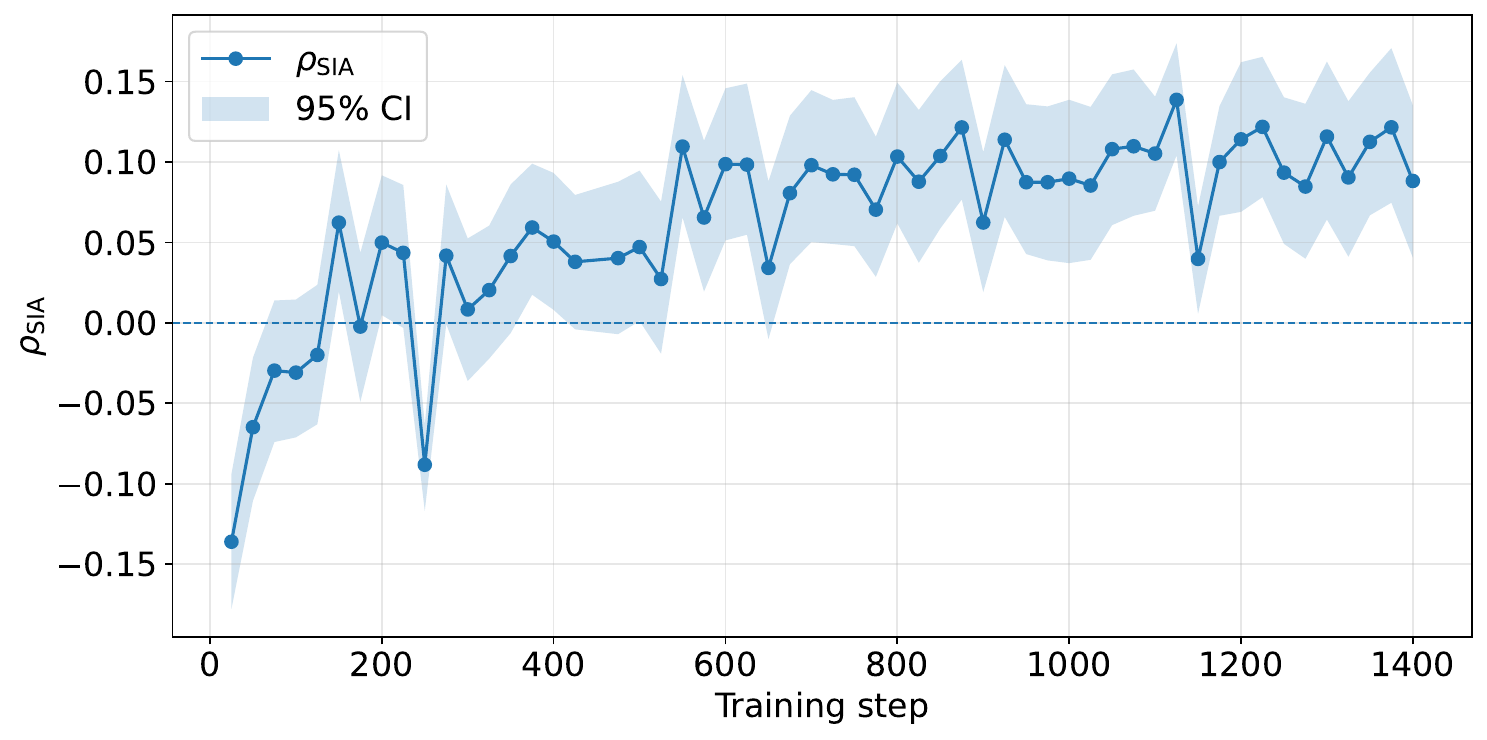}
    \caption{
    \textbf{LoRA SFT induces entropy-correctness alignment.}
    We fine-tune Gemma-2-2B base on GSM8K chain-of-thought solutions using LoRA and evaluate
    the SIA alignment coefficient \(\rho_{\mathrm{SIA}}\) at intermediate checkpoints. The coefficient increases from negative values early in training to consistently positive values after supervised fine-tuning, indicating that reasoning supervision causally strengthens the coupling between conditional answer entropy reduction and increasing probability of the gold answer. Shaded regions denote 95\% bootstrap confidence intervals.
    }
    \label{fig:lora_sia_training}
\end{figure*}

\subsection{Scaling trends in entropy-correctness alignment}
\label{app:scaling-sia}

Table~\ref{tab:sia_alignment_by_dataset} shows that training induces entropy-answer alignment, but also suggest that this effect depends on model capacity relative to task difficulty. To examine this more directly, we aggregate the SIA alignment coefficient $\rho_{\mathrm{SIA}}$ across datasets and plot it against model size for model families for which more than one size is evaluated.

Figure~\ref{fig:size-vs-rho-sia} shows a consistent scaling trend. Within each family, larger models tend to exhibit higher mean $\rho_{\mathrm{SIA}}$, and instruction-tuned variants are systematically above their base counterparts.

These results support the interpretation that entropy-correctness alignment emerges most reliably when model capacity and training signal are sufficient for the model to represent and exploit answer-informative reasoning prefixes. Smaller base models may generate locally coherent reasoning traces that are not reliably informative about the true answer. Instruction tuning strengthens this coupling, and increased model size makes the coupling more robust across datasets.

\begin{figure*}[t]
    \centering
    \includegraphics[width=0.75\linewidth]{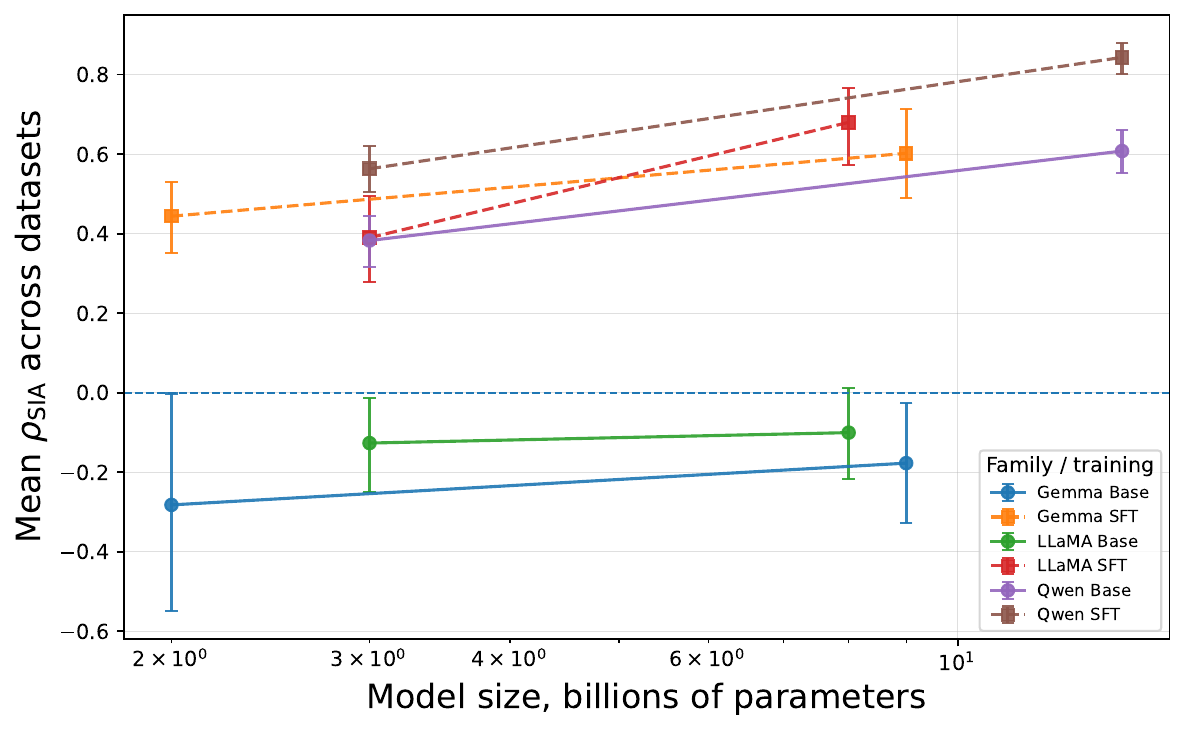}
    \caption{
    \textbf{Scaling trends in entropy-correctness alignment.} Mean $\rho_{\mathrm{SIA}}$ across datasets is plotted against model size for base and instruction-tuned variants. Larger models generally exhibit stronger entropy-correctness alignment, and instruction-tuned models consistently outperform base models within the same family. This suggests that model scale and reasoning-oriented supervision jointly determine whether entropy descent becomes aligned with the true answer.
    }
    \label{fig:size-vs-rho-sia}
\end{figure*}

\section{Ablations}
\label{app:ablations}

\paragraph{Shuffle-prefix ablation (post-hoc).}
To test whether entropy-correctness alignment depends on the ordered structure of reasoning prefixes instead of prefix length alone, we perform a shuffle-prefix ablation: for each evaluated prefix, we randomly permute the tokens within that prefix while preserving its length, then recompute the conditional answer distribution and entropy from the shuffled prefix. We apply this ablation on $100$ GSM8K instances across a subset of models. Table~\ref{fig:shuffle_ablation} reproduces a subset of Table~\ref{tab:sia_alignment_by_dataset} under this setting. Results show that randomly permuting prefix tokens sharply degrades alignment, indicating that truth-directed entropy descent depends on structured accumulation rather than token count. This permutation is applied only at evaluation time and does not affect generation.

\begin{table}[hbtp!]
\caption{\textbf{Shuffle-prefix ablation}. Entropy-correctness alignment ($\rho_{\text{SIA}}$) drops sharply when prefix tokens are permuted. \bad{Negative} or \weak{near-zero} values indicate coupling misalignment.}
\label{fig:shuffle_ablation}
\centering
\small
\begin{tabular}{lcc}
\toprule
\textbf{Model} & {Original mean} & {Shuffled mean} \\
\midrule
Qwen2.5-3B      &  0.682 & \bad{-0.132} \\
Qwen2.5-3B-it   &  0.744 & \bad{-0.005} \\
DeepSeek-R1-Distilled    &  0.795 &  \weak{0.020} \\
Gemma-2-2B-it   &  0.522 & \bad{-0.063} \\
\bottomrule
\end{tabular}
\end{table}

\paragraph{Monte-Carlo approximation.}
Our entropy estimates rely on Monte-Carlo rollouts. To assess robustness to approximation quality, we reran a subset of experiments using a coarser estimator with stride $4$ and $32$ samples, on $100$ GSM8K instances across a subset of models. Table~\ref{tab:mc_ablation} reproduces a subset of Table~\ref{tab:sia_alignment_by_dataset} under this setting. Results remain qualitatively unchanged, indicating that entropy-correctness alignment is not an artifact of low-fidelity Monte-Carlo estimation.

\begin{table}[hbtp!]
\centering
\caption{\textbf{Monte-Carlo ablation.} Entropy-correctness alignment ($\rho_{\text{SIA}}$) remains qualitatively stable across different fidelity Monte-Carlo approximations.}
\label{tab:mc_ablation}
\small
\begin{tabular}{lcc}
\toprule
\textbf{Model} & {Original mean} & {Ablated mean} \\
\midrule
Qwen2.5-3B      &  0.682 & 0.635 \\
Qwen2.5-3B-it   &  0.744 & 0.831 \\
DeepSeek-R1-Distilled    &  0.795 &  0.711 \\
Gemma-2-2B-it & 0.522 & 0.506 \\
\bottomrule
\end{tabular}
\end{table}

\paragraph{Robustness to decoding choices.}
To assess whether entropy-correctness alignment depends on a particular decoding temperature, we reran a subset of experiments across a range of sampling temperatures on $100$ GSM8K instances across a subset of models. Table~\ref{tab:temperature_ablation} reproduces a subset of Table~\ref{tab:sia_alignment_by_dataset} under this setting. Results remain qualitatively unchanged, indicating that entropy-correctness alignment is not an artifact of a set temperature.

\begin{table*}[hbtp!]
\centering
\caption{\textbf{Temperature ablation}. Entropy-correctness alignment ($\rho_{\text{SIA}}$) remains qualitatively stable across decoding temperatures.}
\label{tab:temperature_ablation}
\small
\begin{tabular}{lcccccc}
\toprule
\textbf{Model} & {$T=0.2$} & {$T=0.4$} & {$T=0.6$} & {$T=0.7$} & {$T=0.8$} & {$T=1.0$} \\
\midrule
Qwen2.5-3B              & 0.644 & 0.689 & 0.752 & 0.682 & 0.599 & 0.652 \\
Qwen2.5-3B-it           & 0.762 & 0.733 & 0.905 & 0.744 & 0.756 & 0.821 \\
DeepSeek-R1-Distilled   & 0.837 & 0.843 & 0.744 & 0.795 & 0.686 & 0.781 \\
Gemma-2-2B-it           & 0.616 & 0.385 & 0.543 & 0.522 & 0.417 & 0.613 \\
\bottomrule
\end{tabular}
\end{table*}

\section{Licenses}
\label{subsec:licenses}

We do not introduce or release any new datasets or model checkpoints. All experiments use publicly released benchmark datasets and open-weight model checkpoints under their respective licenses. We use these artifacts only for academic evaluation and do not redistribute the original dataset contents or model weights. Table~\ref{tab:licenses} summarizes the licenses and usage terms for all artifacts used in our experiments.

\begin{table*}[t]
\centering
\small
\caption{Licenses and usage terms for datasets and open-weight model checkpoints used in our experiments.}
\label{tab:licenses}
\resizebox{\textwidth}{!}{%
\begin{tabular}{lll}
\toprule
Artifact & Source & License / terms \\
\midrule
GSM8K & OpenAI & \href{https://huggingface.co/datasets/choosealicense/licenses/blob/main/markdown/cc-by-sa-4.0.md}{MIT License} \\
ARC & AllenAI AI2 ARC & \href{https://huggingface.co/datasets/choosealicense/licenses/blob/main/markdown/cc-by-sa-4.0.md}{CC BY-SA 4.0} \\
SVAMP & Patel et al. / HF mirror & \href{https://huggingface.co/datasets/choosealicense/licenses/blob/main/markdown/cc-by-sa-4.0.md}{MIT License} \\
MATH & Hendrycks et al. & \href{https://huggingface.co/datasets/choosealicense/licenses/blob/main/markdown/cc-by-sa-4.0.md}{MIT License} \\
\midrule
Gemma-2-2B, Gemma-2-9B & Google & \href{https://ai.google.dev/gemma/terms}{Gemma Terms of Use} \\
LLaMA-3.2-3B & Meta & \href{https://www.llama.com/llama3_2/license/}{Llama 3.2 Community License} \\
LLaMA-3.1-8B & Meta & \href{https://www.llama.com/llama3_1/license/}{Llama 3.1 Community License} \\
Qwen-2.5-3B & Qwen & \href{https://huggingface.co/Qwen/Qwen2.5-3B/blob/main/LICENSE}{Qwen Research} \\
Qwen-2.5-14B & Qwen & \href{https://huggingface.co/Qwen/Qwen2.5-14B/blob/main/LICENSE}{Apache 2.0} \\
Qwen-2.5-Math-1.5B & Qwen & \href{https://huggingface.co/Qwen/Qwen2.5-Math-1.5B/blob/main/LICENSE}{Apache 2.0} \\
DeepSeek-Chat-7B & DeepSeek & \href{https://github.com/deepseek-ai/DeepSeek-LLM/blob/main/LICENSE-MODEL}{DeepSeek Model License}; \href{https://github.com/deepseek-ai/DeepSeek-LLM/blob/main/LICENSE-CODE}{code under MIT} \\
DeepSeek-R1-distilled-7B & DeepSeek & \href{https://huggingface.co/datasets/choosealicense/licenses/blob/main/markdown/mit.md}{MIT License} \\
Yi-1.5-34B & 01.AI & \href{https://huggingface.co/datasets/choosealicense/licenses/blob/main/markdown/apache-2.0.md}{Apache 2.0} \\
\bottomrule
\end{tabular}
}
\end{table*}

\section{Impact Statement}
\label{app:broader_impacts}
This paper aims to advance the field of Machine Learning. While our work has potential societal implications, we do not identify any specific concerns that require particular emphasis at this stage.

\onecolumn

\section{Proofs and Auxiliary Lemmas}
\label{app:proofs}
\begin{proof}[Proof of Lemma~\ref{lemma:log-likelihood_identity}]
By definition of expectation under $r$, we have
\[
\mathcal{L}(\theta)
= \mathbb{E}_{X\sim r}[-\log p_\theta(X)]
= -\sum_{x} r(x)\log p_\theta(x).
\]
We now add and subtract $\sum_{x} r(x)\log r(x)$, which equals zero:
\[
\mathcal{L}(\theta)
= -\sum_{x} r(x)\log p_\theta(x)
  + \sum_{x} r(x)\log r(x)
  - \sum_{x} r(x)\log r(x).
\]
Rearranging the terms gives
\[
\mathcal{L}(\theta)
= -\sum_{x} r(x)\log r(x)
  + \sum_{x} r(x)\log\frac{r(x)}{p_\theta(x)}.
\]
The first term is the Shannon entropy
\[
H(r) = -\sum_{x} r(x)\log r(x),
\]
which depends only on the data-generating distribution $r$ and not on $\theta$. It represents the irreducible uncertainty of the data source. In natural language, this idea goes back to Shannon’s analysis of the entropy of printed English \cite{6773263}, and in modern language modeling manifests as a non-zero lower bound on achievable cross-entropy or perplexity, as observed in empirical scaling laws \cite{kaplan2020scalinglawsneurallanguage}. 

The second term is the forward Kullback-Leibler divergence
\[
\KL(r \,\Vert\, p_\theta)
= \sum_{x} r(x)\log \frac{r(x)}{p_\theta(x)}.
\]
Thus we obtain the exact decomposition
\[
\mathcal{L}(\theta)
= H(r) + \KL(r \,\Vert\, p_\theta).
\]

Since $H(r)$ is constant with respect to~$\theta$, minimizing $\mathcal{L}(\theta)$ is equivalent to minimizing $\KL(r \,\Vert\, p_\theta)$.   Therefore, any sequence of parameters $\theta$ that decreases the negative log-likelihood necessarily drives the model distribution $p_\theta$ toward the data distribution $r$ in Kullback-Leibler divergence.  This establishes that $p_\theta \approx r$ whenever $\mathcal{L}(\theta)$ is near its minimum.
\end{proof}

\begin{proof}[Proof of Lemma~\ref{lemma:KLDecompositionoftheJointConditional}]
Using the chain rule of probability, both distributions factorize as
\[
r(C_{1:K},A\mid Q)
= r(C_{1:K}\mid Q)\, r(A\mid Q,C_{1:K}),
\qquad
p_\theta(C_{1:K},A\mid Q)
= p_\theta(C_{1:K}\mid Q)\, p_\theta(A\mid Q,C_{1:K}).
\]
Hence the KL divergence expands to
\[
\KL
= 
\mathbb{E}_{r(C_{1:K},A\mid Q)}
\left[
\log \frac{
r(C_{1:K}\mid Q)\, r(A\mid Q,C_{1:K})
}{
p_\theta(C_{1:K}\mid Q)\, p_\theta(A\mid Q,C_{1:K})
}
\right].
\]
Splitting the logarithm into two terms yields
\[
\KL
=
\mathbb{E}_{r(C_{1:K},A\mid Q)}
\!\left[
\log\frac{r(C_{1:K}\mid Q)}{p_\theta(C_{1:K}\mid Q)}
\right]
+
\mathbb{E}_{r(C_{1:K},A\mid Q)}
\!\left[
\log\frac{r(A\mid Q,C_{1:K})}{p_\theta(A\mid Q,C_{1:K})}
\right].
\]
In the first term the integrand depends only on $C_{1:K}$, so the outer expectation reduces to
$\mathbb{E}_{r(C_{1:K}\mid Q)}$, giving
\[
\mathbb{E}_{r(C_{1:K}\mid Q)}
\left[
\log\frac{r(C_{1:K}\mid Q)}{p_\theta(C_{1:K}\mid Q)}
\right]
= 
\KL\!\left(r(C_{1:K}\mid Q)\,\Vert\,p_\theta(C_{1:K}\mid Q)\right).
\]
For the second term, conditioning on $C_{1:K}$ gives
\[
\mathbb{E}_{r(C_{1:K}\mid Q)}
\!\Big[
\mathbb{E}_{r(A\mid Q,C_{1:K})}
\left[
\log\frac{r(A\mid Q,C_{1:K})}{p_\theta(A\mid Q,C_{1:K})}
\right]
\Big]
=
\mathbb{E}_{r(C_{1:K}\mid Q)}
\!\left[
\KL\!\left(r(A\mid Q,C_{1:K}) \,\Vert\, p_\theta(A\mid Q,C_{1:K})\right)
\right].
\]
Combining both parts yields the claimed identity.
\end{proof}

\begin{lemma}[MLE Implies Marginal and Conditional Alignment]
\label{lemma:MLE_implies_alignment}
Let $Q$ be any question in the support of $r(Q)$.
Given $\delta>0$, if
\(
\KL\!\left(
r(C_{1:K},A \mid Q)\,
\Vert\,
p_\theta(C_{1:K},A \mid Q)
\right)
\le \delta,
\label{eq:joint-smallness}
\)
then 
\(
\KL\!\left(
    r(C_{1:K}\mid Q)
    \,\Vert\,
    p_\theta(C_{1:K}\mid Q)
    \right)
    \le \delta
\) and 
\(
\mathbb{E}_{r(C_{1:K}\mid Q)}
    \!\left[
    \KL\!\left(
    r(A\mid Q,C_{1:K})
    \,\Vert\,
    p_\theta(A\mid Q,C_{1:K})
    \right)
    \right]
    \le \delta .
\)
\end{lemma}

\begin{proof}[Proof of Lemma~\ref{lemma:MLE_implies_alignment}]
By the decomposition in Lemma~\ref{lemma:joint-kl-decomposition},
the joint KL is a sum of two nonnegative terms:
\[
\KL(r\Vert p_\theta)
=
\underbrace{
\KL\!\left(r(C_{1:K}\mid Q)\Vert p_\theta(C_{1:K}\mid Q)\right)
}_{\text{marginal term}}
+
\underbrace{
\mathbb{E}_{r(C_{1:K}\mid Q)} 
\left[
\KL\!\left(r(A\mid Q,C_{1:K})\Vert p_\theta(A\mid Q,C_{1:K})\right)
\right]
}_{\text{conditional term}}.
\]
Thus, if the sum is bounded by $\delta$, then each individual term must also be bounded by
$\delta$:
\[
\KL\!\left(
r(C_{1:K}\mid Q)
\Vert
p_\theta(C_{1:K}\mid Q)
\right)
\le \delta,
\]
and
\[
\mathbb{E}_{r(C_{1:K}\mid Q)}
\!\left[
\KL\!\left(
r(A\mid Q,C_{1:K})
\Vert
p_\theta(A\mid Q,C_{1:K})
\right)
\right]
\le \delta.
\]

This establishes both claims.
\end{proof}

\begin{lemma}[Continuity of Entropy under $\KL$]\label{lemma:entropy-continuity}
Let $P$ and $Q$ be probability distributions on a finite alphabet
$\mathcal{X}$ satisfying $\KL(P \Vert Q) \le \delta.$
Then there exists a function $f_\mathcal{X} : [0,\infty) \to [0,\infty)$
with $f_\mathcal{X}(\delta) \to 0$ as $\delta \to 0$ such that $\bigl| H(P) - H(Q) \bigr| \le f_\mathcal{X}(\delta).$ In particular, for every $\varepsilon > 0$ there exists
$\delta > 0$ such that $\KL(P \Vert Q) \le \delta \quad\Longrightarrow\quad
\bigl| H(P) - H(Q) \bigr| \le \varepsilon.$
\end{lemma}

\begin{proof}[Proof of Lemma~\ref{lemma:entropy-continuity}]
Let $\|\cdot\|_{\mathrm{TV}}$ denote total variation distance,
\[
\|P - Q\|_{\mathrm{TV}} := \frac12 \sum_{x \in \mathcal{X}} |P(x) - Q(x)|.
\]
By Pinsker's inequality,
\[
\|P - Q\|_{\mathrm{TV}} \le \sqrt{\frac12 \KL(P \Vert Q)}
\le \sqrt{\frac{\delta}{2}}.
\]
Let $\varepsilon := \|P-Q\|_{\mathrm{TV}}$. The Fannes-Audenaert inequality \cite{Audenaert_2007} (continuity of entropy on a finite alphabet) states that
for $\varepsilon \le 1 - \tfrac{1}{|\mathcal{X}|}$,
\[
\bigl|H(P) - H(Q)\bigr|
\le \varepsilon \log(|\mathcal{X}| - 1) + h_2(\varepsilon),
\]
where $h_2(\varepsilon) := -\varepsilon \log \varepsilon
- (1-\varepsilon)\log(1-\varepsilon)$ is the binary entropy function.

Combining these two inequalities, for all $\delta > 0$ such that
$\sqrt{\delta/2} \le 1 - 1/|\mathcal{X}|$ we obtain
\[
\bigl|H(P) - H(Q)\bigr|
\le f_\mathcal{X}(\delta),
\]
where one admissible choice is
\[
f_\mathcal{X}(\delta)
:= \sqrt{\frac{\delta}{2}} \log(|\mathcal{X}| - 1)
+ h_2\!\Bigl(\sqrt{\tfrac{\delta}{2}}\Bigr).
\]
The function $f_\mathcal{X}$ is continuous and satisfies
$f_\mathcal{X}(\delta) \to 0$ as $\delta \to 0$, because both
terms on the right-hand side vanish in this limit.

Finally, the $\varepsilon$-$\delta$ formulation follows by continuity:
for any fixed $\varepsilon > 0$, choose $\delta > 0$ such that
$f_\mathcal{X}(\delta) \le \varepsilon$.
\end{proof}

\begin{lemma}[Continuity of Conditional Entropy]\label{lemma:cond-entropy-continuity}
Let $P$ and $Q$ be distributions on a finite product alphabet
$\mathcal{X} \times \mathcal{Y}$, with $\KL(P \Vert Q) \le \delta.$
Then there exists a function $g_{\mathcal{X},\mathcal{Y}}(\delta)$
with $g_{\mathcal{X},\mathcal{Y}}(\delta) \to 0$ as $\delta \to 0$
such that $\bigl| H_P(Y \mid X) - H_Q(Y \mid X) \bigr|
\le g_{\mathcal{X},\mathcal{Y}}(\delta).$
Equivalently, for every $\varepsilon > 0$ there exists $\delta > 0$
such that $\KL(P \Vert Q) \le \delta
\quad\Longrightarrow\quad
\bigl| H_P(Y \mid X) - H_Q(Y \mid X) \bigr| \le \varepsilon.$
\end{lemma}

\begin{proof}[Proof of Lemma~\ref{lemma:cond-entropy-continuity}]
Recall that conditional entropy can be expressed in terms of joint
entropies:
\[
H_P(Y \mid X) = H_P(X,Y) - H_P(X), \qquad
H_Q(Y \mid X) = H_Q(X,Y) - H_Q(X).
\]
Therefore
\begin{align*}
\bigl| H_P(Y \mid X) - H_Q(Y \mid X) \bigr|
&= \bigl| \bigl(H_P(X,Y) - H_Q(X,Y)\bigr)
- \bigl(H_P(X) - H_Q(X)\bigr) \bigr| \\
&\le \bigl|H_P(X,Y) - H_Q(X,Y)\bigr|
+ \bigl|H_P(X) - H_Q(X)\bigr|.
\end{align*}

Let $P_{XY}$ and $Q_{XY}$ denote the joint distributions on
$\mathcal{X}\times\mathcal{Y}$, and $P_X$, $Q_X$ the corresponding
marginals on $\mathcal{X}$. Since marginalization cannot increase
KL divergence, we have
\[
\KL(P_X \Vert Q_X)
\le \KL(P_{XY} \Vert Q_{XY})
= \KL(P\Vert Q) \le \delta.
\]
Applying Lemma~\ref{lemma:entropy-continuity} first to
$(P_{XY},Q_{XY})$ on the alphabet $\mathcal{X}\times\mathcal{Y}$ and
then to $(P_X,Q_X)$ on the alphabet $\mathcal{X}$ yields
\[
\bigl|H_P(X,Y) - H_Q(X,Y)\bigr|
\le f_{\mathcal{X}\times\mathcal{Y}}(\delta), \qquad
\bigl|H_P(X) - H_Q(X)\bigr|
\le f_{\mathcal{X}}(\delta).
\]
Combining the bounds, we obtain
\[
\bigl| H_P(Y \mid X) - H_Q(Y \mid X) \bigr|
\le f_{\mathcal{X}\times\mathcal{Y}}(\delta)
+ f_{\mathcal{X}}(\delta)
=: g_{\mathcal{X},\mathcal{Y}}(\delta),
\]
where $g_{\mathcal{X},\mathcal{Y}}(\delta)\to 0$ as $\delta\to 0$
because each $f$ has this property.
The $\varepsilon$-$\delta$ formulation follows as before.
\end{proof}

\begin{lemma}[Continuity of Conditional Mutual Information]\label{lemma:mi-continuity}
Let $r$ and $p_\theta$ be distributions on a finite product alphabet
$\mathcal{Q}\times\mathcal{C}_1\times\cdots\times\mathcal{C}_K
\times\mathcal{A}$, and fix $k\in\{1,\dots,K\}$. Suppose that $\KL(r \Vert p_\theta) \le \delta.$ Then there exists a function $G_k(\delta)$ with $G_k(\delta)\to 0$ as
$\delta\to 0$ such that $\bigl|
I_r(A; C_{\le k} \mid Q)
-
I_{p_\theta}(A; C_{\le k} \mid Q)
\bigr|
\le G_k(\delta),$ where the mutual informations and entropies on the right-hand side are computed under $p_\theta$, and those on the left-hand side under $r$. Equivalently, for every $\varepsilon>0$ there exists $\delta>0$ such
that $\KL(r \Vert p_\theta) \le \delta
\quad\Longrightarrow\quad
\bigl|
I_r(A; C_{\le k} \mid Q)
-
I_{p_\theta}(A; C_{\le k} \mid Q)
\bigr|
\le \varepsilon.$
\end{lemma}

\begin{proof}[Proof of Lemma~\ref{lemma:mi-continuity}]
We work with finite alphabets, so all random variables take values in finite sets. Recall the standard entropy identity
\[
I(A; C_{\le k} \mid Q)
=
H(A \mid Q)
-
H(A \mid Q, C_{\le k}).
\]
For the distribution $r$ we have
\[
I_r(A; C_{\le k} \mid Q)
=
H_r(A \mid Q)
-
H_r(A \mid Q, C_{\le k}),
\]
and similarly for $p_\theta$:
\[
I_{p_\theta}(A; C_{\le k} \mid Q)
=
H_{p_\theta}(A \mid Q)
-
H_{p_\theta}(A \mid Q, C_{\le k}).
\]
Subtracting the two expressions gives
\begin{align*}
&\bigl|
I_r(A; C_{\le k} \mid Q)
-
I_{p_\theta}(A; C_{\le k} \mid Q)
\bigr| \\
&\qquad\le
\bigl|
H_r(A \mid Q)
-
H_{p_\theta}(A \mid Q)
\bigr|
+
\bigl|
H_r(A \mid Q, C_{\le k})
-
H_{p_\theta}(A \mid Q, C_{\le k})
\bigr|.
\end{align*}

We now apply Lemma~\ref{lemma:cond-entropy-continuity} twice, once
with $(X,Y) = (Q,A)$ and once with $(X,Y) = (Q,C_{\le k},A)$. Since marginalization cannot increase KL divergence, we have
\[
\KL\bigl(r(Q,A)
\Vert
p_\theta(Q,A)
\bigr)
\le \KL(r\Vert p_\theta)
\le \delta,
\]
and similarly for the joint $(Q,C_{\le k},A)$. Thus, Lemma
\ref{lemma:cond-entropy-continuity} yields functions
$g^{(<k)}(\delta)$ and $g^{(\le k)}(\delta)$, each vanishing as
$\delta\to 0$, such that
\[
\bigl|
H_r(A \mid Q)
-
H_{p_\theta}(A \mid Q)
\bigr|
\le g^{(<k)}(\delta),
\]
and
\[
\bigl|
H_r(A \mid Q, C_{\le k})
-
H_{p_\theta}(A \mid Q, C_{\le k})
\bigr|
\le g^{(\le k)}(\delta).
\]
Combining these inequalities gives
\[
\bigl|
I_r(A; C_{\le k} \mid Q)
-
I_{p_\theta}(A; C_{\le k} \mid Q)
\bigr|
\le
g^{(<k)}(\delta) + g^{(\le k)}(\delta)
=: G_k(\delta),
\]
where $G_k(\delta)\to 0$ as $\delta\to 0$. The
$\varepsilon$-$\delta$ formulation is immediate.
\end{proof}

\begin{proof}[Proof of Theorem~\ref{thm:transfer-sia}]
Given a step $k \geq 1$, we have
\[
I_r(A; C_{\leq k} \mid Q) > 0.
\]
Then there exists $\varepsilon_k > 0$ such that
\[
I_r(A; C_{\leq k} \mid Q) \ge \varepsilon_k > 0.
\]
Lemma~\ref{lemma:mi-continuity} (continuity of conditional mutual information) states that if 
\[
\KL(r \Vert p_\theta) \le \delta,
\]
then there exists a function $G_k(\delta)$ with $G_k(\delta) \to 0$ as $\delta \to 0$ such that
\[
\bigl| I_r(A; C_{\le k} \mid Q) - I_{p_\theta}(A; C_{\le k} \mid Q) \bigr|
\;\le\; G_k(\delta).
\]
Equivalently,
\[
I_{p_\theta}(A; C_{\le k} \mid Q)
\ge I_r(A; C_{\le k} \mid Q) - G_k(\delta).
\]
Combining this inequality with the assumed positive prefix-answer mutual information we have
\[
I_{p_\theta}(A; C_{\le k} \mid Q)
\ge \varepsilon_k - G_k(\delta).
\]
If $\delta$ is chosen such that $G_k(\delta) < \varepsilon_k/2$, then
\[
I_{p_\theta}(A; C_{\le k} \mid Q)
\ge \varepsilon_k - G_k(\delta)
\ge \frac{\varepsilon_k}{2}.
\]
This establishes that positive prefix–answer mutual information in the data-generating distribution is preserved under sufficiently small distributional approximation error.
\end{proof}

\begin{proof}[Proof of Proposition~\ref{prop:entropy-not-sia}]
We construct a counterexample with a fixed query $q$ and
binary answer space $\mathcal{A}=\{a^\star,b\}$, where
$\mathcal{A}^\star(q)=\{a^\star\}$.
Suppose the model initially assigns
\[
p_\theta(a^\star\mid q)
=p_\theta(b\mid q)=\frac12.
\]
At the first reasoning step, the model generates one of two
tokens, $c^+$ or $c^-$, each with probability $1/2$.
The conditional answer distributions are
\[
\begin{aligned}
p_\theta(a^\star\mid q,c^+)&=0.9,\\
p_\theta(a^\star\mid q,c^-)&=0.1.
\end{aligned}
\]
These distributions are consistent with the initial
distribution, since their equally weighted mixture assigns
probability $1/2$ to each answer.

Initially, the predictive answer entropy is $\log 2$.
After either token, it becomes
\[
H_\theta(A\mid q,c^\pm)
=-0.9\log 0.9-0.1\log 0.1<\log 2.
\]
Thus, $\delta_1(q,c^\pm)>0$ for both outcomes, and hence
$\Pr[\delta_1>0]=1$.

However, the expected valid-answer information gain is
\[
\begin{aligned}
\mathbb{E}[\Delta_1^\star\mid\delta_1>0]
&=\frac12\log\frac{0.9}{0.5}
 +\frac12\log\frac{0.1}{0.5}\\
&=\frac12\log(0.36)<0.
\end{aligned}
\]
Therefore, predictive answer entropy decreases with
probability one, while the expected true-answer gain is
negative. This establishes the proposition.
\end{proof}

\end{document}